\PassOptionsToPackage{authoryear}{natbib}
\documentclass[onefignum,onetabnum]{siamonline250211}

\usepackage{amsfonts, amsmath, amssymb, bm}
\usepackage{dsfont}
\usepackage{mathrsfs}
\usepackage{mathtools}
\usepackage{amsopn}
\usepackage{pifont}   
\usepackage{xspace}

\usepackage{graphicx}
\usepackage{epstopdf}
\usepackage{tikz}
\usetikzlibrary{arrows.meta}

\ifpdf
  \DeclareGraphicsExtensions{.eps,.pdf,.png,.jpg}
\else
  \DeclareGraphicsExtensions{.eps}
\fi
\graphicspath{{figs/}{plots/}}

\usepackage{longtable}
\usepackage{booktabs}
\usepackage{multirow}
\usepackage{tabularx}
\usepackage{makecell}
\usepackage{ltablex}
\keepXColumns

\newcolumntype{Y}{>{\raggedright\arraybackslash}X}
\newcolumntype{s}{>{\normalsize\centering\arraybackslash}X}

\usepackage{algpseudocode}
\usepackage{enumitem}
\setlist[enumerate]{leftmargin=*, label=(\roman*)}

\usepackage{subcaption}

\makeatletter
\g@addto@macro\normalsize{
  \setlength{\abovedisplayskip}{6pt}
  \setlength{\belowdisplayskip}{6pt}
  \setlength{\abovedisplayshortskip}{4pt}
  \setlength{\belowdisplayshortskip}{4pt}
}
\makeatother

\usepackage{needspace}

\usepackage{ulem}  

\renewcommand{\emph}[1]{\textit{#1}}

\usepackage{natbib}
\bibpunct{(}{)}{;}{a}{,}{,}

\usepackage{hyperref, bookmark}

\hypersetup{
  bookmarksnumbered,
  colorlinks=false,
  allcolors=blue,
  pdfauthor={Aurelien Pion, Emmanuel Vazquez},
  pdftitle={Design-marginal calibration of Gaussian process predictive distributions: Bayesian and conformal approaches},
  pdfsubject={},
  pdfkeywords={}
}

\makeatletter
\newcommand{\maybeClearpage}{}

\@ifclasswith{siamonline250211}{review}{
  \renewcommand{\maybeClearpage}{\clearpage}
}{}
\makeatother

\newsiamremark{remark}{Remark}
\newsiamremark{hypothesis}{Hypothesis}
\crefname{hypothesis}{Hypothesis}{Hypotheses}
\newsiamthm{claim}{Claim}
\newsiamremark{fact}{Fact}
\crefname{fact}{Fact}{Facts}

\DeclareMathOperator \GP {{\mathrm GP}}
\DeclareMathOperator \Var {{\mathrm var}}
\DeclareMathOperator*{\argmin}{arg\,min}

\newcommand \Ccal   {\mathcal{C}}
\newcommand \Dcal   {\mathcal{D}}

\newcommand \U      {\mathcal{U}}

\renewcommand \P    {\mathsf{P}}

\newcommand \E      {\mathsf{E}}

\newcommand \N      {\mathbb{N}}
\newcommand \R      {\mathbb{R}}
\newcommand \XX     {\mathbb{X}}
\newcommand {\one}  {\mathds{1}}

\newcommand \dd     {\text{d}}

\newcommand{\jpgp}{\textsc{j\textsuperscript{+}--gp}\xspace}
\newcommand{\bcrgp}{\textsc{bcr--gp}\xspace}
\newcommand{\cpsgp}{\textsc{cps--gp}\xspace}
\newcommand{\fcp}{\textsc{fcp}\xspace}

\newcommand{\fnsb}[1]{\uline{{#1}}}
\newcommand{\smbf}[1]{{\small\textbf{#1}}}

\headers{Calibration of Gaussian process predictive distributions}{A. Pion, E. Vazquez}

\title{Design-marginal calibration of Gaussian process predictive
  distributions: Bayesian and conformal approaches
  \thanks{This work was funded by Transvalor S.A.}}

\author{
  Aurélien Pion\thanks{Transvalor S.A., Biot, France; 
    Univ. Paris-Saclay, CNRS, CentraleSupélec, L2S, 
    Gif-sur-Yvette, France.} \and
  Emmanuel Vazquez\thanks{Univ. Paris-Saclay, CNRS,
    CentraleSupélec, L2S, Gif-\-sur-Yvette, France. \email{firstname.lastname@centralesupelec.fr}}
}


\begin{document}

\maketitle
\begin{abstract}
  We study the calibration of Gaussian process (GP) predictive 
  distributions in the interpolation setting from a design-marginal 
  perspective. Conditioning on the data and averaging over a design 
  measure~$\mu$, we formalize $\mu$-coverage for central intervals and 
  $\mu$-probabilistic calibration through randomized probability integral 
  transforms. 
  We introduce two methods. \cpsgp adapts conformal predictive systems 
  to GP interpolation using standardized leave-one-out residuals, yielding 
  stepwise predictive distributions with finite-sample marginal calibration. 
  \bcrgp retains the GP posterior mean and replaces the Gaussian residual by 
  a generalized normal model fitted to cross-validated standardized residuals. 
  A Bayesian selection rule---based either on a posterior upper quantile of the 
  variance for conservative prediction or on a cross-posterior 
  Kolmogorov–Smirnov criterion for probabilistic calibration---controls 
  dispersion and tail behavior while producing smooth predictive 
  distributions suitable for sequential design. 
  Numerical experiments on benchmark functions compare \cpsgp, \bcrgp, 
  Jackknife+ for GPs, and the full conformal Gaussian process, using 
  calibration metrics (coverage, Kolmogorov–Smirnov, integral absolute error) 
  and accuracy or sharpness through the scaled continuous ranked probability score.
\end{abstract}


\section{Introduction}
\label{sec:intro}

Gaussian processes (GPs) are classical Bayesian models used to
approximate an unknown real-valued deterministic function $f$ from
limited evaluations over a design space $\XX \subset \mathbb{R}^d$
\citep[see, e.g.,][]{stein_interpolation_1999, santner2003design,
  rasmussen}.  They provide not only point predictions but also a
measure of uncertainty, which makes them central to tasks where
decisions depend on both accuracy and risk. Applications include
Bayesian optimization \citep[see, e.g.,][]{Jones1998:article_ego,
  villemonteix_informational_2009, feliot2017bayesian} and the
estimation of excursion sets \citep{bect2017bayesian,
  azzimonti2021adaptive}, where uncertainty quantification directly
drives the exploration of the design space.

In this work we focus on the interpolation setting, where observations
are assumed exact. A GP prior with mean function $m$ and covariance
kernel $k$ induces a Gaussian posterior predictive distribution at any
location $x\in\XX$. The posterior mean serves as an interpolator of the observed data, while its variance
quantifies predictive uncertainty. Both quantities are available in
closed form.

The combination of exact interpolation and quantified predictive
uncertainty explains the success of GPs in applications where
uncertainty guides the exploration of the input space. Yet, the
quality of decisions relies on the calibration of predictive
distributions: the empirical frequency with which $f(x)$ falls within
nominal prediction intervals, when $x$ varies, should match the intended 
coverage. We assess these frequencies with respect to a design measure $\mu$ on
$\XX$, which governs how observation points are drawn.
Although this design-marginal notion of calibration has appeared in the
literature under various forms, we examine it more closely in this article and refer to it as \emph{$\mu$-calibration}.
In practice, GP-based predictive intervals
are often miscalibrated, leading to overconfident or overly
conservative predictions. Such discrepancies have been documented in the interpolation setting,
for instance with a constant mean function and Matérn covariance
\citep{lod_pion_vazquez}.

Figure~\ref{fig:intro-pareto} illustrates this issue. It displays the
trade-off between predictive accuracy, measured by RMSE, and calibration
quality, measured by the Kolmogorov--Smirnov (KS) metric of
probability integral transform (PIT) values (calibration metrics are
detailed in Section~\ref{sec:background}). The maximum-likelihood (ML)
hyperparameters achieve low RMSE, consistent with the empirical findings
of \cite{petit_parameter_2023}, but poor calibration. In contrast,
post-hoc calibration using the methods presented in this article
improves calibration without degrading accuracy.

\begin{figure}[h!]
  \centering
  \includegraphics[width=\textwidth]{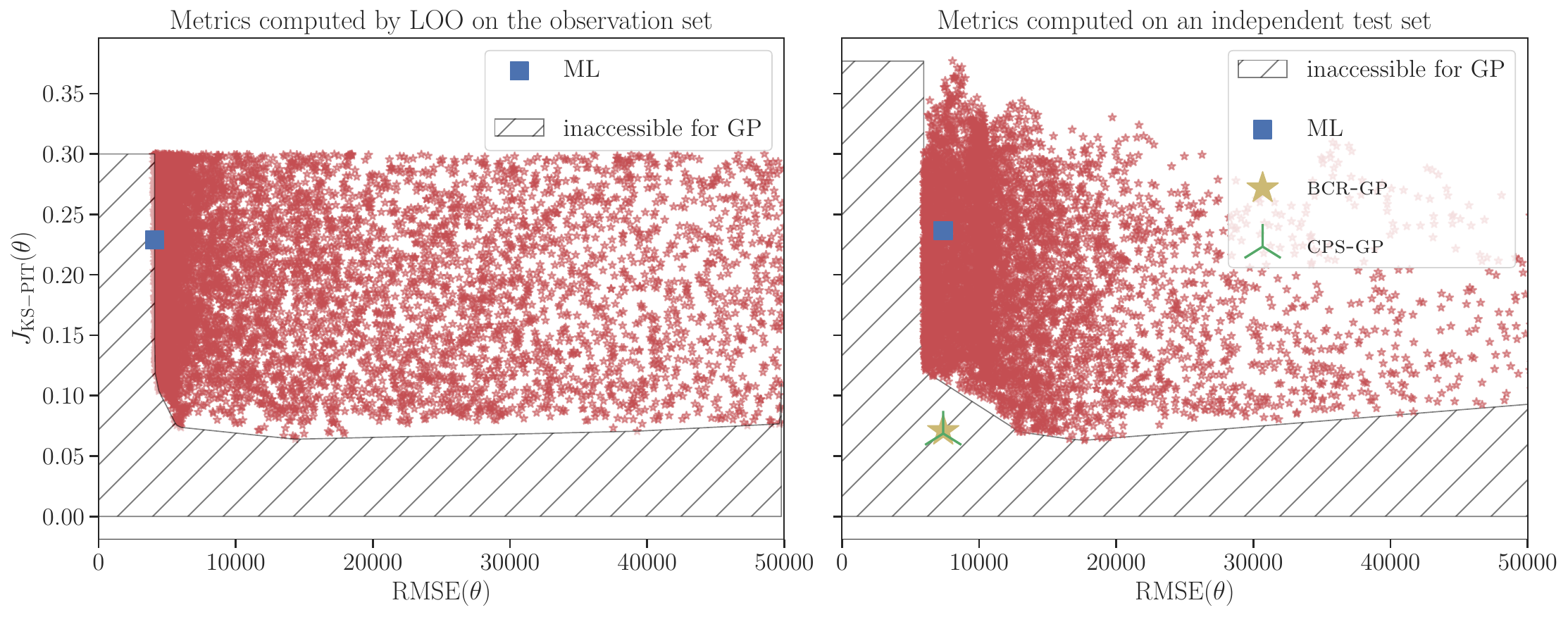}
  \caption{Trade-off between predictive accuracy (RMSE) and calibration quality
    (KS--PIT, smaller is better, see Section~\ref{sec:background}) for a uniform
    random sample of GP kernel hyperparameters (red points).
    We interpolate the Goldstein--Price function using a GP with constant mean and Matérn covariance.
    The left panel shows
    metrics computed by leave-one-out (LOO) on the observation set (150 points),
    and the right
    panel shows metrics computed on an independent test set drawn from $\mu$ (1500 points).
    The ML-selected hyperparameters (blue square) yield accurate but poorly
    calibrated predictions. Post-processed predictors using \cpsgp\ (green
    symbol) and \bcrgp\ (gold star) improve calibration on the test set without
    degrading accuracy. The hatched region corresponds to RMSE--KS--PIT pairs
    that cannot be attained by any GP posterior under the considered GP
    family.}
  \label{fig:intro-pareto}
\end{figure}

Among the approaches developed to address miscalibration, conformal
prediction (CP) is particularly attractive. It is model-agnostic and
provides distribution-free guarantees on marginal coverage
\citep{vovk_Gammerman_2005}. Several recent works adapt CP to Gaussian
processes in order to construct calibrated predictive intervals at
user-specified levels \citep{Papadopoulos_2024, jaber_conformal_2024}.
Beyond interval prediction, \citet{vovk19:_nonpar} introduced the
\emph{conformal predictive systems} (CPS), which extend CP to produce a
calibrated predictive distribution at each test point.

Building on these ideas, we develop two calibration approaches for GP
interpolation: an adaptation of CPS to the interpolation setting
(\cpsgp) and a new Bayesian parametric method (\bcrgp).

The \cpsgp\ construction yields predictive distributions whose associated
prediction sets satisfy the standard conformal finite-sample marginal
coverage guarantee under exchangeability, but their stepwise,
non-differentiable form limits their direct use in standard Bayesian
optimization or excursion-set estimation algorithms.

The second approach, \bcrgp, short for \emph{Bayesian-calibrated
  residuals for Gaussian processes}, retains the GP posterior mean but
replaces the predictive distribution with a parametric family fitted
to normalized GP residuals. We adopt a generalized normal
distribution, with shape and scale parameters selected through a
Bayesian strategy inspired by tolerance-interval constructions
\citep{meeker2017statistical}. The resulting predictive distributions
support inference at arbitrary confidence levels and can be tuned to
favour more or less conservative uncertainty quantification. Unlike
\cpsgp, this parametric formulation produces smooth predictive CDFs,
directly usable in sequential-design algorithms that rely on
closed-form expressions (e.g.\ for expected improvement or related
criteria).

We compare \bcrgp\ with \cpsgp, the Jackknife+ Gaussian process 
method of \citet{jaber_conformal_2024}, and the full conformal 
approach of \citet{Papadopoulos_2024}, using calibration diagnostics 
and proper scoring rules that assess both calibration and sharpness.

The remainder of the article is organized as follows.
Section~\ref{sec:related-work} reviews calibration of GP predictive
distributions, including hyperparameter selection, conformal
prediction, and post-hoc recalibration. Section~\ref{sec:background}
fixes notation and setup, introduces $\mu$-calibration (coverage and
PIT), recalls randomized PIT, and summarizes calibration metrics.
Section~\ref{sec:cps} develops \cpsgp for interpolation and
discusses its properties and practical limitations.
Section~\ref{sec:calibr-using-gener} presents \bcrgp
(generalized normal residual model), parameter estimation, and
implementation details. Section~\ref{sec:numerical} reports
experiments on benchmarks, comparing all methods.
Section~\ref{sec:discussion} concludes with main findings and future
directions.

\section{Related work}
\label{sec:related-work}
We summarize results on GP hyperparameter selection and post-hoc calibration,
including conformal and distributional adjustments.

\subsection{Effect of hyperparameter selection in GPs}
\label{sec:effect-hyperp-select}

Calibration of GP predictive intervals is strongly influenced by how
kernel hyperparameters are selected. Karvonen et
al.~\citep{karvonen2020maximum} study noiseless GP interpolation with
Sobolev/Matérn kernels when only a global scale parameter is estimated
by maximum likelihood. In the RKHS setting, they show that the ratio
between the interpolation error and the posterior standard deviation
is bounded by a factor of order $n^{1/2}$ as the number of points $n$
increases, with a complementary characterization of underconfidence in
terms of a specific subspace of the RKHS. For Sobolev classes that are
smoother or rougher than the RKHS, related polynomial bounds are
obtained under additional assumptions on the design and
smoothness. These results quantify a form of slow overconfidence in a
worst-case sense (supremum over $\XX$ and over the relevant function
class) when only a global scale is fitted by maximum likelihood.

Beyond scale-only fitting, adaptive empirical Bayes also faces structural
limitations. In the Gaussian white-noise model, \citet{szabo2015} show that
credible balls based on marginal-likelihood tuning of prior smoothness cannot
achieve nominal frequentist coverage uniformly over a Sobolev class, even
though the corresponding posteriors are (near) minimax rate-adaptive. They also
show that asymptotically correct coverage can be recovered after
restricting attention to a suitably regular subclass of functions.
For GPs with squared-exponential kernels, \citet{hadji2021can} show that
empirical-Bayes $L^2$-credible balls, with length-scale chosen by marginal
maximum likelihood, can be severely overconfident for a large subclass of
truths, with frequentist coverage converging to zero. They further prove that
coverage can be restored under additional regularity assumptions by inflating
the credible sets or by modifying the empirical-Bayes estimator. Taken
together, these results indicate that credible sets based on empirical-Bayes
hyperparameter tuning do not generally provide reliable uncertainty
quantification without additional corrections or restrictions.

Empirical studies also document the sensitivity of coverage to the
selection criterion. \citet{petit_parameter_2023} compare
likelihood-based and leave-one-out criteria cast as proper scoring
rules and find that the choice of model family often matters more than
the specific selection criterion, with several criteria yielding
comparable performance for Matérn
models. \citet{marrel2024probabilistic} survey estimation and
validation diagnostics and introduce a multi-objective hyperparameter
estimation algorithm targeting improved predictive
distributions. \citet{acharki2023robust} propose a two-step,
coverage-oriented adjustment: they first tune hyperparameters to match
a target leave-one-out coverage and then calibrate prediction
intervals at a prescribed level. Implementing a $(1-\alpha)$ interval
in this approach requires fitting two quantile-oriented GP models (for
the $\alpha/2$ and $1-\alpha/2$ bounds).

\subsection{Post-hoc calibration methods}
\label{sec:post-hoc-calibration}

Conformal prediction (CP) methods are a widely used class of post-hoc
calibration techniques. They wrap around any base predictive model to
produce valid marginal coverage guarantees without assuming model
correctness. Jackknife+ and full conformal procedures have been adapted
to GPs \citep{jaber_conformal_2024, Papadopoulos_2024},
and are recommended for robust calibration \citep{lod_pion_vazquez}.
However, these methods construct prediction intervals only at fixed
confidence levels and do not provide a full calibrated predictive
distribution.

\cite{vovk17:_nonpar, vovk19:_nonpar} introduce the conformal 
predictive system (CPS), which provides a full predictive distribution and 
supports interval estimation at arbitrary levels.
\cite{vovk17:_confor_predic_distr_kernel} adapt CPS to kernel ridge regression, and in 
this work we extend this approach to GP interpolation. The CPS framework and its
adaptation to our setting are detailed in
Section~\ref{sec:cps}.

\subsection{Auxiliary-model and distributional adjustment methods}
\label{sec:auxil-model-distr}

Another line of post-hoc approaches recalibrates predictions by
fitting an auxiliary model. In the GP setting, \citet{capone2023sharp}
keep the base GP for the mean and compute predictive quantiles with a
second GP whose hyperparameters are tuned on a hold-out calibration
set to meet coverage; this uses data splitting (or cross-fitting).

More generally, distributional recalibration learns a
(typically monotone) mapping from predicted to empirical conditional
distributions using a hold-out set. Examples include monotone
quantile/CDF mappings \citep{kuleshov2018accurate} and
density-estimation--based distribution calibration
\citep{kuleshov2022calibrated}. These methods are model-agnostic and
operate on pointwise marginals; they do not exploit GP prior structure
and act only on marginal distributions. In the same spirit,
\citet{Dey2024LADaR} propose LADaR, which applies a local
probability--probability map (Cal-PIT) for instance-wise calibration of
conditional CDFs.

\section{Background and setup}
\label{sec:background}

\subsection{Setting and notation}
\label{sec:bg-setting}

We consider noise-free observations of a deterministic but unknown function
$f : \XX \subset \R^d \to \R$ of the form
$$
  Z_i = f(X_i), \quad i=1,\dots,n,
$$
where the design points $X_i$ are i.i.d.\ from a probability measure
$\mu$ on $\XX$, that will be referred to as the \emph{design measure}. The dataset is
$\Dcal_n = \{(X_i,Z_i)\}_{i=1}^n$.

Throughout, the Bayesian GP framework is used only as a
\emph{construction device} for building predictive distributions; we
do not adopt a Bayesian interpretation ($f$ is fixed, non-random).
We write
$$
  \P_n(\,\cdot\,) := \P\bigl(\,\cdot\,\mid \Dcal_n\bigr)
$$
for probabilities conditional on the observed dataset.

Given $\Dcal_n$, the remaining randomness under $\P_n$ comes solely
from auxiliary draws (e.g., fresh
test points $X_{i}^{\star}\sim\mu$, $i=1,\, 2\, \ldots$), all independent of $\Dcal_n$.

\subsection{Predictive distributions and prediction intervals}
\label{sec:bg-predictive}

Let $\hat F_n(\cdot\mid x)$ denote a predictive CDF for $f(x)$, at $x\in\XX$, constructed from
$\Dcal_n$. In the GP framework, $\hat F_n(\cdot\mid x)$ is taken as a Gaussian CDF with mean and
variance given by the kriging equations \citep[see, e.g.,][]{chiles1999geostatistics, stein_interpolation_1999}.

For any predictive CDF (not necessarily continuous), we use the generalized (left-conti\-nuous) quantile
\begin{equation}
  \label{eq:generalized-inverse}
  \hat F_n^{-1}(p\mid x) \;:=\; \inf\{ z\in\R:\ \hat F_n(z\mid x)\ge p\}, \qquad p\in(0,1),  
\end{equation}
and define the central $(1-\alpha)$ interval
\begin{equation}
  \label{eq:prediction-interval}
  \Ccal_{n,1-\alpha}(x)
  \;=\; \bigl[\hat F_n^{-1}(\alpha/2\mid x),\,\hat F_n^{-1}(1-\alpha/2\mid x)\bigr].  
\end{equation}
If $V \sim \hat F_n(\cdot\mid x)$ independently of $\Dcal_n$, then, under $\P_n$,
$$
  \P_n \left\{\hat F_n^{-1}(\alpha/2\mid x)\ \le\ V\ \le\ \hat F_n^{-1}(1-\alpha/2\mid x)\right\}
  \;\ge\; 1-\alpha,
$$
with equality when $\hat F_n(\cdot\mid x)$ is continuous at both endpoints. Thus,
when $\hat F_n$ has atoms, $\Ccal_{n,1-\alpha}(x)$ is conservative: its predictive
mass under $\hat F_n(\cdot\mid x)$ exceeds $1-\alpha$ whenever an endpoint coincides
with a jump of the CDF.

\paragraph{Exact mass via boundary randomization}
To obtain exact $(1-\alpha)$ mass in the presence of discontinuities,
introduce randomization at the jumps. Fix $x\in\XX$ and let
$\tau\sim\U(0,1)$ be independent of $\Dcal_n$. Define the randomized
predictive CDF
$$
\hat F_{n,\,\tau}(z\mid x)
= \hat F_n(z^-\mid x)
  + \tau\bigl(\hat F_n(z\mid x)-\hat F_n(z^-\mid x)\bigr),
$$
and its (random) generalized inverse
$\hat F_{n,\,\tau}^{-1}( \cdot \mid x)$, defined analogously to~\eqref{eq:generalized-inverse}.

\begin{proposition}[Boundary randomization and exact interval mass]
\label{prop:randomized-pit}
Let $V \sim \hat F_n(\cdot\mid x)$, $x\in\XX$,
be independent of $(\tau,\,\Dcal_n)$, and set $U :=\hat F_{n,\,\tau}(V \mid x)$.
Then, $U \mid \Dcal_n\sim \U(0,1)$. Moreover, the half-open randomized interval
$$
\Ccal_{n,\,\tau,\, 1-\alpha}(x)
:= \bigl[\hat F_{n,\,\tau}^{-1}(\alpha/2 \mid x),\ \hat F_{n,\,\tau}^{-1}(1 - \alpha/2 \mid x)\bigr)
$$
satisfies
$$
\P_n\{\,V \in \Ccal_{n,\tau, 1-\alpha}(x)\,\} = 1-\alpha\,.
$$
\end{proposition}
\begin{proof}
  See Appendix~\ref{sec:proof-randomized-pit}.
\end{proof}

Note that the random variable $U=\hat F_{n,\, \tau}(V \mid x)$ in
Proposition~\ref{prop:randomized-pit} is a
\emph{probability integral transform (PIT)}, a standard tool for
calibration assessment. Its general role will be developed in
Section~\ref{sec:bg-pit}.

\subsection{\texorpdfstring{$\mu$-calibration: general principle}
  {mu-calibration: general principle}}
\label{sec:bg-mu-calibration}

In this work, calibration is assessed relative to the design
measure $\mu$ on $\XX$. After conditioning on the data $\Dcal_n$, we
consider a fresh input $X\sim\mu$ (independent of $\Dcal_n$) and the
corresponding value $f(X)$. A predictive family
$\{\hat F_n(\cdot\mid x):x\in\XX\}$ is said to be
\emph{$\mu$-calibrated} if its predictive statements match the distribution of
$f(X)$ under $X\sim\mu$. Two forms will be considered: \emph{$\mu$-coverage}, 
which concerns the frequency with which $f(X)$ falls inside prediction 
intervals $\Ccal_{n,1-\alpha}(X)$, and \emph{$\mu$-probabilistic calibration}, 
which requires that the PIT values 
$\hat F_n(f(X)\mid X)$ are uniform on~$[0,1]$ under $X\sim\mu$.

The focus on $\mu$-calibration is motivated by empirical testability.
In the classical Bayesian GP framework, the GP posterior predictive distribution at
any fixed $x$ admits a density that is positive on every open interval
(away from observation points where it degenerates in the noise-free
case), so a single realization $f(x)$ cannot empirically falsify
it. By contrast, $\mu$-coverage and $\mu$-probabilistic calibration
are \emph{spatial} properties: they concern the distribution of $f(X)$
under $X\sim\mu$ and can be estimated and tested on an independent
test design, then rejected if they fail. This perspective is
consonant with Matheron’s theory of \emph{regionalized variables},
where spatial, design-dependent properties are regarded as estimable
and models are assessed through observable consequences under the
sampling protocol \citep{matheron89:_estim_choos}.

\subsection{\texorpdfstring{$\mu$-coverage and integrated error}
  {mu-coverage and integrated error}}
\label{sec:bg-coverage}

As introduced above, $\mu$-coverage assesses interval calibration after conditioning
on $\Dcal_n$ and marginalizing over $X\sim\mu$: for a family of centered
prediction intervals $\Ccal_{n,1-\alpha}(x)$, constructed from
predictive distributions $\hat F_n(\cdot \mid x)$, define its $\mu$-coverage as
\begin{align}
  \label{eq:mu-coverage}
  \delta_{\alpha}(\hat F_n;\mu)
  &= \P_n \bigl\{\,f(X)\in \Ccal_{n,1-\alpha}(X)\,\bigr\}\,,   \qquad X\sim\mu\,,  \\[4pt]
  & = \mu \bigl(\{x:\ f(x) \in \Ccal_{n,\, 1-\alpha}(x) \}\bigr)\,. \nonumber
\end{align}
Calibration at level $1-\alpha$ means $\delta_{\alpha}(\hat F_n;\mu)=1-\alpha$.

Given an independent test design $\{X_j^\star\}_{j=1}^{m}$ with
$X_j^\star \sim \mu$, the natural Monte Carlo estimator of
$\delta_{\alpha}(\hat F_n;\mu)$ is
$$
\hat\delta_{\alpha,\, m} (\hat F_n) =
\frac{1}{m}\sum_{j=1}^{m} 
\one\{f(X_j^\star) \in \Ccal_{n,\, 1-\alpha}(X_j^\star)\}.
$$

\begin{remark}
When no separate test set is available, estimating
$\delta_{\alpha}(\hat F_n;\mu)$ on the design $\Dcal_n$ is not
informative. In GP interpolation, $\hat F_n(\cdot\mid X_i)$ is a Dirac
mass at $Z_i=f(X_i)$, so for all $\alpha\in(0,1)$
$$
\Ccal_{n,1-\alpha}(X_i)
=\bigl[\hat F_n^{-1}(\alpha/2\mid X_i),\,\hat F_n^{-1}(1-\alpha/2\mid X_i)\bigr]
=\{Z_i\},
$$
and hence
$$
\one\{f(X_i) \in \Ccal_{n,1-\alpha}(X_i)\}=1.
$$
Empirical coverage computed on $\Dcal_n$ is therefore identically equal
to $1$ for every $\alpha$ and carries no information about
$\delta_{\alpha}(\hat F_n;\mu)$. Cross-validation (e.g., leave-one-out)
can avoid this degeneracy by removing each $(X_i,Z_i)$ when assessing
its inclusion, but introduces additional biases due to data reuse and
dependence.
\end{remark}

To summarize deviations across all confidence levels,
\citet{marrel2024probabilistic} introduced the
\emph{Integrated Absolute Error} (IAE),
$$
J_{\mathrm{IAE}, \mu}(\hat F_n) = \int_0^1 \bigl|\delta_{\alpha}(\hat F_n;\mu) - (1-\alpha)\bigr|\,d\alpha.
$$
The empirical IAE $\hat J_{\mathrm{IAE},\, m}(\hat F_n)$ is obtained by substituting
$\hat\delta_{\alpha,\,m}(\hat F_n)$ for $\delta_{\alpha}(\hat F_n;\mu)$
in the definition above.

\subsection{\texorpdfstring{Probability integral transform and $\mu$-probabilistic calibration}
{Probability integral transform and mu-probabilistic calibration}}
\label{sec:bg-pit}

\paragraph{Probability integral transform (PIT)}
Beyond interval coverage, calibration can also be assessed through
PIT, a classical notion in
probabilistic forecasting (see Appendix~\ref{app:forecasting-primer}). In its
simplest form, given a continuous predictive CDF $\hat F$ for a random variable 
$Z$, the PIT is
$$
  U_{\hat F}^Z = \hat F(Z).
$$
If $\hat F$ coincides with the true distribution of $Z$, then
$U_{\hat F}^Z \sim \U(0,1)$. This property underlies the
standard use of PIT values in forecast evaluation: when multiple
forecast--observation pairs $(\hat F_i,Z_i)$ are available, the
empirical distribution of $\{\hat F_i(Z_i)\}$ is compared to the
uniform distribution, typically through histograms or empirical CDF
plots \citep{Dawid1984, Gneiting:2023}. Departures from uniformity
reveal systematic miscalibration such as over- or underdispersion (see
Figure~\ref{fig:pit}).

\begin{figure}[h]
  \includegraphics[width=\textwidth]{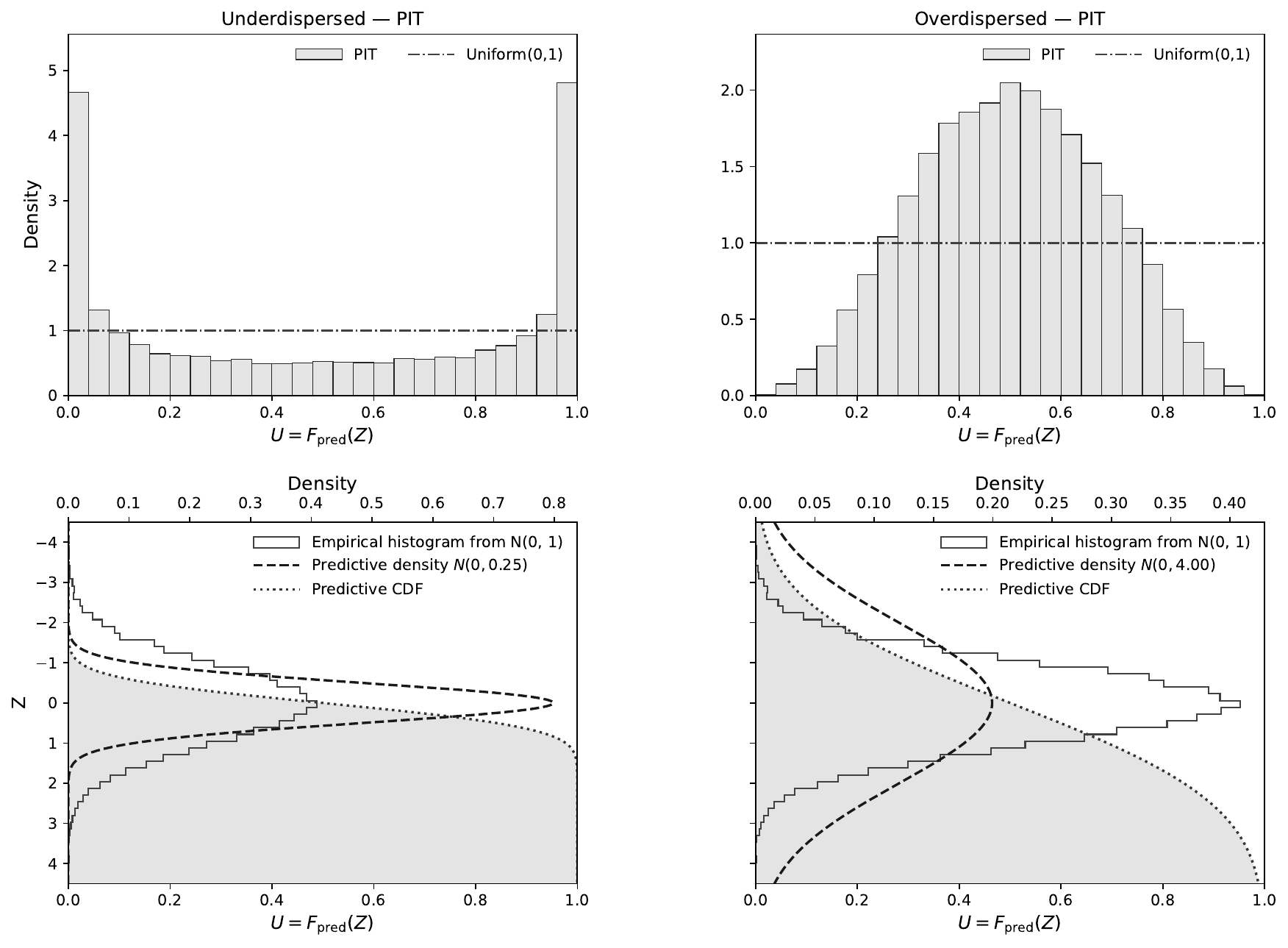}
  \caption{
    \emph{First row}: PIT histograms with the uniform density (dashed line)
    as reference. A \emph{$\cup$-shaped} PIT (mass near $0$ and $1$)
    indicates \emph{underdispersion}; predictive intervals are too
    narrow and observations fall outside too often (optimistic
    coverage). A \emph{$\cap$-shaped} PIT (mass near $1/2$) indicates
    \emph{overdispersion}; intervals are too wide and observations fall
    inside too often (pessimistic coverage). \emph{Second row (rotated
    view):} vertical axis is $z$. Bottom axis shows the predictive CDF
    $u=F_{\mathrm{pred}}(z)$ (shaded area), while top axis shows the
    density scale: horizontal empirical histogram (outline) and
    predictive pdf (dashed). A horizontal slice at a given $z$ maps to
    a CDF value $u$ on the bottom axis, which is precisely the PIT
    value contributing to the histogram in the first row. The empirical
    sample is drawn from the standard normal distribution
    $\mathcal N(0,1)$. Predictive distributions are normal with the same mean
    but different scales: $\mathcal N(0,0.5^{2})$ (underdispersed) and
    $\mathcal N(0,2^{2})$ (overdispersed).}
  \label{fig:pit}
\end{figure}

For discontinuous predictive CDFs, one may use the \emph{randomized PIT}
$$
  U_{\hat F}^{Z,\tau}
  := \hat F(Z^-) + \tau\bigl(\hat F(Z)-\hat F(Z^-)\bigr),
$$
with $\tau \sim \U(0,1)$ independent of $(\hat F,Z)$, which preserves the
uniformity property (see Appendix~\ref{app:forecasting-primer}).

\paragraph{$\mu$-probabilistic calibration}
In our setting, after conditioning on the observed data $\Dcal_n$, we
do not have multiple forecast--observation pairs but instead a single
fitted predictive family $\hat F_n(\cdot\mid x)$ indexed by~$x$. To
apply the PIT idea, we  define a (possibly
randomized) $\mu$-PIT as
\begin{align}
\label{eq:randomized-PIT}
U_{\hat F_n}^{f,\,\mu,\,\tau}
  &:= \hat F_{n,\,\tau}(f(X)\mid X) \\
  & = \hat F_n(f(X)^-\mid X)
   + \tau\bigl(\hat F_n(f(X)\mid X)-\hat F_n(f(X)^-\mid X)\bigr),  \nonumber
\end{align}
with $X \sim \mu$, and where $\tau\sim\U(0,1)$ is independent
of $(\hat F_n,\, X,\, \Dcal_n)$.  Randomization only takes effect when
$\hat F_n(\cdot\mid x)$ has jumps; in the continuous and strictly
increasing case (for $\mu$-a.e.\ $x$), $\mu$-PIT reduces to
$$
U_{\hat F_n}^{f,\,\mu,\,\tau} = U_{\hat F_n}^{f,\,\mu}
= \hat F_n(f(X)\mid X), \qquad X\sim\mu.
$$
We will say that $\hat F_n$ is \emph{$\mu$-probabilistically calibrated} if
$$
U_{\hat F_n}^{f,\,\mu,\,\tau} \sim \U(0,1).
$$

\begin{remark}
\label{rem:pit-under-mu}
Define the distribution function
\begin{equation}
\label{eq:G-mu}
G_\mu(u) := \P_n\bigl(U_{\hat F_n}^{f,\,\mu,\,\tau} \le u\bigr),
\qquad u\in[0,1].  
\end{equation}
Thus,
$\mu$-probabilistic calibration is equivalent to
$$
G_\mu(u) = u,
\qquad \forall u\in[0,1],
$$
In the continuous and strictly monotone case (for $\mu$-a.e.\ $x$),
randomization does not take effect, so that
\begin{align*}
  G_\mu(u)
  &= \P_n \left(U_{\hat F_n}^{f,\,\mu} \le u\right)
   = \E_n \left[\one\{\hat F_n(f(X)\mid X)\le u\}\right] \\[4pt]
  &= \int_{\XX} \one\{\hat F_n(f(x)\mid x)\le u\}\,{\rm d}\mu(x) \\[4pt]
  &= \mu \left(\{x:\ f(x)\le \hat F_n^{-1}(u\mid x)\}\right),
\end{align*}
where $\E_n$ denotes expectation under $\P_n$.
Thus, we obtain the following spatial characterization of
$\mu$-probabilistic calibration: for all $u\in (0,\,1)$,
$$
G_\mu(u) = \mu\bigl(\{x:\ f(x)\le \hat F_n^{-1}(u\mid x)\}\bigr) = u,
\qquad u\in(0,1).
$$

(Without strict monotonicity of $\hat F_n(\cdot \mid x)$, the following bounds hold for all $u\in(0,1)$:
$$
\mu \bigl(\{x:\ f(x)< \hat F_n^{-1}(u\mid x)\}\bigr)\ \le\
G_\mu(u)\ \le\
\mu \bigl(\{x:\ f(x)\le \hat F_n^{-1}(u\mid x)\}\bigr).
$$
See Appendix~\ref{sec:proof-rem-pit-under-mu} for the detailed justification.)
\end{remark}

\paragraph{Example (constant $f$ and non-degenerate $\mu$-probabilistically calibrated $\hat F_n$)}
Assume $f(x)\equiv c$. Let
$\XX=[0,1]$, $\mu=\U [0,1]$, and define
$$
  \hat F_n(z\mid x)=\Phi\big(z-\big(c+\Phi^{-1}(x)\big)\big),
$$
where $\Phi$ is the standard normal CDF. Then, with $X\sim\mu$,
$$
  U_{\hat F_n}^{f,\,\mu}
  \;=\; \hat F_n(c\mid X)
  \;=\; \Phi\big(c-(c+\Phi^{-1}(X))\big)
  \;=\; \Phi\big(-\Phi^{-1}(X)\big)
  \;=\; 1-X \quad\text{a.s.}
$$
Hence $U_{\hat F_n}^{f,\,\mu}\sim \U(0,1)$ and $\hat F_n$ is
$\mu$-probabilistically calibrated, while remaining non-degenerate
(Gaussian predictive law with unit variance at every $x$).
Note, however, that this predictive family is clearly misspecified: although
$f$ is constant, the predictive mean $c + \Phi^{-1}(x)$ varies with $x$
and ranges over $\R$, so $\hat F_n(\cdot\mid x)$ does not represent a
reasonable belief about $f$.

This example highlights a limitation: for a given $f$, many
``incorrect'' predictive families can still be $\mu$-probabilistically
calibrated. However, $\mu$-probabilistic
calibration alone does not penalize overly diffuse or uninformative
predictive distributions. If meaningful uncertainty quantification is
sought, calibration is necessary but not sufficient: one must also
assess \emph{sharpness}---the concentration of predictive
distributions---using proper scoring rules (see
Section~\ref{sec:scoring-rules}), or enforce structural constraints on
the predictive family, as in the \bcrgp method
(see Section~\ref{sec:calibr-using-gener}).

\begin{remark}
\label{rem:mu-coverage-via-PIT}
As already mentioned in Section~\ref{sec:bg-predictive} and
Proposition~\ref{prop:randomized-pit}, $\mu$-coverage and the $\mu$-PIT
are linked as follows (see Appendix~\ref{sec:proof-mu-coverage-via-PIT}
for the detailed derivations).
\begin{enumerate}[label=(\roman*)]
\item If, for $\mu$-a.e.\ $x$, the map
$z\mapsto \hat F_n(z\mid x)$ is continuous and strictly increasing, then
$$
\delta_{\alpha}(\hat F_n;\mu) 
= \P_n\bigl\{\, a \le U_{\hat F_n}^{f,\mu} \le b \,\bigr\},
\qquad \text{with } a=\alpha/2,\ b=1-\alpha/2.
$$

\item Without continuity, one always has the bounds
$$
\P_n\bigl\{\, a \le U_{\hat F_n}^{f,\mu} < b \,\bigr\}
\le \delta_{\alpha}(\hat F_n;\mu)
$$
and
$$
\delta_{\alpha}(\hat F_n;\mu)
\le
\P_n\bigl\{\, a \le U_{\hat F_n}^{f,\mu} \le b \,\bigr\}
+
\mu\bigl(\{x:\ f(x) = q_b,\ \hat F_n( q_b \mid x) > b\}\bigr),
$$
where $q_{b} = \hat F_n^{-1}(b\mid x)$.
(The extra term captures boundary mass when $\hat F_n(\cdot\mid x)$
jumps at the upper endpoint.)

\item When $\hat F_n(\cdot\mid x)$ has jumps, boundary randomization
restores exact equivalence: for every fixed $\tau$,
$$
\mu\bigl(\{x:\ f(x)\in [\,\hat F_{n,\tau}^{-1}(a\mid x),\,\hat
F_{n,\tau}^{-1}(b\mid x)\,)\}\bigr)
= \P_n\bigl\{\, a \le U_{\hat F_n}^{f,\mu,\tau} < b \mid \tau \bigr\}.
$$
\end{enumerate}
Thus, $\mu$-coverage at level $1-\alpha$ can  be expressed as a
PIT inclusion probability, with boundary randomization ensuring exact
equality in the discontinuous case.
\end{remark}

\subsection{\texorpdfstring{Metrics for $\mu$-probabilistic calibration}
  {Metrics for mu-probabilistic calibration}}
\label{sec:mu-pit-metrics}

To evaluate $\mu$-probabilistic calibration we use PIT values on an independent
test design $\{X_j^\star\}_{j=1}^{m}$ with $X_j^\star\sim\mu$.
Given predictive CDFs $\hat F_n(\cdot \mid x)$,  define, for each $j$,
$$
a_j := \hat F_n(f(X_j^\star)\mid X_j^\star),
\qquad
a_j^- := \hat F_n(f(X_j^\star)^-\mid X_j^\star).
$$
The PIT values are then
$$
U_j =
\begin{cases}
a_j, & \text{if $\hat F_n(\cdot\mid X_j^\star)$ is continuous},\\[6pt]
a_j^- + \tau_j (a_j - a_j^-), & \text{otherwise},
\end{cases}
$$
with $\tau_j \sim \U(0,1)$ i.i.d., independent of $(\Dcal_n,\{X_j^\star\})$.

Several metrics can be used to quantify deviations of $\{U_j\}$ from
uniformity.

\paragraph{Variance-based metric}
If $U \sim \U(0,1)$ then $\Var (U)=1/12$. A natural measure of
departure from uniformity is
$$
  J_{\mathrm{Var\text{-}PIT},\,m}(\hat F_n)
  = \frac{1}{m}\sum_{j=1}^{m}(U_j-\tfrac12)^2 - \tfrac{1}{12}.
$$
When the PIT is uniform, $J_{\mathrm{Var\text{-}PIT},\,m}(\hat F_n)\to 0$ as
$m\to\infty$. In particular, for approximately symmetric PIT
distributions, a positive deviation
($J_{\mathrm{Var\text{-}PIT},\,m} > 0$) corresponds to a $\cup$-shaped PIT
histogram: mass accumulates near $0$ and $1$, which means predictive
intervals are too narrow and observed values fall outside too often
(optimistic coverage). A negative deviation
($J_{\mathrm{Var\text{-}PIT},\,m} < 0$) corresponds to a $\cap$-shaped
histogram: mass concentrates near $1/2$, which means predictive intervals
are too wide and observed values fall inside too often (pessimistic
coverage).

\paragraph{KS--PIT metric}
Another possibility is to compare the full empirical distribution of PIT values
to the uniform law using the Kolmogorov--Smirnov (KS) distance. Applied to PIT
values, this yields the \emph{KS--PIT metric} \citep{Diebold1998}:
$$
  J_{\mathrm{KS\text{-}PIT},\,m}(\hat F_n)
  = \sup_{u\in[0,1]}
  \biggl|\frac{1}{m}\sum_{j=1}^{m}\one\{U_j\le u\}-u\biggr|.
$$
Smaller values indicate a PIT distribution closer to uniformity, hence better
$\mu$-probabilistic calibration.

It is natural to introduce the corresponding population quantity.
As defined in Remark~\ref{rem:pit-under-mu} (Equation~\eqref{eq:G-mu}),
$G_\mu(u)$ denotes the distribution function of the $\mu$-PIT values.
The population KS--PIT distance is then
$$
  J_{\mathrm{KS\text{-}PIT},\, \mu}(\hat F_n)
  \;=\; \sup_{u\in[0,1]} \bigl|G_\mu(u)-u\bigr|.
$$
\begin{proposition}
\label{prop:ks_pit_mu_consistency}
Condition on $\Dcal_n$. If $X_j^\star\stackrel{\mathrm{i.i.d.}}{\sim}\mu$,
$\tau_j\stackrel{\mathrm{i.i.d.}}{\sim}\U(0,1)$, indepen\-dent of $(\Dcal_n,\{X_j^\star\})$,
then
$$
  J_{\mathrm{KS\text{-}PIT},\,m}(\hat F_n)
  \xrightarrow[m\to\infty]{\mathrm{a.s.}}
  J_{\mathrm{KS\text{-}PIT},\,\mu}(\hat F_n).
$$
\end{proposition}

\begin{proof}
  See Appendix~\ref{sec:proof_ks_pit_mu_consistency}.
\end{proof}

\begin{remark}[On reuse of observation points]
As with coverage, evaluating KS--PIT on observations using cross-validation can be misleading.
The target is post-data, design-marginal uniformity under $X\sim\mu$, so an
independent test design (or a proper sample split / cross-fitting) should be
used.
\end{remark}

\paragraph{Relation to the Integrated Absolute Error (IAE)}
The Integrated Absolute Error (IAE), recalled in
Section~\ref{sec:bg-coverage}, quantifies deviations between nominal and
empirical coverage of central prediction intervals. The KS--PIT metric,
by contrast, evaluates the full distribution of PIT values against the
uniform law. The two notions are linked:

\begin{proposition}[IAE bounded by KS--PIT]
\label{prop:ks_pit_vs_iae_general}
For both empirical and population measures, the following inequalities hold:
$$
J_{\mathrm{IAE}, m}(\hat F_n) \;\le\; 2\,J_{\mathrm{KS\text{-}PIT}, m}(\hat F_n),
\qquad
J_{\mathrm{IAE}, \mu}(\hat F_n) \;\le\; 2\,J_{\mathrm{KS\text{-}PIT}, \mu}(\hat F_n).
$$
\end{proposition}
\begin{proof}
See Appendix~\ref{sec:proof_ks_pit_vs_iae_general}.  
\end{proof}
Thus, KS--PIT provides a uniformity-based calibration criterion that dominates
IAE: small KS--PIT implies small IAE. In applications, it is useful to
report both metrics: IAE to summarize coverage accuracy across nominal
levels, and KS--PIT as a global indicator of probabilistic calibration.

\subsection{Scoring rules}
\label{sec:scoring-rules}

As noted above, calibration alone does not guarantee useful
uncertainty quantification: overly diffuse forecasts can be perfectly
calibrated yet uninformative. To assess both calibration and
concentration (\emph{sharpness}), we use \emph{proper scoring rules}
\citep{gneiting07:_stric_proper_scorin_rules_predic_estim}, which
assign a numerical score $S(\hat F,z)$ to a predictive CDF $\hat F$
and outcome $z\in\R$. A scoring rule is \emph{strictly proper} if, for
the true distribution $F$,
$$
  \E_{Z\sim F}[S(F,Z)] \;\le\; \E_{Z\sim F}[S(\hat F,Z)],
$$
with equality if and only if $\hat F=F$.

Two standard examples are the (negative) log predictive density (log score)
and the continuous ranked probability score (CRPS):
$$
  S_{\mathrm{NLPD}}(\hat F, z) = -\log \hat f(z),
  \qquad
  S_{\mathrm{CRPS}}(\hat F, z)
  = \int_{-\infty}^{\infty} \bigl(\hat F(u) - \one\{u \ge z\}\bigr)^2 \,\mathrm{d}u,
$$
where $\hat f$ is the density corresponding to $\hat F$ (when it exists).
The CRPS evaluates both calibration and sharpness by integrating the squared
difference between the predictive and empirical CDFs across all thresholds. The CRPS also admits the expectation form
$$
S_{\mathrm{CRPS}}(\hat F, z) = \E \bigl[\,\lvert Z - z\rvert\,\bigr] -
\tfrac{1}{2}\,\E\bigl[\,\lvert Z - Z'\rvert\,\bigr],
$$
where $Z$ and $Z'$ are independent and distributed as $\hat F$.

To reduce the scale sensitivity of CRPS, \citet{Bolin:2023} introduced the
\emph{scaled continuous ranked probability score} (SCRPS).
For a random variable $Z \sim \hat F$ and an independent copy $Z'$, it is defined as
$$
  S_{\mathrm{SCRPS}}(\hat F, z)
  = -\frac{\E\lvert Z - z\rvert}{\E\lvert Z - Z'\rvert}
    - \tfrac{1}{2} \log \bigl(\E\lvert Z - Z'\rvert \bigr).
$$
The SCRPS normalizes the CRPS by the expected dispersion $\E|Z - Z'|$,
making it less sensitive to overall scaling and more suitable for comparing
predictive distributions of different spread.
Lower values correspond to sharper and better-calibrated predictions.
Analytical expressions for computing SCRPS are given in
Appendix~\ref{ap:scrps_formulas}.

Proper scoring rules reward forecasts that are both calibrated (agreement
between predicted probabilities and observed frequencies) and sharp
(concentrated predictive distributions).
Sharpness is meaningful only under calibration; otherwise, a forecast can be
sharp but systematically biased.
CRPS-type scores (including SCRPS) provide a single global measure of
predictive performance and admit decompositions that separate calibration and
sharpness contributions \citep{Arnold:2024}.
By contrast, PIT-based diagnostics (KS--PIT, IAE) focus solely on calibration.
The two perspectives are complementary: scoring rules quantify overall
predictive skill, while PIT-based tools reveal calibration errors more
directly. When tail behavior is of primary interest, tail-sensitive scoring
rules such as those proposed by
\citet{allen23:_weigh_verif_tools_evaluat_univar, allen25:_tail}
offer additional, focused diagnostics.

Given a true distribution~$F$, a forecast $\hat F$, and a scoring rule $S$, 
predictive performance is quantified by the expected score
$$
J_{S}(\hat F) = \E_{Z \sim F}\bigl[S(\hat F, Z)\bigr].
$$
In our interpolation setting, forecasts are given by the family of
predictive CDFs $\hat F_n(\cdot \mid x)$ indexed by $x \in \XX$, and
observation points are taken with respect to the design measure
$\mu$. The corresponding design-marginal score is
$$
J_{S,\mu}(\hat F_n) 
  = \E_{X \sim \mu}\Bigl[\,S\bigl(\hat F_n(\cdot \mid X),\, f(X)\bigr)\,\Bigr],
$$
that is, the score averaged over the input space according to $\mu$.

In practice, with an independent test design 
$X_1^\star,\dots,X_m^\star \stackrel{\text{i.i.d.}}{\sim} \mu$, 
the empirical estimator may be written as
$$
J_{S,\,m}(\hat F_n)  = \frac{1}{m}\sum_{j=1}^m 
S\bigl(\hat F_n(\cdot \mid X_j^\star),\, f(X_j^\star)\bigr)\,.
$$
Such empirical scores are
a standard tool for comparing predictive models and for guiding GP
hyperparameter selection \citep{petit_parameter_2023}.

\section{Conformal predictive systems for Gaussian processes}
\label{sec:cps}

\subsection{Overview of conformal prediction}
\label{sec:overv-conf-pred}

Conformal prediction (CP), introduced by \citet{vovk_Gammerman_2005},
provides a distribution-free method for constructing prediction sets
whose finite-sample coverage is guaranteed, on average over the
randomness of the data, when the observations are exchangeable, which
includes the common case of independent draws.  Many variants exist
(split/inductive CP, jackknife+, CV+, etc.); see
\citet{angelopoulos2025theoreticalfoundationsconformalprediction} for
a recent and comprehensive survey.  Here we recall the randomized full
conformal prediction method
\citep[Chap.~9]{angelopoulos2025theoreticalfoundationsconformalprediction},
implemented with leave-one-out (LOO) conformal scores, which serves as
the theoretical foundation for the \cpsgp method.

Given i.i.d.~data $\Dcal_n=\{(X_i,Z_i)\}_{i=1}^n$ with $X_i\in\XX$ and
$Z_i\in\R$, let $s(x;\Dcal_n)$ denote a regression function providing a point
prediction of $Z$ at $X=x$. In this section, we do not restrict to the interpolation setting:
the responses may remain intrinsically random even conditionally on the
covariates. 
Conformal prediction constructs a $(1-\alpha)$ prediction set for a new $Z_{n+1}$
at covariate $X_{n+1}$, assuming that $(X_{n+1}, Z_{n+1})$ is an independent draw
from the same distribution as the data~$\Dcal_n$.
It relies on a \emph{conformal score} $R(x,z;\Dcal_n)\in\R$, which quantifies the
agreement between a candidate pair $(x,z)$ and the dataset~$\Dcal_n$.
The score is required to be permutation-invariant in its dataset argument:
for any permutation $\sigma$ of $\{1,\ldots,n\}$,
$$
R(x,z;\sigma\Dcal_n)=R(x,z;\Dcal_n).
$$
A common choice is the residual-based score
\begin{equation}
  \label{eq:residual-based-score}
R(x,z;\Dcal_n) = \lvert z - s(x;\Dcal_n)\rvert\,.
\end{equation}

\begin{remark*}[Terminology]
  In the CP literature, a \emph{conformity score} assigns larger
  values to better agreement, while a \emph{nonconformity score}
  (often called a \emph{conformal score}) assigns smaller values to
  better agreement. We use the term conformal score in this text.
\end{remark*}

For a candidate value $z$ at $X_{n+1}$, consider the augmented dataset
$$
\Dcal_{n+1}^z = \{(X_1,Z_1),\ldots,(X_n,Z_n),(X_{n+1},z)\},
$$
and compute the LOO scores
$$
R_i^z = R\bigl( X_i,Z_i;\Dcal_{n+1}^z\setminus\{(X_i,Z_i)\}\bigr),
\quad i=1,\ldots,n+1.
$$

Then, for each $z$, we compare the score $R_{n+1}^z$ to the
$n$ other scores $(R_i^z)_{i=1}^n$. More precisely, introducing an
independent random variable $\tau\sim \U(0, 1)$, define the map
\begin{align}
\pi(z) 
    &= \frac{1}{n+1}\sum_{i=1}^{n+1}\one\{R_i^z < R_{n+1}^z\}
      + \frac{\tau}{n+1}\sum_{i=1}^{n+1}\one\{R_i^z = R_{n+1}^z\}  \label{cp-pi} \\
    &= \frac{1}{n+1}\sum_{i=1}^{n}\one\{R_i^z < R_{n+1}^z\}
    + \frac{\tau}{n+1}\Bigl(1+\sum_{i=1}^{n}\one\{R_i^z = R_{n+1}^z\}\Bigr). \nonumber 
\end{align}
The quantity $\pi(z)$ takes its values in $\left[ 0, 1 \right]$ and
represents the normalized (randomized) rank of the test score $R_{n+1}^z$ in the
augmented dataset: it equals the proportion of scores strictly smaller
than $R_{n+1}^z$, with ties accounted for using $\tau$. While $\pi$ is
not a CDF in the strict sense, it can be interpreted as a predictive
CDF proxy that maps each $z$ to its randomized position within the
augmented sample.

\begin{proposition}[Uniform randomized rank and marginal coverage]
  \label{prop:CP-coverage}
Assume $(X_i,Z_i)_{i=1}^{n+1}$ are i.i.d., and the conformal score
$R(x, z;\Dcal)$ is permutation-invariant in its dataset argument.
Let $\Dcal_{n+1}=\Dcal_{n+1}^{Z_{n+1}}$, and
let $\pi$ be defined as in~\eqref{cp-pi} with a tie--breaker 
$\tau\sim\U([0,1))$ independent of the data. Then
$$
\pi(Z_{n+1}) \sim \U\big([0,1)\big).
$$
In particular, for any $\alpha\in \left[0, 1\right]$,
$$
\P\bigl(\pi(Z_{n+1}) \le 1-\alpha \bigr) = 1 - \alpha.
$$
\end{proposition}
\begin{proof}
  See Appendix~\ref{sec:proof-cp-coverage}.
\end{proof}

The value $p(z) := 1 - \pi(z)$ is known as the conformal $p$-value. 
As in hypothesis testing, $p(z)$ quantifies how extreme the candidate $z$ is 
relative to the observed sample: values consistent with the data yield large 
$p(z)$, while discordant values yield smaller $p(z)$. 
Using this quantity, define the \emph{conformal prediction set}
\begin{equation}
  \label{eq:CP-set}
\Gamma_{n,\tau,1-\alpha}(X_{n+1})
  = \{\,z \in \R : p(z) > \alpha\,\},
\end{equation}
which collects the candidate values sufficiently compatible with the
data.
Proposition~\ref{prop:CP-coverage} implies that
$$
\P\bigl(Z_{n+1} \in \Gamma_{n,\tau,1-\alpha}(X_{n+1})\bigr) = 1-\alpha,
$$
with probability taken over the joint randomness of $\Dcal_{n+1}$ and $\tau$. 
This construction thus provides prediction sets with
exact finite-sample coverage, on average over the randomness of the data, for
any regression function~$s$, under the sole assumption that the observations
are exchangeable, as is the case for independent and identically distributed draws.

Note that the set $\Gamma_{n,\tau,1-\alpha}(X_{n+1})$ need not be an interval in general,
as the comparisons $R_i^z < R_{n+1}^z$ may switch more than once with $z$.
It will, however, reduce to an interval in our \cpsgp construction
(see Section~\ref{sec:cps-gp-interpolation}).

For a comprehensive overview of conformal prediction (CP) and its
variants, we refer the reader to \citet{angelopoulos2025theoreticalfoundationsconformalprediction}.
The discussion above focused on the full conformal method, which forms
the basis for CPS. Other variants have been developed to reduce
computational cost while preserving finite-sample coverage. Split
conformal methods avoid repeated model fitting by using a hold-out
sample, at the cost of wider prediction sets. Jackknife-based methods
such as jackknife+ and CV+ provide an intermediate option, leveraging
leave-one-out or cross-validation fits to improve efficiency without
the full burden of refitting.

\subsection{Conformal predictive systems}
\label{sec:conf-pred-syst}

Standard conformal prediction produces prediction sets, but not full
predictive distributions. To overcome this limitation,
\citet{vovk17:_nonpar, vovk19:_nonpar} introduced the framework of \emph{conformal predictive systems} (CPS),
which extend conformal prediction to output predictive CDFs at each
test covariate $X_{n+1}$. The approach builds on the framework of
full conformal prediction and will be interpreted below through the lens
of $\mu$-probabilistic calibration introduced in
Section~\ref{sec:bg-mu-calibration}.

\paragraph{Randomized predictive systems}

CPS are formalized through the notion of a \emph{randomized predictive system} (RPS).
An RPS is intended to act as a predictive CDF that remains calibrated in finite samples.
Given (i) a dataset $\Dcal_n$ of $n$ observations, (ii) a test covariate--candidate
outcome pair $(x_{n+1},z) \in \XX \times \R$, and (iii) an auxiliary random variable
$\tau \in [0,1]$ for tie-breaking, an RPS returns a value in $[0,1]$ that represents
the predictive probability assigned to the event $\{Z_{n+1} \le z\}$ at covariate $x_{n+1}$.

Formally, an RPS is a function $G$ satisfying the conditions in
Definition~\ref{definition:rps}.

\begin{definition}[\citet{shen2018prediction, vovk19:_nonpar}]
\label{definition:rps}
A measurable function
$$
G : (\XX \times \R) \times (\XX \times \R)^{n} \times [0,1] \to [0,1]
$$
is called a \emph{randomized predictive system} (RPS) if it satisfies:
\begin{enumerate}[leftmargin=*, label=(R.\arabic*)]
\item 
\begin{enumerate}
  \item For each dataset $\Dcal_n$ and test covariate $x_{n+1}\in\XX$,
  the map
  $$
  (z,\tau) \mapsto G\bigl((x_{n+1},z),\,\Dcal_n,\,\tau\bigr)
  $$
  is nondecreasing in both $z \in \R$ and $\tau \in [0,1]$.
  \item For each $\Dcal_n$ and $x_{n+1}\in\XX$,
  $$
  \lim_{z\to -\infty} G\bigl((x_{n+1},z),\,\Dcal_n,\,0\bigr) = 0,
  \qquad
  \lim_{z\to +\infty} G\bigl((x_{n+1},z),\,\Dcal_n,\,1\bigr) = 1.
  $$
\end{enumerate}
\item If $(X_i,Z_i)_{i=1}^n$ are i.i.d.~from a distribution $F$,
$(X_{n+1},Z_{n+1}) \sim F$ independently, and $\tau \sim \U(0, 1)$ is
independent, then, with probability taken over the joint distribution of
$(\Dcal_n,(X_{n+1},Z_{n+1}),\tau)$,
$$
\forall \alpha \in [0,1], \quad 
\P \left(G\bigl((X_{n+1},Z_{n+1}),\,\Dcal_n,\,\tau\bigr) \le \alpha \right) = \alpha.
$$
\end{enumerate}
\end{definition}

Requirement (R.1) ensures that, for fixed data and test covariate, the map
$(z,\tau) \mapsto G((x_{n+1},z),\Dcal_n,\tau)$ behaves like a CDF: it is
monotone in both arguments, and it ranges from $0$ at $z\to -\infty$ (with $\tau=0$)
to $1$ at $z\to +\infty$ (with $\tau=1$). Randomization through $\tau$
provides principled tie handling in conformal ranking and yields exact uniformity of the randomized PIT
and exact mass at CDF jumps; it does not remove the stepwise
nature in $z$ for a fixed $\tau$.

Requirement (R.2) is a probabilistic calibration property: if
$(X_i,Z_i)_{i=1}^n,(X_{n+1},Z_{n+1})$ are i.i.d.~and $\tau\sim\U(0, 1)$ is
independent, the randomized PIT
$$
U := G\bigl((X_{n+1},Z_{n+1}),\,\Dcal_n,\,\tau\bigr)
$$
is uniformly distributed on $[0,1]$.

\begin{remark}[Averaging over the data]
\label{rem:marginal-vs-postdata-mu}
Condition (R.2) ensures that the randomized PIT associated with a CPS
is uniform when averaged over the joint distribution of the observed
data, the test point, and the randomization variable~$\tau$. 
In the $\mu$-probabilistic framework of Section~\ref{sec:bg-mu-calibration},
this would correspond to $\mu$-probabilistic calibration \emph{on average over the
randomness of the observed data}:
$$
\E \left[\P\left\{\,G\bigl((X,f(X)),\,\Dcal_n,\,\tau\bigr)\le u \,\bigm|\, \Dcal_n\,\right\}\right]=u,
\qquad \forall u\in[0,1]\,
$$
(where the expectation is taken over the randomness of $\Dcal_n$).
By contrast, $\mu$-probabilistic calibration as defined in
Section~\ref{sec:bg-mu-calibration} is a \emph{post-data} property:
it requires the conditional equality
$$
\P\left\{\,G\bigl((X,f(X)),\,\Dcal_n,\,\tau\bigr)\le u \,\bigm|\, \Dcal_n\,\right\}=u
$$
to hold for each realization of $\Dcal_n$, which CPS do not guarantee
for finite~$n$. This post-data notion aligns with the \emph{training-conditional} calibration
used in the conformal prediction literature
\citep[see, e.g.,][Ch.~4]{angelopoulos2025theoreticalfoundationsconformalprediction},
while the property ``on average over the data'' corresponds to
\emph{marginal} calibration.
\end{remark}

\paragraph{Conformal predictive distribution}
Using the same setting as in Section~\ref{sec:overv-conf-pred}, let
$\Dcal_{n+1}^z$ be the augmented dataset and
$(R_i^z)_{i=1}^{n+1}$ the associated leave-one-out scores.
Recall from~\eqref{cp-pi} the function $\pi(z)$, which gives the
(normalized, possibly randomized) rank of the test score among the
augmented scores. For notational consistency, we introduce the equivalent notation
\begin{equation}
  \label{eq:def-G}
G\bigl((x,z),\,\Dcal_n,\,\tau\bigr)
  := \frac{1}{n+1}\sum_{i=1}^{n+1}\one\{R_i^z < R_{n+1}^z\}
     + \frac{\tau}{n+1}\sum_{i=1}^{n+1}\one\{R_i^z = R_{n+1}^z\},
\end{equation}
with $\tau \sim \U(0, 1)$ independent of
$(\Dcal_n, (X_{n+1}, Z_{n+1}))$. We now check whether $G$ is an RPS.
Requirement (R.2) holds provided the observation pairs
$(X_i,Z_i)_{i=1}^{n+1}$ are exchangeable (e.g., i.i.d.) and the score
$R(\cdot,\cdot;\cdot)$ is permutation-invariant in its dataset argument; then,
by Proposition~\ref{prop:CP-coverage},
$G\bigl((X_{n+1},Z_{n+1}),\Dcal_n,\tau\bigr)\sim U(0, 1)$.
The monotonicity condition (R.1) depends on the chosen conformal
score. When using the absolute-residual score
\eqref{eq:residual-based-score}, $G$ may fail to be monotone in $z$,
and the construction does not define an RPS. \citet{vovk19:_nonpar}
provide general sufficient conditions under which a function of the
form~\eqref{eq:def-G} defines a valid RPS: broadly, the conformal
scores must vary with the candidate label $z$ in an order-preserving
manner. We will later verify that these properties hold for the
score structure used in the \cpsgp method.

When (R.1) and (R.2) hold, the map $z\mapsto G\bigl((x,z),\Dcal_n,\tau\bigr)$ acts as a
(randomized) predictive CDF at $x$; we then call it a \emph{conformal predictive
distribution (CPD)} and write
$$
\hat F_{n,\tau}^{\mathrm{CPD}}(z\mid x) := G\bigl((x,z),\,\Dcal_n,\,\tau\bigr).
$$
Evaluating $\hat F_{n,\tau}^{\mathrm{CPD}}(\cdot\mid X_{n+1})$ at the realized
outcome $Z_{n+1}$ yields a randomized PIT value that is uniform, so the predictive
distribution is probabilistically calibrated when averaging over the joint
randomness of $\Dcal_n$, $(X_{n+1},Z_{n+1})$, and~$\tau$ (``marginal'' calibration).

\begin{remark}[Terminology]
  In \citet{vovk19:_nonpar}, the candidate CPD in~\eqref{eq:def-G} is
  called a \emph{conformal transducer}. A \emph{conformal predictive
    system (CPS)} is such a map that also satisfies the RPS axioms
  (Definition~\ref{definition:rps}); for a CPS, the induced predictive
  CDF is termed a \emph{conformal predictive distribution (CPD)}.
  For simplicity, we avoid ``transducer'' in the main text: we say ``candidate CPD''
  for~\eqref{eq:def-G}, and ``CPD'' once the RPS conditions are met.
\end{remark}

\subsection{CPS for GP interpolation}
\label{sec:cps-gp-interpolation}

We derive the form of the CPS when
the conformal score is constructed from Gaussian process (GP)
interpolation, following the approach of \citet{vovk17:_confor_predic_distr_kernel}
for kernel ridge regression. In this setting, the CPS admits a closed-form
expression and is a step function of the candidate outcome~$z$, with jumps
determined by thresholds derived from the GP posterior.

\paragraph{Setup and assumptions}
Let $f:\XX \subset \R^d \to \R$ be an unknown deterministic function. We
observe a noiseless sample
$\Dcal_n= \{(X_i,Z_i)\}_{i=1}^n$ with $X_i\sim \mu$ i.i.d.\ and $Z_i=f(X_i)$,
where $\mu$ is the design measure introduced in
Section~\ref{sec:bg-setting}. A GP prior
$\xi \sim\GP(m_\theta,k_\theta)$ is specified with mean $m_\theta$ and
covariance $k_\theta$ depending on hyperparameters~$\theta$. In the
derivation below, $\theta$ is treated as a fixed constant, independent
of $\Dcal_n$ (for instance, pre-specified from prior knowledge or selected
once in a separate, external step). This assumption yields closed-form
expressions and preserves the exchangeability property needed for CPS
validity.

\medskip\noindent\textbf{NB.}
Although $(X_i,Z_i)_{i=1}^n$ and $X_{n+1}$ are random, in the algebra below we
condition on a fixed dataset $\Dcal_n=\{(x_i,z_i)\}_{i=1}^n$ and a fixed test
location $x_{n+1}$. All GP posterior quantities such as $m_n$ and $\sigma_n$
are then deterministic functions of $(\Dcal_n,x_{n+1},\theta)$. The same holds
for the auxiliary coefficients introduced later (e.g.\ the slopes and
thresholds in Proposition~\ref{prop:affine-diff}). We return to the random-design
setting when stating calibration results.

For a test covariate $x_{n+1}$ and candidate outcome $z\in\R$, form the
augmented dataset $\Dcal_{n+1}^z=\Dcal_n\cup\{(x_{n+1},z)\}$ and use the
standardized residual score \citep{Papadopoulos_2024}. The score at the test
point is
\begin{equation}
  \label{eq:cps-gp-score}
  R_{n+1}^z = \frac{z - m_n(x_{n+1})}{\sigma_n(x_{n+1})},
\end{equation}
where $m_n$ and $\sigma_n$ denote the GP posterior mean and standard deviation
given~$\Dcal_n$ (under the same fixed~$\theta$).
For each observed pair
$(x_i,z_i)$, the leave–one–out score computed on
$\Dcal_{n+1}^z\setminus\{(x_i,z_i)\}$ is
\begin{equation}
  \label{eq:cps-gp-loo-score}
  R_i^z = \frac{z_i - m_{n+1,-i}(x_i)}{\sigma_{n+1,-i}(x_i)}.
\end{equation}

Under these assumptions, the difference between the
test score and each data--point score is affine in~$z$.

\begin{proposition}[Affine differences]
  \label{prop:affine-diff}
Fix $\Dcal_n$, $x_{n+1}$, and $\theta$. For each $i=1,\ldots,n$, there exist a
slope $\beta_i>0$ and a threshold $c_i\in\R$, depending only on
$(\Dcal_n,\,x_{n+1},\,\theta)$, such that
$$
R_{n+1}^z - R_i^z \;=\; \beta_i\,(z - c_i), \qquad z\in\R.
$$
\end{proposition}

\begin{proof}
See Appendix~\ref{app:AiBi-deriv}, which provides explicit expressions for
$\beta_i$ and $c_i$.
\end{proof}
This affine representation will be used in Proposition~\ref{prop:stepwise}
to show that the resulting conformal predictive distribution is a step
function with jumps at the thresholds $\{c_i\}_{i=1}^n$.

\begin{remark}[Hyperparameters selected on $\Dcal_n$]
\label{rem:theta-selected}
In practice, $\theta$ is often selected from the same data $\Dcal_n$ (e.g., by
maximum likelihood, \citealp{petit_parameter_2023}) and then kept fixed when
evaluating conformal scores. With $\theta$ fixed after selection,
Proposition~\ref{prop:affine-diff} still applies.

This choice breaks the permutation invariance required by full
conformal prediction: $\theta=\theta(\Dcal_n)$ is a function of the
$n$ observed pairs only, not of the candidate $(X_{n+1},z)$. Under a
permutation of the $n+1$ pairs, $\theta$ should change to reflect the
permuted first $n$ pairs. Keeping $\theta(\Dcal_n)$ fixed prevents
this, so the joint distribution of conformity scores is no longer
invariant under permutation, and exact finite–sample validity does not
hold.

If formal validity is needed, one can either (i) use a split scheme (select
$\theta$ on a disjoint subset, then conformalize on the remainder), or
(ii) re-select $\theta$ for each augmented dataset
$\Dcal_{n+1}^z\setminus\{(X_i,Z_i)\}$ (or over a grid in $z$). The latter
restores exchangeability but is computationally expensive and removes the simple
affine dependence on $z$.
\end{remark}

\begin{remark}[Extensions to universal and intrinsic kriging]
The framework extends beyond the case of a known mean function and strictly
positive definite kernel. It also applies to \emph{universal kriging} models
with an unknown linear mean $m(x)=h(x)^\top\beta$ (for known regressors
$h(x)\in\R^q$ and unknown coefficients~$\beta$) and to \emph{intrinsic kriging}
based on conditionally positive definite kernels
\citep{matheron1973intrinsic,chiles1999geostatistics}.  
The corresponding modifications to the leave-one-out formulas and to the affine
representation of $R_{n+1}^z - R_i^z$ are outlined in
Appendix~\ref{app:AiBi-deriv}.
\end{remark}

\paragraph{CPD construction}
Specializing the
map~\eqref{eq:def-G} to the standardized-residual scores
\eqref{eq:cps-gp-score}--\eqref{eq:cps-gp-loo-score}, we define
\begin{equation}
\label{eq:def-cps-gp}
\hat F_{n,\tau}^{\mathrm{CPS\text{-}GP}}(z\mid x_{n+1})
\;:=\; G\bigl((x_{n+1},z),\,\Dcal_n,\,\tau\bigr),
\end{equation}
with independent $\tau\sim\U(0, 1)$.

The function
$\hat F_{n,\tau}^{\mathrm{CPS\text{-}GP}}\left(\;\cdot \mid x_{n+1}\right)$ is a
candidate CPD at the test location~$x_{n+1}$. The next result
shows that, in our GP interpolation setting, this function has a
stepwise form with discontinuities at the thresholds
$\{c_i\}_{i=1}^n$.
We then verify that
it satisfies the RPS axioms, showing that
$\hat F_{n,\tau}^{\mathrm{CPS\text{-}GP}}$ defines a valid CPD.

\begin{proposition}[Stepwise form of $\hat F_{n,\tau}^{\mathrm{CPS\text{-}GP}}$]
\label{prop:stepwise}
Fix $\Dcal_n$ and $x_{n+1}$, and use the standardized residual score.
Let $\{c_i\}_{i=1}^n$ be the thresholds from Proposition~\ref{prop:affine-diff},
and order them as
$$
-\infty=:c_{(0)} < c_{(1)} \le \cdots \le c_{(n)} < c_{(n+1)}:=\infty.
$$
Then, for any $\tau\in[0,1]$,
$$
\hat F_{n,\tau}^{\mathrm{CPS\text{-}GP}}(z\mid x_{n+1})
=
\begin{cases}
\dfrac{i+\tau}{n+1}, & \text{if } z\in(c_{(i)},c_{(i+1)}),~~i=0,\ldots,n,\\[8pt]
\dfrac{i'-1+\tau\,(i''-i'+2)}{n+1}, & \text{if } z=c_{(i)} \text{ for some } i,
\end{cases}
$$
where, for a tie point $z=c_{(i)}$, the block endpoints
$$
i' := \min\{r:~c_{(r)}=c_{(i)}\}, \qquad
i'' := \max\{r:~c_{(r)}=c_{(i)}\}
$$
denote the first and last indices of the tied block.
In particular, $z\mapsto \hat F_{n,\tau}^{\mathrm{CPS\text{-}GP}}(z\mid x_{n+1})$
is a step function with jumps at $\{c_i\}_{i=1}^n$.
\end{proposition}

\begin{proof}
Since $\beta_i>0$, the relative ordering of $R_{n+1}^z$ and $R_i^z$
can change only at $z=c_i$. Sorting these thresholds partitions $\R$
into the intervals $(c_{(i)},c_{(i+1)})$.  
If $z\in(c_{(i)},c_{(i+1)})$, then $z$ exceeds exactly $i$ thresholds,
so $R_{n+1}^z$ exceeds $i$ observed scores and is smaller than the
others. Thus
$$
\hat F_{n,\tau}^{\mathrm{CPS\text{-}GP}}(z\mid x_{n+1}) = \frac{i+\tau}{n+1}.
$$
If $z=c_{(i)}$ lies in a block of ties
$c_{(i')}=\cdots=c_{(i'')}=c_{(i)}$, then the test score exceeds $i'-1$
scores, equals $(i''-i'+1)$ scores, and is smaller than the rest.
The randomized rank with tie–breaking is
$$
\hat F_{n,\tau}^{\mathrm{CPS\text{-}GP}}(c_{(i)}\mid x_{n+1})
=\frac{i'-1+\tau\,(i''-i'+2)}{n+1}.
$$
\mbox{}
\end{proof}

Up to this point, we conditioned on a fixed dataset
$\Dcal_n=\{(x_i,z_i)\}_{i=1}^n$ and a fixed test location
$x_{n+1}$. We now lift this conditioning: let $(X_i,Z_i)_{i=1}^n$ be
i.i.d., let $(X_{n+1},Z_{n+1})$ be an independent draw from the same
distribution, and let $\tau\sim\U(0, 1)$ be independent of everything
else. Throughout the calibration statements below, the hyperparameters
$\theta$ used in the scores are taken as fixed independently of the
sample (or chosen on a disjoint split).

We refer to this construction as the \emph{conformal predictive system for
Gaussian-process interpolation} (\cpsgp). The next result shows that,
under the assumptions of Section~\ref{sec:cps-gp-interpolation}, the
map $z\mapsto \hat F_{n,\tau}^{\mathrm{CPS\text{-}GP}}(z\mid x_{n+1})$
satisfies the RPS axioms and therefore defines a valid CPD.

\begin{proposition}
\label{prop:cps-valid}
Assume the setup of Section~\ref{sec:cps-gp-interpolation}:
interpolation with known mean function $m$ and strictly positive
definite kernel $k$, the standardized
residual score, and an independent tie–breaker $\tau\sim\U(0, 1)$.
If, in addition, $(X_i,Z_i)_{i=1}^n$ and $(X_{n+1},Z_{n+1})$ are
exchangeable (e.g., i.i.d.), then
$$
\hat F_{n,\tau}^{\mathrm{CPS\text{-}GP}}(\;\cdot\mid x_{n+1})
$$
satisfies the RPS axioms (Definition~\ref{definition:rps}); hence
\cpsgp defines a valid CPD.
\end{proposition}

\begin{proof}
  (R.1): By Proposition~\ref{prop:stepwise},
  $z\mapsto \hat F_{n,\tau}^{\mathrm{CPS\text{-}GP}}(z\mid x_{n+1})$
  is nondecreasing, with the following properties:
  $\lim_{z\to-\infty}\hat F_{n,0}^{\mathrm{CPS\text{-}GP}}(z\mid
  x_{n+1})=0$ and
  $\lim_{z\to+\infty}\hat F_{n,1}^{\mathrm{CPS\text{-}GP}}(z\mid
  x_{n+1})=1$, and it is nondecreasing in $\tau$.  (R.2): By
  exchangeability and independence of $\tau\sim\U(0, 1)$,
  $\,\hat F_{n,\tau}^{\mathrm{CPS\text{-}GP}}(Z_{n+1}\mid
  X_{n+1})\sim\U(0, 1)$ (Proposition~\ref{prop:CP-coverage}).
\end{proof}

By Proposition~\ref{prop:cps-valid}, the randomized PIT based on
$\hat F_{n,\tau}^{\mathrm{CPS\text{-}GP}}(\cdot\mid X)$ is uniform
when averaged over the joint randomness of $(\Dcal_n, X, \tau)$.
Thus, \cpsgp achieves design-marginal probabilistic calibration
on average over the data. This guarantee is not data-conditional:
for a fixed $\Dcal_n$ the $\mu$-PIT need not be exactly uniform.
In numerical experiments on benchmark functions
(Section~\ref{sec:numerical}), post-data calibration diagnostics
(KS--PIT, IAE, and coverage) computed on an independent test design
$X^\star\sim\mu$ remain close to their ideal targets.

In practical terms, Proposition~\ref{prop:stepwise} shows that
$\hat F_{n,\tau}^{\mathrm{CPS\text{-}GP}}(\cdot\mid x_{n+1})$ is
a piecewise–constant function whose jumps occur at data–dependent
thresholds $\{c_i\}_{i=1}^n$. Figure~\ref{fig:comp_cdf} contrasts this
stepwise conformal predictive distribution with the smooth Gaussian
posterior CDF.

\begin{figure}[h!]
  \centering
  \includegraphics[width=\textwidth]{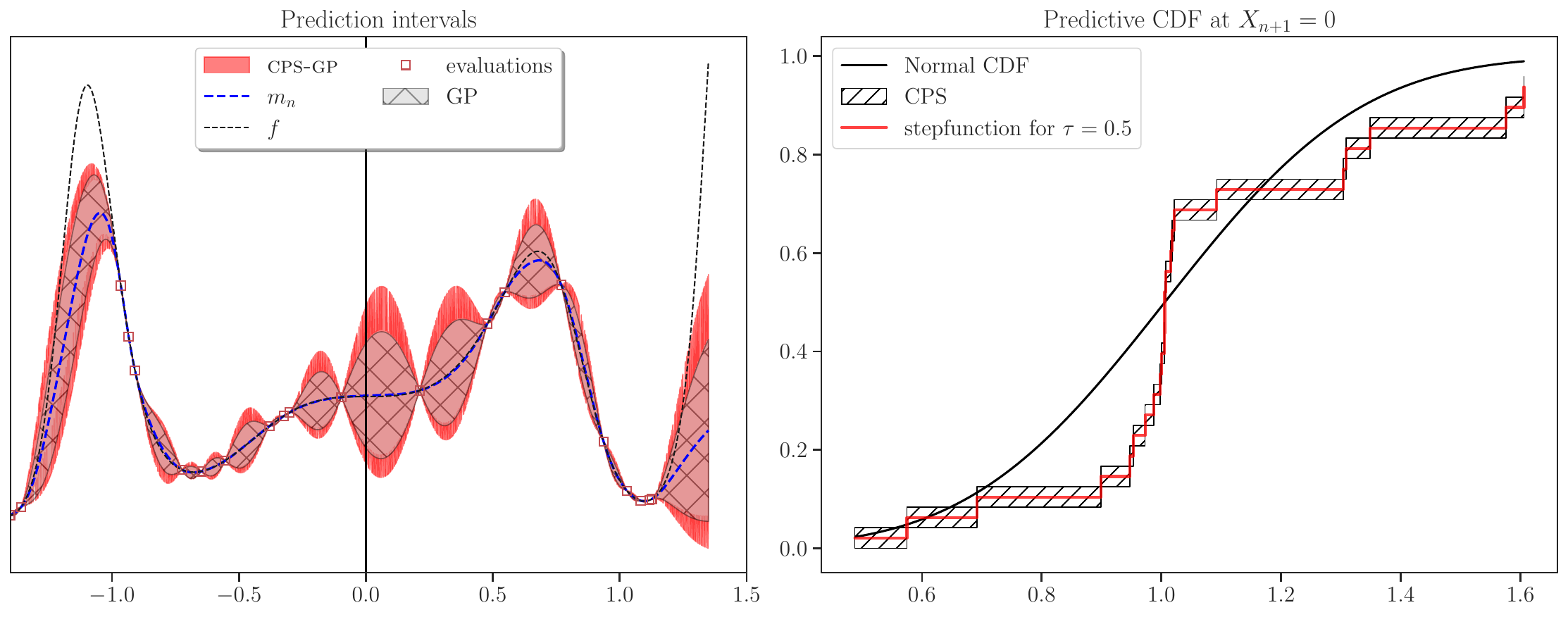}
  \caption{Stepwise CPD (with $\tau=0.5$) compared to the Gaussian
    posterior CDF at $x_{n+1}=0$.
    The hyperparameters of the GP are fitted on $\Dcal_n$ and then kept fixed.
    The CPD has discrete jumps at thresholds determined by GP interpolation,
    while the Gaussian posterior yields a smooth curve.}
  \label{fig:comp_cdf}
\end{figure}

\paragraph{Prediction intervals}
For a test location $x_{n+1}$ and a tie–breaking variable
$\tau \sim \U(0,1)$, the central $(1-\alpha)$ interval is defined as the
(randomized) half–open interval
$$
\Ccal_{n,\tau,1-\alpha}^{\mathrm{CPS\text{-}GP}}(x_{n+1})
= \Big[
  \bigl(\hat F_{n,\tau}^{\mathrm{CPS\text{-}GP}}\bigr)^{-1}(\alpha/2\mid x_{n+1}),~
  \bigl(\hat F_{n,\tau}^{\mathrm{CPS\text{-}GP}}\bigr)^{-1}(1-\alpha/2\mid x_{n+1})
\Big).
$$
Because $\hat F_{n,\tau}^{\mathrm{CPS\text{-}GP}}(\cdot\mid x_{n+1})$
is stepwise, each endpoint is a threshold (see Figure~\ref{fig:cpd-gp}).

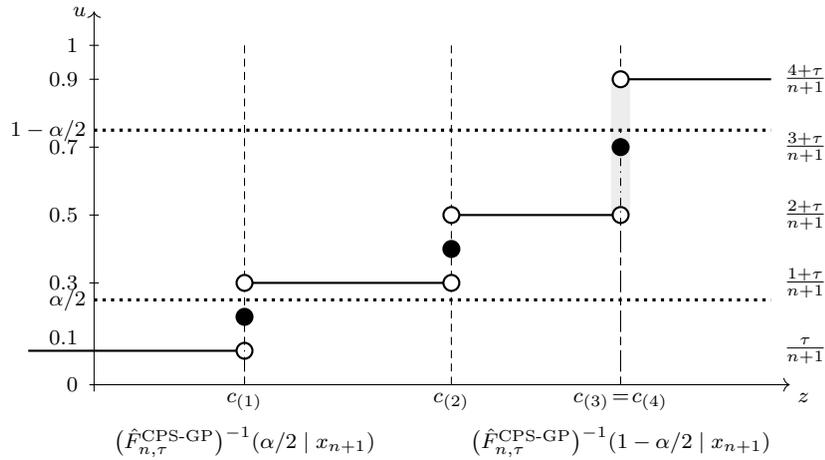
\begin{figure}[t]
  \centering
  \scriptsize
\begin{tikzpicture}[x=1.25cm,y=4.5cm]
  \tikzset{
    open/.style   = {draw, circle, fill=white, inner sep=0pt, minimum size=6pt, thick, transform shape=false},
    closed/.style = {draw, circle, fill=black, inner sep=0pt, minimum size=6pt, thick, transform shape=false}
  }

  \def\yA{0.10} \def\yB{0.30} \def\yC{0.50} \def\yD{0.70} \def\yE{0.90}
  \def\yCone{0.20}
  \def\yCtwo{0.40}

  \def\cA{1.6} \def\cB{3.8} \def\cT{5.6}

  \def\yal{0.25} \def\yone{0.75}

  \def\tiew{0.10}

  \draw[->] (-0.02,0) -- (7.4,0) node[below right] {$z$};
  \draw[->] (0,-0.02) -- (0,1.1) node[left] {$u$};

  \foreach \y/\lab in {0/$0$, 0.3/$0.3$, 0.5/$0.5$, 0.7/$0.7$, 0.9/$0.9$, 1/$1$} {
    \draw (-0.05,\y) -- (0.05, \y);
    \node[left=2pt] at (-0.03, \y) {\lab};
  }
  \draw (-0.05, 0.1) -- (0.05, 0.1);
  \node[left=2pt] at (-0.03, 0.14) {$0.1$};

  \draw[densely dashed] (\cA,0) -- (\cA,1);
  \draw[densely dashed] (\cB,0) -- (\cB,1);
  \draw[densely dashed] (\cT,0) -- (\cT,1);
  \node[below] at (\cA,0) {$c_{(1)}$};
  \node[below] at (\cB,0) {$c_{(2)}$};
  \node[below] at (\cT,0) {$c_{(3)}\!=\!c_{(4)}$};

  \fill[black!7] (\cT-\tiew,\yC) rectangle (\cT+\tiew,\yE);

  \draw[thick] (-0.7,\yA) -- (\cA,\yA); 
  \draw[thick] (\cA,\yB) -- (\cB,\yB);  
  \draw[thick] (\cB,\yC) -- (\cT,\yC);  
  \draw[thick] (\cT,\yE) -- (7.2,\yE);  

  \node[open]   at (\cA,\yA)   {};
  \node[open]   at (\cA,\yB)   {};
  \node[closed] at (\cA,\yCone){};

  \node[open]   at (\cB,\yB)   {};
  \node[open]   at (\cB,\yC)   {};
  \node[closed] at (\cB,\yCtwo){};

  \node[open]   at (\cT,\yC)   {};
  \node[open]   at (\cT,\yE)   {};
  \node[closed] at (\cT,\yD)   {};

  \draw[dotted,very thick] (0,\yal) -- (7.2,\yal);
  \draw[dotted,very thick] (0,\yone) -- (7.2,\yone);
  \node[left] at (0,\yal) {$\alpha/2$};
  \node[left] at (0,\yone) {$1-\alpha/2$};

  \draw[dash pattern=on 3pt off 2pt on 0.5pt off 2pt] (\cA,0) -- (\cA,\yal);
  \draw[dash pattern=on 3pt off 2pt on 0.5pt off 2pt] (\cT,0) -- (\cT,\yone);
  \node[below,align=center] at (\cA,-0.10)
    {$\bigl(\hat F_{n,\tau}^{\mathrm{CPS\text{-}GP}}\bigr)^{-1}(\alpha/2\mid x_{n+1})$};
  \node[below,align=center] at (\cT,-0.10)
    {$\bigl(\hat F_{n,\tau}^{\mathrm{CPS\text{-}GP}}\bigr)^{-1}(1-\alpha/2\mid x_{n+1})$};

  \node[anchor=west] at (7.25,\yA) {$\tfrac{\tau}{n+1}$};
  \node[anchor=west] at (7.25,\yB) {$\tfrac{1+\tau}{n+1}$};
  \node[anchor=west] at (7.25,\yC) {$\tfrac{2+\tau}{n+1}$};
  \node[anchor=west] at (7.25,\yD) {$\tfrac{3+\tau}{n+1}$};
  \node[anchor=west] at (7.25,\yE) {$\tfrac{4+\tau}{n+1}$};
\end{tikzpicture}
\caption{Conformal predictive distribution $\hat F_{n,\tau}^{\mathrm{CPS\text{-}GP}}(\cdot\mid x_{n+1})$ for fixed $\tau$.
On each open interval $(c_{(i)},c_{(i+1)})$ the value is $(i+\tau)/(n+1)$ (solid segments).
At a singleton threshold $c_{(i)}$ the actual value is $(i-1+2\tau)/(n+1)$ (filled circle), with open circles marking
the left and right limits. At a tie $c_{(i')}=\cdots=c_{(i'')}$, the value is $(i'-1+\tau(i''-i'+2))/(n+1)$ (filled circle),
again between the open-circle limits.}
\label{fig:cpd-gp}
\end{figure}

\begin{remark}
From the definition of the generalized inverse $F^{-1}(p)=\inf\{z:~F(z)\ge p\}$, we have
$$
\bigl(\hat F_{n,\tau}^{\mathrm{CPS\text{-}GP}}\bigr)^{-1}(p\mid x_{n+1})
= \min\Bigl\{\,c_{(r)}:~\hat F_{n,\tau}^{\mathrm{CPS\text{-}GP}}\bigl(c_{(r)}^{+}\mid x_{n+1}\bigr)~\ge~p\,\Bigr\},
\qquad p\in(0,1),
$$
where $F(z^{+}):=\lim_{t\downarrow 0}F(z+t)$.
Equivalently, since on $(c_{(i)},c_{(i+1)})$ the level is $(i+\tau)/(n+1)$,
$$
\bigl(\hat F_{n,\tau}^{\mathrm{CPS\text{-}GP}}\bigr)^{-1}(p\mid x_{n+1})
= c_{(i)},\qquad
i=\min\Bigl\{\,r\in\{1,\dots,n\}:~\frac{r+\tau}{n+1}~\ge~p\,\Bigr\}.
$$
If there is a tie block $c_{(i')}=\cdots=c_{(i'')}$, the left and right limits
at the block are $\frac{i'-1+\tau}{n+1}$ and $\frac{i''+\tau}{n+1}$,
respectively, and
$$
\bigl(\hat F_{n,\tau}^{\mathrm{CPS\text{-}GP}}\bigr)^{-1}(p\mid x_{n+1}) \;=\; c_{(i')}
\quad\text{for}\quad
p\in\Bigl(\tfrac{i'-1+\tau}{n+1},\,\tfrac{i''+\tau}{n+1}\Bigr].
$$
(Any index in $\{i',\dots,i''\}$ yields the same threshold since
$c_{(i')}=\cdots=c_{(i'')}$.)  
\end{remark}

By probabilistic calibration with randomized tie–breaking, the interval
$\Ccal_{n,\tau,1-\alpha}^{\mathrm{CPS\text{-}GP}}(X_{n+1})$ achieves exact marginal coverage:
$$
\P\left\{\,Z_{n+1}\in \Ccal_{n,\tau,1-\alpha}^{\mathrm{CPS\text{-}GP}}(X_{n+1})\,\right\}
= 1-\alpha,
\quad\text{averaged over } (\Dcal_n,X_{n+1},\tau).
$$
Replacing the right-open interval $[a,b)$ by the closed interval $[a,b]$
yields conservative coverage, with equality restored by boundary randomization
as in Proposition~\ref{prop:randomized-pit}.

\begin{remark}
Fixing $\tau$ yields a deterministic, stepwise predictive
distribution. Averaging over $\tau$,
$$
\bar F_n(z\mid x_{n+1}) \;=\; \E_{\tau}\left[\hat F_{n,\tau}^{\mathrm{CPS\text{-}GP}}(z\mid x_{n+1})\right],
$$
keeps the same jump locations and replaces each plateau level by its
mean over $\tau$; it does not remove discontinuities in
$z$. Exact PIT uniformity requires random~$\tau$.
\end{remark}

\begin{proposition}[Finite bounds]
\label{prop:cps-gp-finite}
Assume the thresholds $\{c_i\}$ are all distinct and fix $\tau\in(0,1)$.  
Then the endpoints of $\Ccal_{n,\tau,1-\alpha}^{\mathrm{CPS\text{-}GP}}(x_{n+1})$
are finite if and only if
$$
\frac{\alpha}{2} > \frac{\tau}{n+1}
\quad\text{and}\quad
\frac{\alpha}{2} \ge \frac{1-\tau}{n+1},
$$
or equivalently,
$$
\alpha \;\ge\; \frac{2}{n+1}\,\max\{\tau,\,1-\tau\},
$$
with strict inequality required when the maximum is $\tau$.
\end{proposition}
\begin{proof}
  See Appendix~\ref{sec:cps-gp-finite-bounds}.
\end{proof}

This shows that for small $n$ or small $\alpha$, CPS intervals may
have infinite endpoints, a common feature of conformal methods.

\paragraph{Computational complexity}

We now examine the computational complexity of CPS--GP.

\begin{proposition}
\label{prop:complexity}
Let $n$ denote the sample size and $m$ the number of prediction points.  
The CPD based on GP interpolation can be computed with
$$
O(n^3) \text{ precomputation} \quad\text{and}\quad O(n^2 + n\log n) \text{ per prediction point.}
$$
The total computational cost is therefore
$$
O(n^3) \;+\; m\,O(n^2 + n\log n).
$$
If only interval endpoints are needed (without constructing the full stepwise CPD),
the sorting step can be omitted, reducing the per–point cost to $O(n^2 + n)$.
\end{proposition}

Once the Cholesky factorization of $K_n$ is computed at cost $O(n^3)$,
evaluating the GP posterior mean and variance at a new location $x_{n+1}$ costs $O(n^2)$.  
\cpsgp involves the same $O(n^2)$ operations, plus $O(n)$ for computing
the thresholds $c_i$ and $O(n\log n)$ for sorting them.
Thus, the dominant per–point complexity remains $O(n^2)$,
with an additional $O(n\log n)$ term for sorting beyond standard GP prediction.

\maybeClearpage

\section{Bayesian calibration using a generalized normal distribution}
\label{sec:calibr-using-gener}

\subsection{Generalized normal model for prediction errors}
\label{sec:gener-norm-model}

In Section~\ref{sec:cps-gp-interpolation}, leave--one--out residuals were treated
nonparametrically to construct a conformal predictive system. An alternative is
to regard these residuals as samples from a parametric model and to apply a
post--hoc calibration of the GP posterior. The posterior mean is kept as the
point predictor, while the predictive variability is adjusted by fitting a
generalized normal model to the residuals. In this way, one obtains explicit
control over the tail behavior and predictive distributions that are, in
principle, easier to use within GP--based sequential design algorithms. This
parametric modeling of the residuals forms the basis of the \bcrgp
method (Bayesian--calibrated residuals for Gaussian processes).

In more detail, let $X\sim\mu$ be a random input, where $\mu$ is a fixed design
measure on~$\XX$. As above, we
consider a Gaussian process prior on the unknown function
$f:\XX\to\R$, $\xi\sim\GP(m,k)$, and regard $f$ as a fixed
realization of~$\xi$. Let $X_1,\ldots,X_n \stackrel{\text{i.i.d.}}{\sim}\mu$
be input locations. The observed dataset is then
$\Dcal_n=\{(X_i, f(X_i))\}_{i=1}^n$. For any $x\in\XX$, the GP posterior has mean
$m_n(x)$ and standard deviation $\sigma_n(x)$. Define the normalized prediction
error
$$
R_n(x,f(x))=\frac{f(x)-m_n(x)}{\sigma_n(x)},\qquad
\text{with }R_n(x,f(x))=0\text{ if }\sigma_n(x)=0.
$$
The \emph{$\mu$-marginal distribution} of the normalized error
is the distribution of $R_n(X,f(X))$ when $X\sim\mu$ and
$\Dcal_n$ is fixed. This is the distributional object that will be
calibrated.

For context, under the Bayesian GP model with fixed hyperparameters and for any
fixed $x$ with $\sigma_n(x)>0$, the posterior predictive error satisfies
$$
  R_n(x,\xi(x)) \mid (x,\Dcal_n) \;\sim\; \mathcal N(0,1).
$$
Thus $R_n(x,\xi(x))$ is a pivot: its conditional distribution does not
depend on the parameters of the GP prior $\xi$ (with the convention
$R_n(x,\cdot)=0$ when $\sigma_n(x)=0$).

This pointwise pivot property does not imply that the normalized
residual $R_n(X,f(X))$ is Gaussian when $f$ is a fixed but unknown sample path
and $X\sim\mu$. For such an $f$, $R_n(x,f(x))$ is deterministic at each $x$, and
the only randomness comes from sampling $X$ according to $\mu$. The resulting
distribution of $R_n(X,f(X))$ aggregates these deterministic values across $x$
and generally deviates substantially from $\mathcal N(0,1)$, especially in the
tails, as illustrated in Figure~\ref{fig:qq_plot}.

Nevertheless, normalized residuals remain the natural object to model: the pivot
property justifies using them as scale--free standardized errors, while their
$\mu$-marginal distribution, though not Gaussian for a fixed $f$, captures the
relevant predictive variability to be calibrated. Our objective is to build a
parametric approximation of this distribution.

In practice however, the $\mu$-marginal distribution of $R_n(X,f(X))$, $X\sim\mu$,
is not directly observable.  To approximate it, we use the
leave--one--out standardized residuals
$$
R_{n,-i}=\frac{Z_i-m_{n,-i}(X_i)}{\sigma_{n,-i}(X_i)},\qquad i=1,\ldots,n,
$$
where $m_{n,-i}$ and $\sigma_{n,-i}$ are computed from
$\Dcal_n\setminus\{(X_i,Z_i)\}$. We adopt the working assumption that
the conditional empirical distribution of $\{R_{n,-i}\}_{i=1}^n$,
given $\Dcal_n$, provides a proxy for the $\mu$-marginal distribution
of $R_n(X,f(X))$.  This assumption is plausible when the GP fit is
stable so that replacing $(m_n,\sigma_n)$ by
$(m_{n,-i},\sigma_{n,-i})$ has negligible global effect. Intuitively,
each $R_{n,-i}$ is an out-of-sample standardized error at a location
distributed like a future input.

To obtain a tractable post-hoc calibration, we fit a parametric model to this
proxy distribution. Specifically, we use the generalized normal (GN) family
$$
\mathcal{GN}(\beta,0,\lambda),\qquad \beta>0,~\lambda>0,
$$
with density given in Appendix~\ref{apd:gnp}. The scale $\lambda$
controls dispersion, while the shape parameter $\beta$ tunes the tail
behavior: $\beta=2$ corresponds to the Gaussian case, and $\beta=1$ to
the Laplace case. The GP posterior mean $m_n$ is retained as the point
predictor, and uncertainty is recalibrated through the GN fit. This
yields a closed-form predictive distribution that is easier to handle
than a purely empirical residual distribution and integrates naturally into
GP-based sequential search algorithms.

Figure~\ref{fig:qq_plot} illustrates the motivation: empirical quantiles
of leave-one-out residuals typically diverge from Gaussian quantiles,
while a GN fit adapts to tail deviations more accurately.

\begin{figure}[h!]
  \centering
  \includegraphics[width=\textwidth]{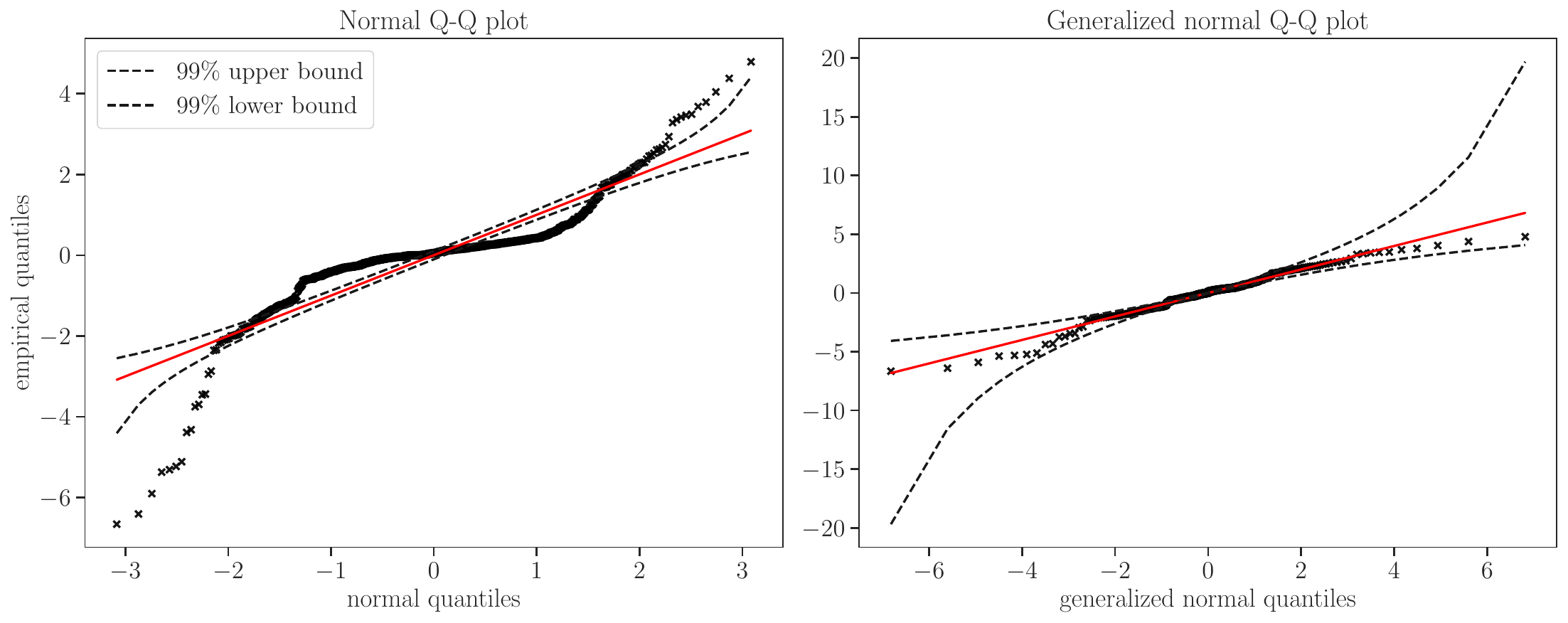}
  \caption{Empirical quantiles of standardized residuals
  $R = (Z - m_n(X))/\sigma_n(X)$ at $1000$ test points, compared with those of
  the standard normal and a fitted generalized normal distribution. Shaded bands
  indicate $99\%$ pointwise confidence intervals. Example: Goldstein--Price
  function in $d=2$, $n=40$.}
  \label{fig:qq_plot}
\end{figure}

\paragraph{Remarks} 
If the design points are not exactly $\mu$-distributed, the same idea can be 
applied by reweighting the residuals according to importance weights that map 
the empirical design distribution to~$\mu$.  

Alternative residual families, such as the Student-$t$, could also be 
considered. The generalized normal is adopted here as a convenient 
two-parameter choice, spanning Gaussian to Laplace tails, with closed-form 
CDF and straightforward estimation.

\subsection{Bayesian parameter selection}
\label{sec:bayes-param-select}

We now detail the Bayesian parameter selection strategy that completes the 
\bcrgp method.
The generalized normal model involves two free parameters, the shape $\beta$ 
and scale $\lambda$, which we group as $\theta=(\beta,\lambda)$. These must be 
estimated from the leave-one-out residuals to obtain a predictive distribution. 
Rather than relying on point estimation (e.g., maximum likelihood), we adopt a 
Bayesian framework: a posterior distribution over $\theta$ is constructed from 
the data, and a selection rule is applied to choose a representative parameter. 
This approach connects to the idea of \emph{Bayesian tolerance intervals} 
\citep{meeker2017statistical}, where posterior draws are used to control 
conservatism in uncertainty quantification. Two complementary rules are 
considered below, depending on whether the goal is conservative coverage or 
$\mu$-probabilistic calibration.

Formally, we place weakly informative priors
$$
  \beta \sim \U(0,a), 
  \qquad 
  \lambda \sim \U(0,b),
$$
and define the posterior given the LOO residuals $\{R_{n,-i}\}_{i=1}^n$ as
$$
  p(\theta \mid \{R_{n,-i}\}_{i=1}^{n})
  \;\propto\;
  \prod_{i=1}^n p(R_{n,-i} \mid \theta)\,
  \one_{\{0<\beta<a\}}\,\one_{\{0<\lambda<b\}}.
$$
Posterior draws $\theta_j=(\beta_j,\lambda_j)$, $j=1,\ldots,K$, are obtained by 
MCMC sampling. Two rules for selecting $\theta^*$ are then defined.

\paragraph{Rule 1: Conservative prediction}
For each posterior draw $\theta_j=(\beta_j,\lambda_j)$, compute the variance
$$
v_j \;=\; \Var\big(\mathcal{GN}(\theta_j)\big) 
\;=\;
\lambda_j^{2}\,\frac{\Gamma(3/\beta_j)}{\Gamma(1/\beta_j)}.
$$
The values $\{v_j\}_{j=1}^K$ form a Monte Carlo sample from the posterior
distribution of $\Var(\mathcal{GN}(\theta))$ given the LOO residuals.
Fix a tolerance level $\delta\in(0,1)$ (small in practice) and select
$\theta^*$ such that
$v^*=\Var(\mathcal{GN}(\theta^*))$ is the $(1-\delta)$-quantile of this
distribution.  

The rationale is that the true parameter $\theta$ is unknown. Selecting the
$(1-\delta)$-quantile $v^*$ of the posterior variance provides an
upper credible bound: with posterior probability at least $1-\delta$, the
variance of the true residual distribution is no larger than $v^*$. 
Equivalently, the resulting predictive distribution is at least as dispersed as
the true one with probability $1-\delta$. Only a fraction $\delta$ of posterior
draws correspond to even more dispersed (more conservative) models. Hence,
smaller values of $\delta$ yield more pessimistic (wider) predictions, whereas
larger values produce less conservative ones. This quantile rule parallels the
construction of Bayesian tolerance bounds in \cite{meeker2017statistical}.

\begin{figure}[h!]
  \centering
  \includegraphics[width=\textwidth]{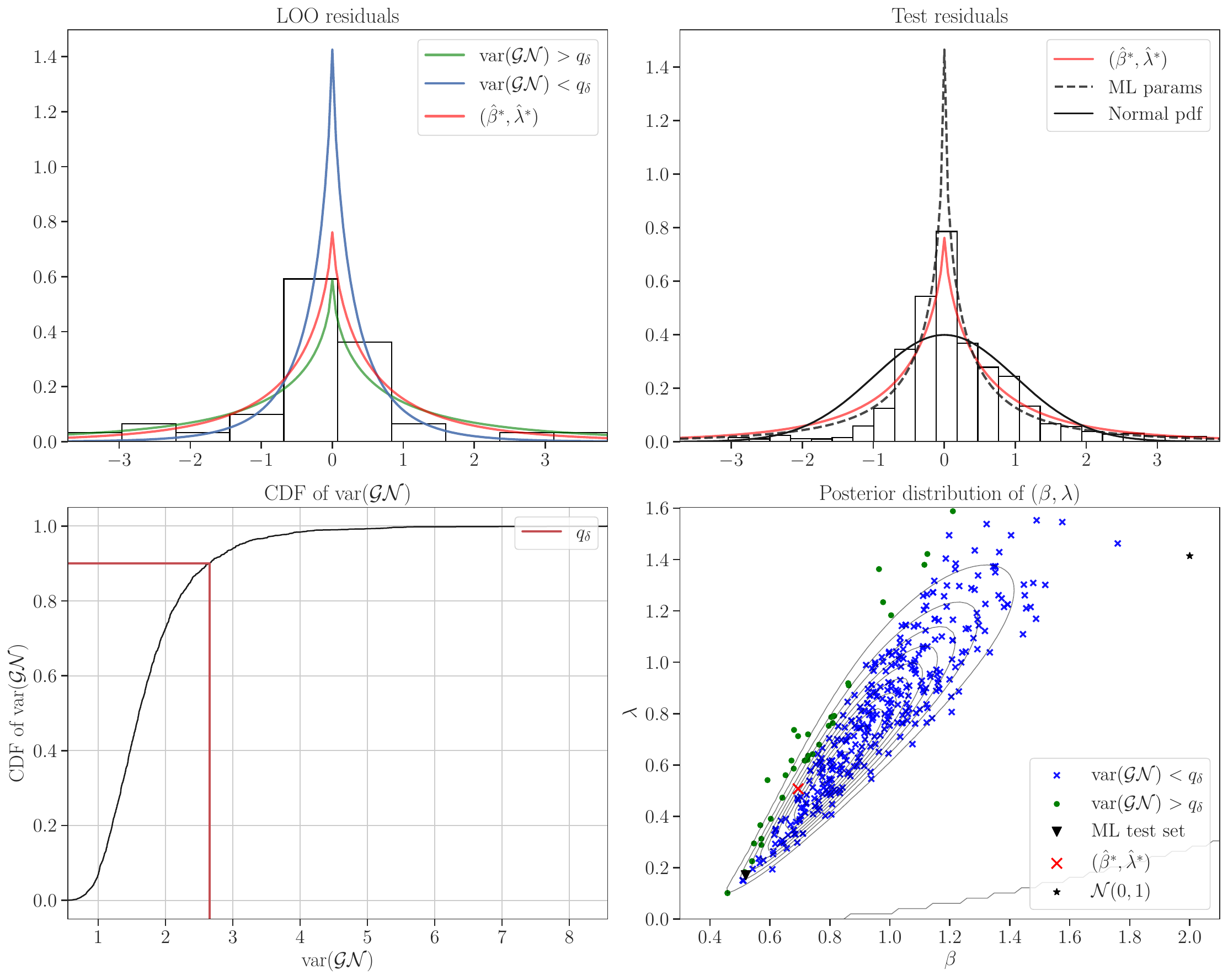}
  \caption{Illustration of the \bcrgp method on the Goldstein--Price 
  function with $n=40$ observations. \textbf{First row}: Left-- Histogram of 
  the LOO normalized residuals overlaid with predictive density functions 
  corresponding to various parameter configurations. The red curve corresponds 
  to the pdf with parameters $(\hat{\beta}^*, \hat{\lambda}^*)$; the green and 
  blue curves represent more and less conservative forecasts, respectively. 
  Right-- Histogram of a sample of normalized residuals on a test set, where 
  the dashed line represents the maximum likelihood fit, the red curve the pdf of the 
  selected model and the black curve to the normal model. \textbf{Second row}:  
  Left-- CDF of variances from 3000 posterior samples. The selected $1-\delta$ 
  quantile is marked. Right: posterior samples in the $(\beta, \lambda)$ space, 
  colored by variance: more conservative than the selected pair $(\hat{\beta}
  ^*, \hat{\lambda}^*)$ (green), less conservative (blue). The selected pair $
  (\hat{\beta}^*, \hat{\lambda}^*)$ is shown as a red cross; the maximum 
  likelihood estimate (black triangle) is unavailable in practice; the 
  parameter for normal model ($(\beta, \lambda) = (2, \sqrt{2})$) is shown as 
  the black star.}
  \label{fig:posterior_params}
\end{figure}

\paragraph{Rule 2: Probabilistic calibration}
Draw posterior samples $\{\theta_j\}_{j=1}^K$.  
For any parameter $\theta$, let $F_\theta$ denote the CDF of
$\mathcal{GN}(\theta)$ and $F_\theta^{-1}$ its quantile function.  

Consider two models $F_{\theta_i}$ (generating) and $F_{\theta_j}$ (forecasting).
For a residual $R \sim F_{\theta_i}$, the PIT under $F_{\theta_j}$ is
$U = F_{\theta_j}(R)$, whose CDF is
$$
G_{j\mid i}(u)
=
\P\big(F_{\theta_j}(R)\le u\big)
=
F_{\theta_i}\big(F_{\theta_j}^{-1}(u)\big),
\qquad u\in[0,1].
$$
The corresponding cross-posterior KS--PIT discrepancy is
$$
D_{j\mid i}
=
\sup_{u\in[0,1]}
\lvert G_{j\mid i}(u)-u\rvert
=
\sup_{u\in[0,1]}
\Big\lvert F_{\theta_i}\big(F_{\theta_j}^{-1}(u)\big)-u\Big\rvert,
\qquad i\neq j.
$$
Since $u = F_{\theta_j}(z)$ is a one-to-one mapping on~$\R$, this is
equivalently the Kolmogorov distance between the two CDFs:
$$
D_{j\mid i}
=
\sup_{z\in\R}
\lvert F_{\theta_i}(z)-F_{\theta_j}(z)\rvert.
$$

Fix a small level $\delta\in(0,1)$ and compute
$$
T_j(1-\delta)
=
\mathrm{Quantile}_{\,1-\delta}\big(\{D_{j\mid i}\mid i\neq j\}\big).
$$
Here, $T_j(1-\delta)$ is an upper posterior credible bound on the
calibration error of model $F_{\theta_j}$: with posterior probability at least
$1-\delta$, the KS--PIT discrepancy between $F_{\theta_j}$ and a
posterior-plausible generating model does not exceed $T_j(1-\delta)$.  
This quantity controls the worst-case lack of calibration over the central
$(1-\delta)$ fraction of posterior uncertainty.

Select
$$
\theta^* = \argmin_{j=1,\ldots,K} T_j(1-\delta).
$$
The selected parameter $\theta^*$ minimizes this upper bound, yielding the
forecast that remains closest to uniform calibration under nearly all
posterior-plausible residual distributions.  
When the generalized normal family is well specified and the data are
informative, the posterior concentrates and the selected $\theta^*$ yields PIT
values close to uniform.  
Under model misspecification, the rule acts as a robust Bayesian procedure by
controlling a high-probability upper bound on the calibration loss.

\begin{remark}[KS--PIT, sharpness, and posterior constraint.]
Section~\ref{sec:bg-pit} shows that $\mu$-proba\-bilistic calibration alone does
not preclude uninformative predictions: even for a constant $f$, many
predictive families can yield $U_{\hat F_n}^{f,\mu}\sim \U(0,1)$ while differing
substantially in dispersion. KS--PIT targets calibration only and ignores
sharpness.

Rule~2 optimizes a calibration discrepancy \emph{under the posterior constraint}:
the search is restricted to residual models supported by the posterior built
from the LOO residuals, so the selected distribution is not arbitrary. This does
not, however, control sharpness or task-specific risk. In practice, one may
augment or replace the KS--PIT objective by a proper scoring rule (e.g., NLPD,
CRPS, SCRPS) or by a calibration functional tailored to the application
(e.g., excursion-probability calibration, tail coverage at level $1-\alpha$).
Formally, the same quantile-robust selection applies: replace $D_{j\mid i}$ by a
loss $L_{j\mid i}$ of interest and minimize an upper posterior quantile
$\mathrm{Quantile}_{\,1-\delta}(\{L_{j\mid i}\mid i\neq j\})$ over $j$, or
combine Rule~2 with Rule~1 to enforce a lower bound on dispersion (equivalently,
an upper bound on sharpness) while keeping PIT-based calibration under posterior
constraints.
\end{remark}

\paragraph{Predictive distribution}
Once a parameter $\theta^*=(\beta^*,\lambda^*)$ has been selected, define
$G_{\beta,\lambda}$ as the CDF of the $\mathcal{GN}(\beta,0,\lambda)$
distribution and $g_{\beta,\lambda}$ as its density. For any $x\in\XX$, the
\bcrgp predictive distribution is then
\begin{equation}
\label{eq:bcrgp-cdf}
\hat F_n^{\mathrm{BCR}}(z\mid x)
\;=\;
G_{\beta^*,\lambda^*}\left(
  \frac{z - m_n(x)}{\sigma_n(x)}
\right),
\qquad z\in\R,
\end{equation}
with corresponding predictive density
\begin{equation}
\label{eq:bcrgp-pdf}
\hat f_n^{\mathrm{BCR}}(z\mid x)
\;=\;
\frac{1}{\sigma_n(x)}\,
g_{\beta^*,\lambda^*}\left(
  \frac{z - m_n(x)}{\sigma_n(x)}
\right).
\end{equation}
The GP posterior mean $m_n(x)$ is preserved as the point predictor, while the
uncertainty is recalibrated through the selected parameter $\theta^*$, which
governs both tail shape and dispersion.

\maybeClearpage
\section{Numerical study}
\label{sec:numerical}

\subsection{Outline}
\label{sec:outline}

We conduct a numerical comparison between \bcrgp~and \cpsgp, together
with the full conformal approach of \citet{Papadopoulos_2024} and the Jackknife+
for GP (\jpgp) method of \citet{jaber_conformal_2024}, on a set of benchmark 
functions. Section~\ref{sec:experimental-setup} details the experimental 
protocol and evaluation metrics. 
Section~\ref{sec:variance_tolerance} studies the influence of the tolerance 
parameter~$\delta$ in the variance-based \bcrgp. 
Section~\ref{sec:design_size} investigates the effect of the number of design 
points on empirical coverage and interval width, with emphasis on parameter 
selection in \cpsgp. 
Section~\ref{sec:res_coverage} compares the coverage and width of prediction 
intervals across methods, and 
Section~\ref{sec:res_pred_distr} analyzes their predictive distributions using 
KS--PIT and SCRPS.

\subsection{Experimental setup}
\label{sec:experimental-setup}

Predictions are based on a Gaussian process model~$\xi$
with a constant mean function and an anisotropic Matérn covariance kernel.
Uncertainty quantification is then constructed using one of four procedures:
\bcrgp, \cpsgp, \jpgp, or \fcp.

All methods are implemented using the Gaussian process micro 
package (\texttt{gpmp}) \citep{gpmp2025} for Gaussian-process computations in 
Python\footnote{The code used to generate all results in this 
section will be made publicly available in a future version of this article.}.
The implementation of \fcp\ adapts the public code released by 
\citet{Papadopoulos_2024}.

In the experiment below, the GP prior has a constant mean $m$ and an
anisotropic Matérn kernel covariance. The anisotropic Matérn kernel is defined 
as
$$
  k_{\sigma, \nu, \rho}(x, y) = \sigma^2 \kappa_{\nu}(h), \qquad
  h^2 = \sum_{i=1}^d \frac{(x_{[i]} - y_{[i]})^2}{\rho_i^2}, \quad x, y \in \R^d,
$$
where $\sigma^2$ is the variance parameter,
$\rho = (\rho_1, \ldots, \rho_d)$ are component-wise lengthscales, and
$\kappa_\nu$ denotes the Matérn correlation function as defined 
in~\cite[Chapter~2.7]{stein_interpolation_1999}. The smoothness parameter
is set to $\nu = p + 1/2$ with $p \in \N^\star$.

The GP hyperparameters $(m, \sigma^2, \rho)$ are selected by maximum likelihood 
on the dataset~$\Dcal_n$. In some experiments, they are instead selected on an 
independent dataset to assess the effect of data reuse on 
calibration.

For each test function of input dimension~$d$, a design set~$\Dcal_n$ of size 
proportional to~$d$ is constructed, and the GP posterior is evaluated on 
$n_{\mathrm{test}} = 4000$ test inputs sampled uniformly from the domain~$\XX$.
Each experiment is repeated $100$ times. The hyperparameters 
of the posterior distribution of $(\beta, \lambda)$ in \bcrgp\ are fixed to 
$a = 10$ and $b = 10$. The analytical definitions and properties of the test 
functions are summarized in Table~\ref{tab:test_functions} in 
Appendix~\ref{ap:test_functions}.

As an initial illustration, Figure~\ref{fig:one_d_example} presents a 
one-dimensional comparison of the prediction intervals produced by 
\bcrgp with variance-based selection, and \cpsgp.

\begin{figure}[htbp]
  \centering
  \includegraphics[width=\textwidth]{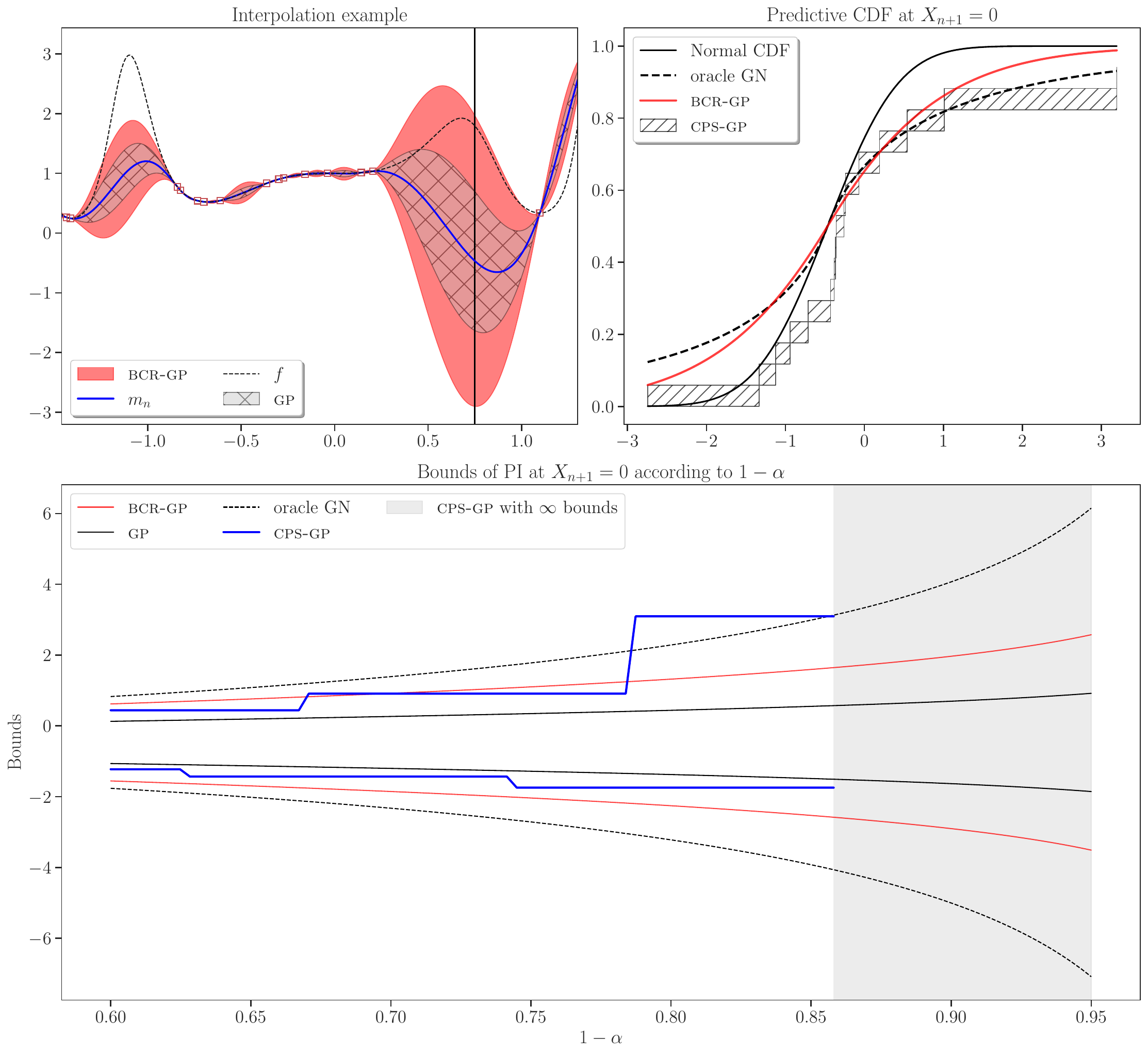}
  \caption{
    \textbf{Top left:} Prediction intervals constructed from the
    GP posterior distribution and from \bcrgp\ at confidence level
    $1 - \alpha = 0.9$. \bcrgp\ uses the variance of the generalized
    normal distribution for selection, with $\delta = 0.1$.
    \textbf{Top right:} Predicted CDFs at $x = 0.75$, comparing the
    GP posterior, the CDF from \bcrgp\ (red), the stepwise CDF from
    \cpsgp\ (black hatches), and an oracle CDF (black) obtained
    from a generalized normal model fitted on a test grid of
    $n_{\mathrm{test}} = 2000$ points. This dataset exhibits strong
    miscalibration of the GP posterior predictive distributions.
    \textbf{Bottom:} Interval bounds as a function of $1 - \alpha$.
    The GP posterior underestimates uncertainty across confidence
    levels, while \bcrgp\ and \cpsgp\ produce larger
    intervals. \cpsgp\ yields unbounded interval widths for
    $1 - \alpha \gtrsim 0.85$.}
  \label{fig:one_d_example}
\end{figure}

\subsection{Influence of the tolerance level for variance-based \bcrgp}
\label{sec:variance_tolerance}

We study the effect of the tolerance~$\delta$ on the variance-based
version of \bcrgp, where selection relies on the variance of
$\mathcal{GN}(\beta,0,\lambda)$, Section~\ref{sec:bayes-param-select},
Rule~1, on~$\Dcal_n$. Results are shown in
Figure~\ref{fig:conf_level}.

For small tolerances, \bcrgp\ produces conservative prediction
intervals at all confidence levels. As~$\delta$ increases, intervals
become narrower and undercoverage appears, particularly in the tails
(e.g., at 90–95\% central levels).  The transition from conservative
to optimistic behavior depends on the problem, the particular
design and the GP specification. In these experiments, values
of~$\delta$ around~0.1 provided a good balance between coverage and
interval width. When a single value must be fixed in practice, $\delta$ can be tuned by 
splitting $\Dcal_n$ into calibration and validation subsets, or via a
$K$-fold cross-validation scheme applied to $\Dcal_n$.

\begin{figure}[htbp]
  \centering
  \begin{subfigure}[b]{\textwidth}
    \centering
    \includegraphics[width=\textwidth]{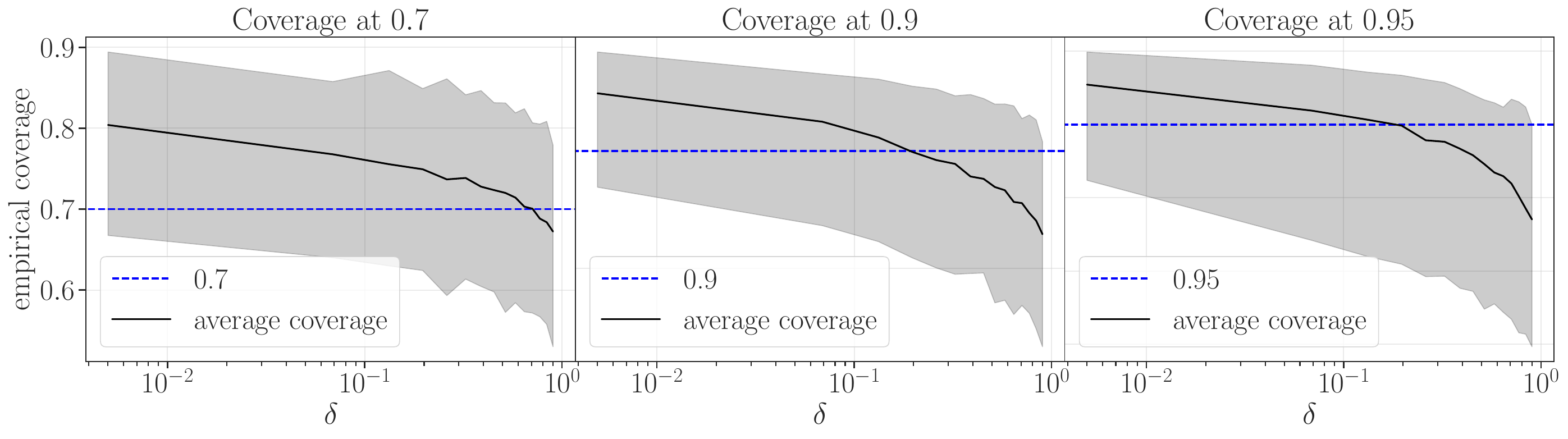}
    \caption{Empirical coverage for $\delta \in [0,1]$ for the Goldstein--Price function.}
    \label{fig:conf_level_goldstein}
  \end{subfigure}

  \begin{subfigure}[b]{\textwidth}
    \centering
    \includegraphics[width=\textwidth]{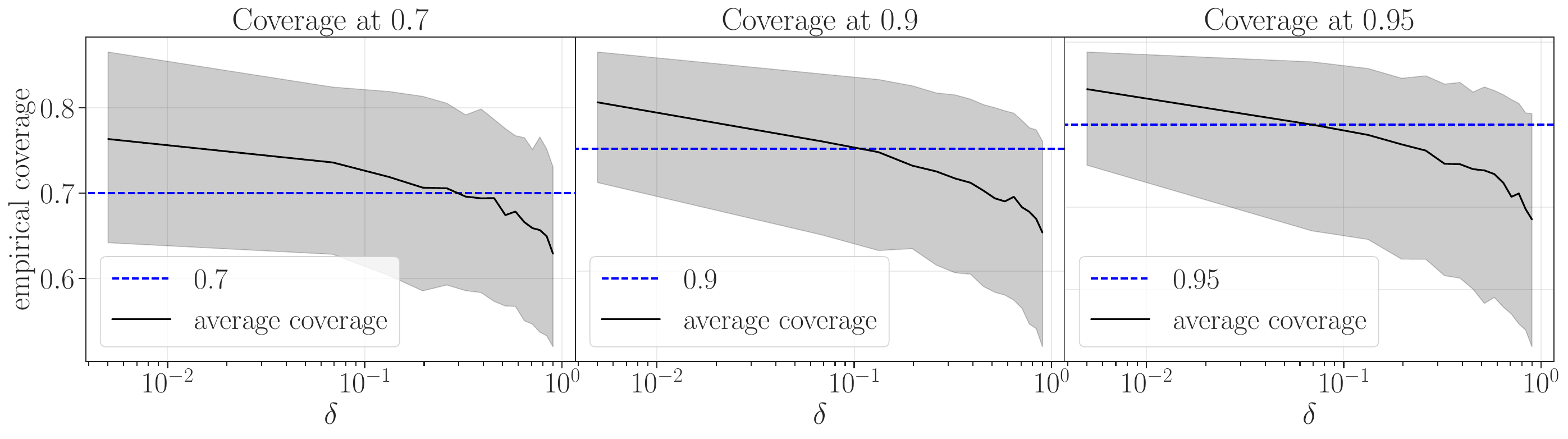}
    \caption{Empirical coverage for $\delta \in [0,1]$ for the Dixon--Price function.}
    \label{fig:conf_level_dixon}
  \end{subfigure}

  \caption{Empirical coverage of 70\%, 90\%, and 95\%  prediction intervals obtained 
  with the variance-based \bcrgp\ method for $\delta \in [0,1]$. 
  The regularity parameter is fixed to $p = 2$. 
  The top panel corresponds to the Goldstein--Price function ($n = 40$), and the 
  bottom panel to the Dixon--Price function ($n = 60$). 
  Evaluation is performed on 4000 test points. 
  The black curve indicates the mean coverage across 100 repetitions, and the 
  shaded area denotes the 0.05 and 0.95 quantiles.}
  \label{fig:conf_level}
\end{figure}

\subsection{Effect of design set size}
\label{sec:design_size}

This section examines how the number of design points~$n$ influences empirical
coverage and the average width of prediction intervals. We first analyze this
effect in relation to parameter selection strategies in the \cpsgp\ method,
since its theoretical guarantees depend on the independence between parameter
selection and prediction.

\paragraph{Parameter selection strategies for \cpsgp}

The experiment was conducted on
several test functions; results are reported here for the Hartmann6 function,
which represents the most challenging case for \cpsgp\ among those considered.
This benchmark is known to require a large number of design points for accurate
parameter selection and approximation.

For each dataset~$\Dcal_n$, kernel parameters are either  
(i)~selected using a subset of~$\gamma n$ observations and predictions computed
on the remaining $(1-\gamma)n$ points (split strategy), or  
(ii)~selected and used on the full dataset~$\Dcal_n$.  
As a reference, an oracle configuration is also considered, where parameters
are selected on an independent dataset of~$10d$ points.
Figure~\ref{fig:n_values_fixed_param} reports empirical coverages obtained on an
independent test grid.

When parameters are selected on an independent dataset, \cpsgp\ achieves
empirical coverage whose average over repetitions matches the intended
confidence levels, consistent with its theoretical guarantees. In contrast,
selecting parameters on the same data used for prediction leads to poorer
coverage, especially for small~$n$. This effect diminishes as~$n$ increases:
beyond roughly~$20d$, results with or without data splitting become nearly
identical. The improvement reflects the decreasing sensitivity of the selected
hyperparameters to sample perturbations as the design size grows, which reduces
the impact of the exchangeability violation.

It is worth noting that the Hartmann6 function represents a difficult case for
\cpsgp, as parameter selection tends to be unstable for small~$n$. For most
other test functions considered in this study, \cpsgp\ is less sensitive to the
design size and achieves calibration more easily.

Because data splitting reduces the number of points available for both parameter 
selection and prediction, some loss in performance is expected. 
Table~\ref{table:rmse_cps_gp} reports average RMSE and KS--PIT values for 
different split ratios~$\gamma$. As expected, RMSE increases when fewer points 
are used for prediction, with the largest degradation at $\gamma=0.5$.
Calibration, measured by KS--PIT, deteriorates most for $\gamma=0.8$. 
Using the same dataset~$\Dcal_n$ for both steps yields RMSE values close to the 
oracle configuration and calibration comparable to the best split strategy, 
though the average coverage remains biased.
In practice, since any deterioration in RMSE relative to the standard GP model 
is undesirable, we adopt this configuration in the remainder of the paper to 
ensure a fair comparison across methods.

\begin{figure}[htbp]
  \centering
  \includegraphics[width=0.95\textwidth]{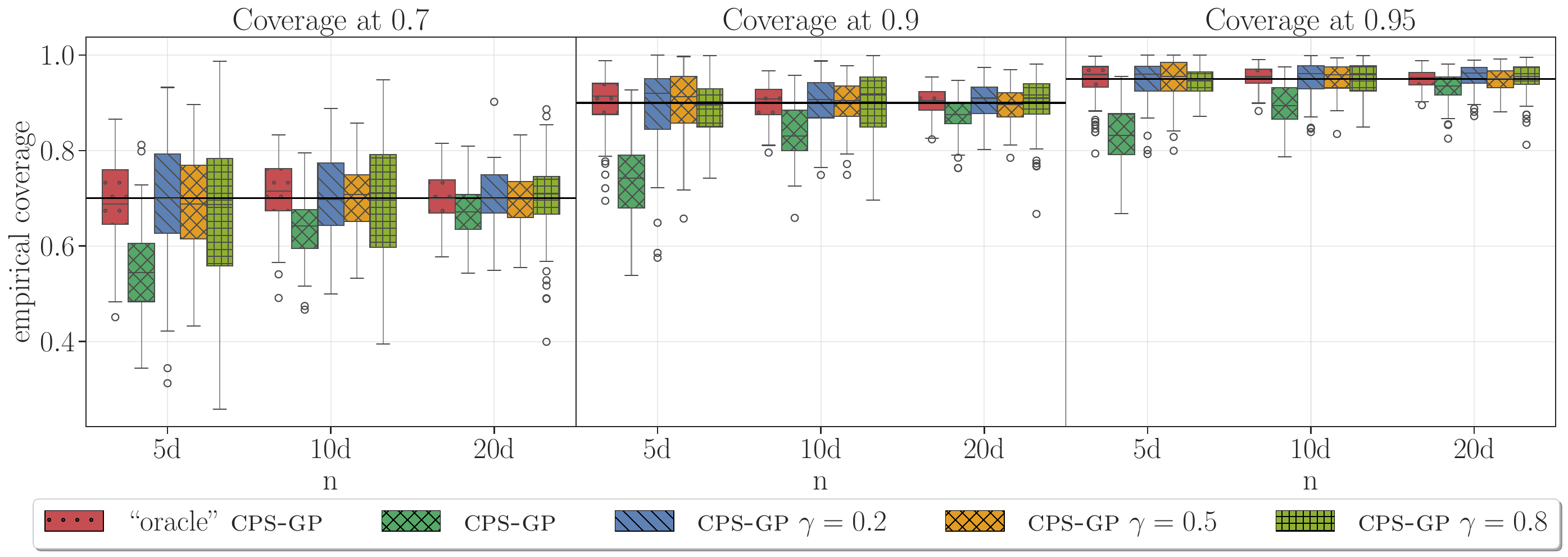}
  \caption{Empirical coverage of 70\%, 90\%, and 95\% prediction
    intervals for the Hartmann6 function as a function of the design
    set size~$n$. Horizontal lines indicate nominal coverage.
    Evaluation is based on 4000 test points and 100 repetitions.  The
    GP regularity is fixed at $p=2$.  \cpsgp\ is applied using either
    a split of $\gamma n$ points for parameter selection and
    $(1-\gamma)n$ for prediction, or the same dataset~$\Dcal_n$ for
    both steps. The oracle \cpsgp\ variant uses parameters
    pre-selected on an independent dataset of~$10d$ points. }
  \label{fig:n_values_fixed_param}
\end{figure}

\begin{table}[htbp]
  \centering
  \setlength\tabcolsep{4pt}
  \begin{tabularx}{\textwidth}{l *{6}{X}}
      \toprule
      \textbf{Method} & \multicolumn{2}{c}{$5d$} & \multicolumn{2}{c}{$10d$} & \multicolumn{2}{c}{$20d$} \\
      ~ & \tiny KS--PIT & \tiny RMSE & \tiny KS--PIT & \tiny RMSE & \tiny KS--PIT & \tiny RMSE \\
      \hline \hline
      oracle \textsc{cps-gp}       & 0.14   & 0.35 & 0.09 & 0.30 & 0.07 & 0.27 \\
      \textsc{cps-gp}              & 0.17   & 0.39 & 0.11 & 0.30 & 0.08 & 0.22 \\
      \textsc{cps-gp} $\gamma=0.2$ & 0.18   & 0.47 & 0.13 & 0.40 & 0.10 & 0.31 \\
      \textsc{cps-gp} $\gamma=0.5$ & 0.16   & 1.11 & 0.11 & 1.03 & 0.08 & 0.48 \\
      \textsc{cps-gp} $\gamma=0.8$ & 0.22   & 0.42 & 0.19 & 0.38 & 0.14 & 0.33 \\
      \bottomrule
  \end{tabularx}
  \caption{Average KS--PIT and RMSE over 100 repetitions for the Hartmann6 function at
different design set sizes~$n$. The GP regularity is fixed at $p=2$. Predictive
distributions are obtained with the \cpsgp\ method using either 
(i)~a data split ($\gamma n$ points for parameter selection and $(1-\gamma)n$
for prediction), 
(ii)~the same dataset~$\Dcal_n$ for both steps, or 
(iii)~oracle parameters pre-selected on an independent dataset of~$10d$ points.  
RMSE increases under data splitting, with the largest degradation observed for
$\gamma=0.5$, while calibration is poorest for $\gamma=0.8$. The full-data
configuration performs close to the oracle case but retains biased coverage.}
  \label{table:rmse_cps_gp}
\end{table}

\begin{remark}
\citet{LiangBarber2025} introduced a notion of algorithmic stability in
conformal prediction, showing that if a procedure is stable, its intervals can
be slightly widened to recover exact coverage. Although their analysis focuses
on predictive intervals, the same idea may apply to \cpsgp: if the selected
hyperparameters vary little under sample perturbations, calibration is
approximately preserved. A formal analysis of this stability for \cpsgp
and its impact on finite-sample validity remains open.
\end{remark}

\paragraph{Parameter selection on the prediction dataset}

We next compare \cpsgp\ and \bcrgp\ when both use parameters selected on the
same dataset~$\Dcal_n$ as that used for prediction. This configuration, already
discussed above, breaks exchangeability and therefore invalidates the marginal
coverage guarantees of Section~\ref{sec:cps-gp-interpolation}.
Figure~\ref{fig:n_values} reports the resulting empirical coverages for the
Hartmann6 and Goldstein--Price functions. For Hartmann6, \bcrgp\ with
$\delta=0.01$ attains coverage close to the intended confidence levels, whereas
\cpsgp\ and the GP posterior remain noticeably miscalibrated. For
Goldstein--Price, \cpsgp\ provides the most accurate coverage among the
compared methods.

As shown in Proposition~\ref{prop:cps-gp-finite}, the average interval width of
\cpsgp\ becomes unbounded for $n=5d$ at the 95\% level, illustrating a
finite-sample limitation of the procedure. In contrast, \bcrgp\ with
$\delta=0.01$ yields finite, slightly conservative intervals.

\begin{figure}[htbp]
  \centering
  \includegraphics[width=0.95\textwidth]{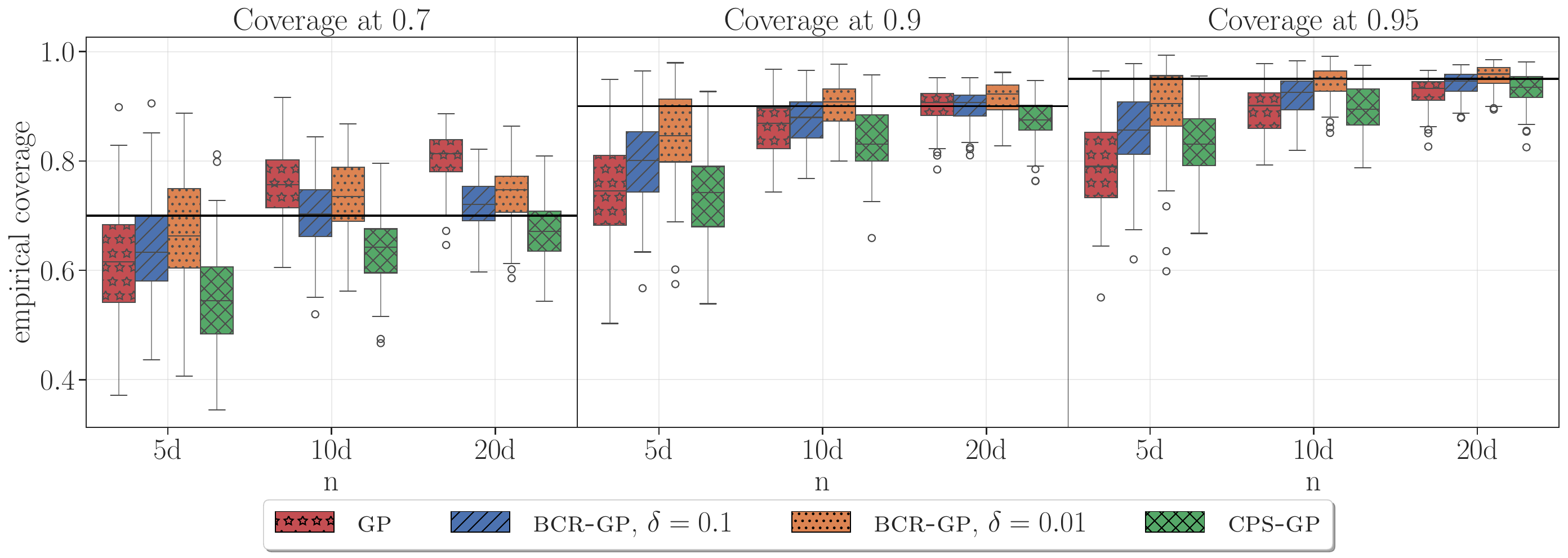}
  \includegraphics[width=0.95\textwidth]{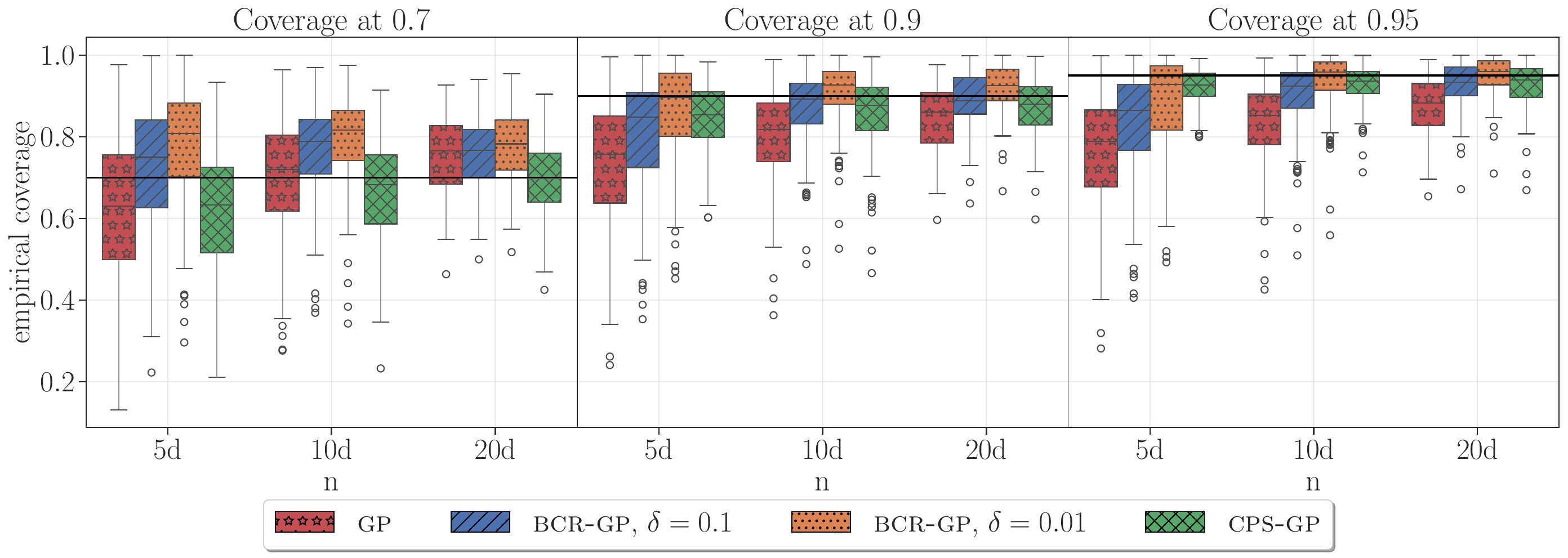}
  \caption{Empirical coverage of 70\%, 90\%, and 95\% prediction intervals for
  the Hartmann6 (first row) and Goldstein--Price (second row) functions, for several design set
  sizes~$n$. The
  GP regularity is fixed at $p=2$.
  Prediction intervals are constructed with the variance-based \bcrgp\ method
  for $\delta=0.1$ and $\delta=0.01$, with parameters selected on the same
  dataset~$\Dcal_n$.}
  \label{fig:n_values}
\end{figure}

\subsection{Comparison of coverage and width of prediction intervals}
\label{sec:res_coverage}

We evaluate empirical coverage and average interval width for the GP, \fcp, 
\jpgp, \bcrgp, and \cpsgp\ methods. Each result is averaged over 100 repetitions. 
Parameters are selected on the same dataset~$\Dcal_n$ used for prediction, so the 
marginal validity of conformal procedures no longer holds.

Figure~\ref{fig:coverage} displays empirical coverages at 70\%, 90\%, and 95\% 
for the Goldstein--Price function and for Matérn regularity orders $p=0,1,2$. 
Tables~\ref{tab:coverage} and~\ref{tab:width} report, respectively, mean coverage 
and relative interval width across all test functions.

All conformal and Bayesian–conformal methods improve empirical coverage compared 
with the GP posterior. 
As expected, \fcp\ and \jpgp\ have similar results. 
Both \cpsgp\ and \bcrgp\ with $\delta=0.1$ achieve coverage closest to the target 
levels, with \bcrgp\ slightly conservative. 
When using the KS–PIT selection rule with $\delta=0.1$, \bcrgp\ tends to become optimistic at high 
confidence (90–95\%). 
None of the conformal variants reach exact coverage once parameters are reused.

Across test functions and for all values of~$p$, 
\bcrgp\ remains near the intended coverage levels, showing stable behavior 
across problems.

Regarding interval width, \bcrgp\ with $\delta=0.1$ produces the narrowest 
intervals among well-calibrated methods, while $\delta=0.01$ yields wider, more 
conservative intervals. 
\jpgp, \fcp, and \cpsgp\ produce shorter intervals on average, consistent with 
their more optimistic coverage shown in Table~\ref{tab:coverage}.

\begin{figure}[htbp]
  \centering
  \includegraphics[width=\textwidth]{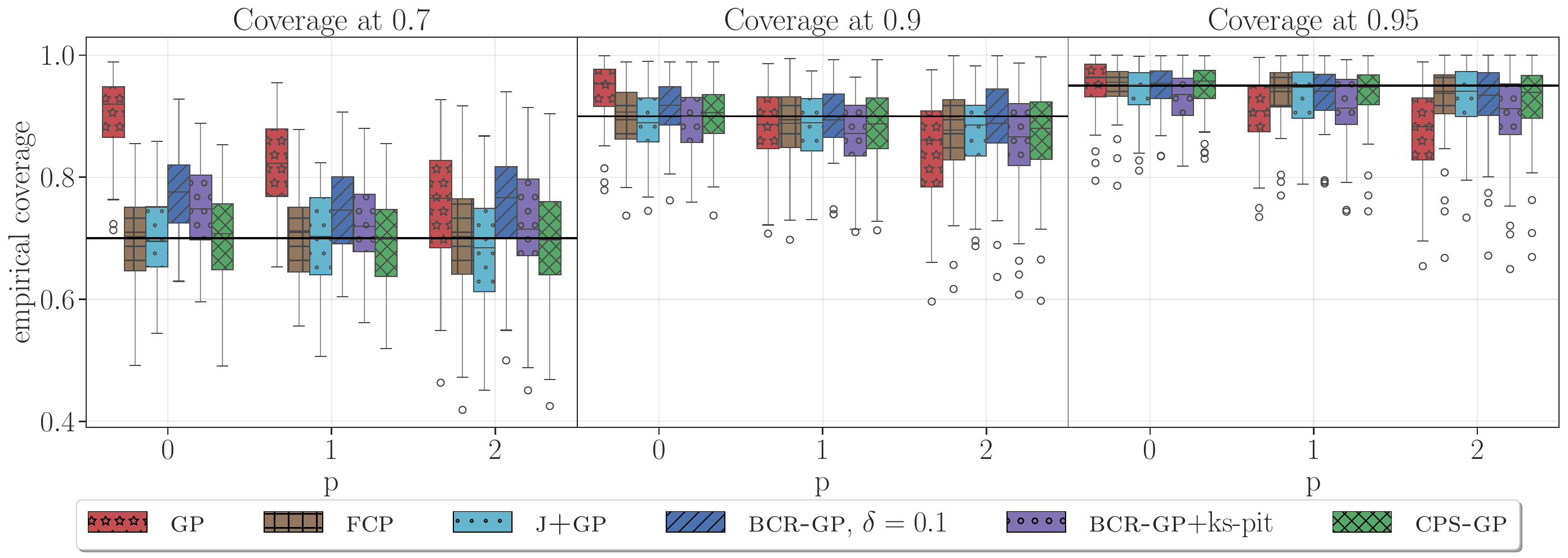}
  \caption{Empirical coverage of 70\%, 90\%, and 95\% prediction intervals for 
  the Goldstein--Price function ($n=20d=40$). Results are shown for Matérn 
  regularity orders $p=0,1,2$. Prediction intervals are computed with \bcrgp\ 
  using either the variance-based rule ($\delta=0.1$) or the KS--PIT rule ($\delta=0.1$).}
  \label{fig:coverage}
\end{figure}

\begin{table}[hpbt]
  \centering
  \setlength\tabcolsep{4pt}
  \begin{tabularx}{\textwidth}{l  *{12}{>{\centering\arraybackslash}p{0.85cm}}}
    \toprule
    \textbf{Method} & 
    \multicolumn{2}{c}{\scriptsize \makecell{Goldstein\\--Price}} & 
    \multicolumn{2}{c}{\scriptsize Ackley4} & 
    \multicolumn{2}{c}{\scriptsize Hartman3} & 
    \multicolumn{2}{c}{\scriptsize \makecell{Dixon\\--Price4}} & 
    \multicolumn{2}{c}{\scriptsize  Rosenbrock6} & 
    \multicolumn{2}{c}{\scriptsize  Branin}  \\
    ~ & 90\% & 95\% & 90\% & 95\% & 90\% & 95\% & 90\% & 95\% & 90\% & 95\% & 90\% & 95\%  \\
    \hline \hline
    \smbf{bcr-gp} \tiny (ks-pit) & 0.86 & 0.90 & 0.89 & 0.93 & 0.89 & 0.94 & 0.87 & 0.92 & 0.87 & 0.92 & 0.91 & \fnsb{0.95} \\
    \smbf{bcr-gp} \tiny $(0.01)$ & \fnsb{0.92} & \fnsb{0.95} & 0.94 & 0.96 & 0.94 & 0.97 & 0.93 & \fnsb{0.96} & \fnsb{0.91} & \fnsb{0.96} & 0.96 & 0.99 \\
    \smbf{bcr-gp} \tiny $(0.1)$ & 0.89 & 0.93 & \fnsb{0.92} & \fnsb{0.95} & \fnsb{0.92} & 0.96 & \fnsb{0.90} & 0.94 & 0.89 & 0.94 & 0.94 & 0.97 \\
    \smbf{fcp} & 0.87 & 0.93 & 0.87 & 0.93 & 0.89 & 0.94 & 0.87 & 0.93 & 0.86 & 0.93 & \fnsb{0.90} & 0.96 \\
    \smbf{gp} & 0.85 & 0.88 & \fnsb{0.90} & 0.93 & 0.93 & \fnsb{0.95} & 0.87 & 0.92 & 0.85 & 0.91 & 1.00 & 1.00 \\
    \smbf{cps-gp} & 0.87 & 0.93 & 0.87 & 0.93 & 0.89 & 0.94 & 0.87 & 0.93 & 0.86 & 0.93 & \fnsb{0.90} & \fnsb{0.95} \\
    \smbf{j+gp} & 0.87 & 0.93 & 0.88 & 0.93 & 0.89 & 0.94 & 0.88 & 0.93 & 0.87 & 0.93 & \fnsb{0.90} & \fnsb{0.95} \\
    \bottomrule
  \end{tabularx}
  \caption{Average empirical coverage over $100$ repetitions for several 
  confidence levels for multiple test functions. The empirical coverage is 
  computed on a test grid with $4000$ points and $n=20\times d$ observations.
  For every method, the regularity of the covariance is $p=2$. \bcrgp\ with the KS--PIT rule is used with $\delta = 0.1$.}
    \label{tab:coverage}
\end{table}

\begin{table}[hpbt]
  \centering
  \setlength\tabcolsep{4pt}
  \begin{tabularx}{\textwidth}{l *{12}{>{\centering\arraybackslash}p{0.85cm}}}
    \toprule
    \textbf{Method} & 
    \multicolumn{2}{c}{\scriptsize \makecell{Goldstein\\--Price}} & 
    \multicolumn{2}{c}{\scriptsize Ackley4} & 
    \multicolumn{2}{c}{\scriptsize Hartman3} & 
    \multicolumn{2}{c}{\scriptsize \makecell{Dixon\\--Price4}} & 
    \multicolumn{2}{c}{\scriptsize  Rosenbrock6} & 
    \multicolumn{2}{c}{\scriptsize  Branin} \\
    ~ & 90\% & 95\% & 90\% & 95\% & 90\% & 95\% & 90\% & 95\% & 90\% & 95\% & 90\% & 95\%  \\
    \hline \hline
      \smbf{bcr-gp} \tiny (ks-pit) & 3 & 3.41 & 1.03 & 1.10 & 1.21 & 1.32 & 1.60 & 1.60 & 1.27 & 1.28 & 1.55 &\fnsb{1.64} \\
      \smbf{bcr-gp} \tiny $(0.01)$ & \fnsb{4.43} & \fnsb{5.36} & 1.31 & 1.44 & 1.64 & 1.90 & 1.91 & \fnsb{1.92} & \fnsb{1.50} & \fnsb{1.52} & 2.17 & 2.40 \\
      \smbf{bcr-gp} \tiny $(0.1)$ & 3.55 & 4.06 & \fnsb{1.16} & \fnsb{1.24} & \fnsb{1.40} & 1.56 & \fnsb{1.73} & 1.72 & 1.38 & 1.40 & 1.81 & 1.91 \\
      \smbf{fcp} & 1.22 & 1.59 & 0.89 & 0.99 & 0.84 & 0.99 & 1.03 & 1.06 & 1.05 & 1.08 & \fnsb{0.50} & 0.60 \\
      \smbf{cps-gp} & 3.37 & 4.41 & 0.99 & 1.12 & 1.27 & 1.59 & 1.61 & 1.67 & 1.29 & 1.34 & \fnsb{1.66} & \fnsb{2.09} \\
      \smbf{j+gp} & 1.23 & 1.60 & 0.91 & 1.00 & 0.87 & 0.99 & 1.04 & 1.06 & 1.05 & 1.07 & \fnsb{0.49} & \fnsb{0.58} \\
    \bottomrule
  \end{tabularx}
  \caption{Relative width of PIs compared with the width of PIs predicted by GP 
  over $100$ repetitions for several confidence levels 
  for multiple test functions. The widths are computed on a test grid with $4000$ 
  points and $n=20\times d$ observations. For every method, the regularity of 
  the covariance is $p=2$. In bold are reported the width of the predictions intervals with the empirical coverage closest to the targeted confidence level. \bcrgp\ with the KS--PIT rule is used with $\delta = 0.1$.}
  \label{tab:width}
\end{table}
\subsection{Comparison of predictive distributions}
\label{sec:res_pred_distr}

We evaluate the predictive distributions of the posterior GP, \bcrgp, and \cpsgp\ using
KS--PIT (uniformity of PIT) and SCRPS (proper scoring). Results averaged over
100 repetitions are reported in Table~\ref{tab:ks_pit_scrps}; the SCRPS
definition and expressions for the generalized normal distribution are given in Appendix~\ref{ap:scrps_formulas}. 
\fcp\ and \jpgp\ are excluded from this subsection because they do not define full
predictive distributions (intervals only).

Predictive laws are obtained as follows: GP uses its Gaussian posterior;
\bcrgp\ fits a generalized normal to normalized residuals; \cpsgp\ yields a
stepwise predictive law (PIT computed with boundary randomization).
All methods follow the same hyperparameter-selection protocol as in
Section~\ref{sec:experimental-setup}.

\bcrgp\ with the KS--PIT selection rule attains the lowest KS--PIT on most
benchmarks, indicating better distributional calibration. The variance-based
\bcrgp\ (e.g., $\delta\in\{0.1,0.01\}$) achieves comparable or better SCRPS and
slightly heavier tails. \cpsgp\ is competitive on both metrics but typically
shows marginally larger KS--PIT when parameters are selected on the same
dataset~$\Dcal_n$. Consistent with Section~\ref{sec:res_coverage}, the KS--PIT
rule can be optimistic at high confidence, whereas the variance-based rule is
more stable across problems.

\begin{table}[hpbt]
  \centering
  \setlength\tabcolsep{4pt}
  \begin{tabularx}{\textwidth}{l *{14}{X}}
    \toprule
    \textbf{Method} & 
    \multicolumn{2}{c}{\scriptsize \makecell{Goldstein\\--Price}} & 
    \multicolumn{2}{c}{\scriptsize Ackley4} & 
    \multicolumn{2}{c}{\scriptsize Hartman3} & 
    \multicolumn{2}{c}{\scriptsize \makecell{Dixon\\--Price4}} & 
    \multicolumn{2}{c}{\scriptsize  Rosenbrock6} & 
    \multicolumn{2}{c}{\scriptsize  Branin}  \\
    ~ &\tiny KS--PIT &\tiny SCRPS &\tiny KS--PIT &\tiny SCRPS &\tiny KS--PIT &\tiny SCRPS &\tiny KS--PIT & \tiny SCRPS &\tiny KS--PIT &\tiny SCRPS &\tiny KS--PIT &\tiny SCRPS  \\
    \hline \hline
        \smbf{bcr-gp} \tiny (ks-pit) & \fnsb{0.11} & 5.78 & \fnsb{0.09} & 0.86 &  \fnsb{0.09} & \fnsb{-0.15} & 0.08 & 5.40 &  \fnsb{0.06} & 6.81 &  \fnsb{0.11} &  \fnsb{0.30} \\
        \smbf{bcr-gp} \tiny $(0.01)$ & 0.12 & 5.76 & 0.10 & 0.86 & 0.11 & -0.12 & 0.09 & 5.40 & 0.07 & 6.81 & 0.13 & 0.34 \\
        \smbf{bcr-gp} \tiny $(0.1)$ & 0.12 & \fnsb{5.75} & \fnsb{0.09} & \fnsb{0.85} & 0.10 &  -0.14 & 0.08 & 5.40 &  \fnsb{0.06} & 6.81 & 0.13 & 0.31 \\
        \smbf{gp} & 0.16 & 5.78 & 0.10 & 0.87 & 0.16 & -0.12 &  \fnsb{0.07} & 5.40 &  \fnsb{0.06} & 6.82 & 0.24 & 0.49 \\
        \smbf{cps-gp} & 0.12 & 5.79 & \fnsb{0.09} & 0.86 & 0.10 &  \fnsb{-0.14} & 0.09 & 5.41 & 0.08 & 6.82 & 0.14 & 0.31 \\
    \bottomrule
  \end{tabularx}
  \caption{Average KS--PIT and SCRPS over $100$ repetitions, computed
    on a test set of $4000$ points and with $n=20\times d$
    observations for GP, \bcrgp with $\delta = 0.1$ and
    $\delta = 0.01$, and \cpsgp. The regularity of the Matérn
    covariance function is $p=2$. \bcrgp\ with the KS--PIT rule is used with $\delta = 0.1$.}
  \label{tab:ks_pit_scrps}
\end{table}

\section{Discussion}
\label{sec:discussion}

This work targets $\mu$-calibration of Gaussian-process (GP) predictive
distributions. We first presented \cpsgp, an adaptation of conformal
predictive systems to GP interpolation that yields design-marginal
probabilistic calibration (on average over the data). We then introduced
\bcrgp, a Bayesian post-processing method that preserves the GP posterior mean
and calibrates dispersion by fitting a parametric residual model, producing
smooth predictive CDFs usable in standard sequential-design algorithms.

In empirical benchmarks, \bcrgp, \cpsgp, and \jpgp\ improve empirical
coverage relative to the Gaussian posterior, and \bcrgp\ and \cpsgp\ attain
comparable probabilistic-calibration quality as measured by KS--PIT and SCRPS.
Unlike \fcp\ or \jpgp, which only provide prediction intervals, \bcrgp\ and
\cpsgp\ produce full predictive distributions. The additional computational
cost relative to standard GP prediction remains modest in the regimes
considered.

\paragraph{Applicability and limitations}
A practical limitation of \cpsgp\ is that it does not yield a simple parametric
predictive distribution, which complicates its use in sampling criteria (aka acquisition functions)
that rely on closed-form Gaussian formulas (e.g.\ expected-improvement
Bayesian optimization). Another limitation, common to all conformal approaches,
is that \cpsgp\ and \jpgp\ guarantee design-marginal rather than
training-conditional coverage. The formal marginal validity of \cpsgp\ further
requires kernel hyperparameters to be chosen independently of the calibration
step; when they are estimated on the same data, coverage remains approximately
correct but is no longer exact, even marginally over the data.

By contrast, \bcrgp\ has no distribution-free guarantee: its calibration is
purely model-based and relies on (i) the adequacy of the generalized-normal
residual family and (ii) the working assumption that LOO standardized residuals
provide a good proxy for the $\mu$-marginal distribution of $R_n(X,f(X))$.
In this work, these assumptions are only checked empirically; quantifying their
deviation from the true residual law, and the resulting impact on coverage and
PIT-based diagnostics, would be a natural and important direction for future
study.

\paragraph{Noisy observations}
Although both methods were developed for interpolation, many practical
applications involve noisy observations $Z = f(X) + \varepsilon$, where the goal
is to predict the latent function $f$. In such settings, distribution-free
coverage guarantees for $f$ are generally unattainable without modelling the
noise. Once a noise model is specified, both \cpsgp\ and \bcrgp\ could be extended
to yield model-conditional calibration, by applying the same constructions to
the latent GP and treating the noise within the likelihood.

\paragraph{Sequential design}
Applying $\mu$-calibrated predictors within sequential algorithms raises the
question of defining the design measure $\mu$ when design points depend on past
data. In practice, for sampling criteria such as expected improvement, the
adaptive design explores the domain, and calibration can be assessed empirically
with respect to the empirical distribution of the visited points. Periodic
recalibration, or calibration diagnostics computed with importance weights that
map the empirical design distribution to a target~$\mu$, can be used to monitor
and maintain $\mu$-calibration as the design evolves.

\paragraph{Tail calibration}
The \bcrgp\ framework can also be tailored toward upper-tail calibration by
modifying the posterior selection rule or adopting a heavier-tailed residual
model. Such adjustments make it possible to target exceedance probabilities or
risk metrics while preserving smooth predictive CDFs and the GP mean structure.

Both proposed approaches enhance the design-marginal calibration of GP-based
predictive distributions, a property often expected by practitioners but not
guaranteed by the standard GP posterior. \cpsgp\ offers a distribution-free
calibration mechanism suited to fixed designs, while \bcrgp\ provides smooth,
parametric predictive distributions directly usable in sequential design and
optimization. Together, they improve the practical reliability of GP modelling
by producing calibrated uncertainty estimates that remain interpretable and
computationally tractable.

\clearpage
\appendix
\section{Proofs}
\label{sec:appendix-proofs}

\subsection{Generalized inverse identities}
\label{sec:gen-inv-id}

\begin{lemma}[Generalized inverse identities]
\label{lem:gen-inverse}
Let $G:\R\to[0,1]$ be nondecreasing and right-continuous, and define its
generalized inverse by
$$
G^{-1}(u) := \inf\{z\in\R:\ G(z)\ge u\},\qquad u\in(0,1).
$$
Then, for all $u\in(0,1)$ and $z\in\R$,
\begin{align}
\{\,G(z)<u\,\} &\Longleftrightarrow \{\,z<G^{-1}(u)\,\}, \label{eq:lemma1}\\[3pt]
\{\,G(z^-)\le u\,\} &\Longleftrightarrow \{\,z\le G^{-1}(u)\,\}. \label{eq:lemma2}
\end{align}
Moreover, for any $0<p<q<1$ and $v\in\R$,
\begin{equation}
\label{eq:lemma3}
\{\,G^{-1}(p)\le v < G^{-1}(q)\,\}
\Longleftrightarrow
\{\,p \le G(v) < q\,\}.
\end{equation}
\end{lemma}

\newif\iflongproofs

\begin{proof}
\iflongproofs
\else
The equivalences \eqref{eq:lemma1}–\eqref{eq:lemma2} follow directly from the
definition of $G^{-1}$ and the fact that $G(z^-)=\sup_{w<z}G(w)$.
For \eqref{eq:lemma3}, apply \eqref{eq:lemma1} with $u=q$ and its contrapositive
with $u=p$ to relate $G^{-1}(p),G^{-1}(q)$ and $G(v)$.
\fi
\end{proof}

\subsection{Proof of Proposition~\ref{prop:randomized-pit}}
\label{sec:proof-randomized-pit}

Fix $x\in\XX$ and condition throughout on $\Dcal_n$. Let
$$
F(\cdot) := \hat F_n(\cdot\mid x),
\qquad
\Delta F(z) := F(z)-F(z^-),
$$
and draw $V\sim F$ independently of $\tau\sim\U(0,1)$. Define
$$
U := F(V^-)+\tau\,\Delta F(V).
$$

\paragraph{Uniformity}
Fix $t\in[0,1]$ and set the (left-continuous) quantile
$$
z_t := F^{-1}(t) = \inf\{z\in\R:\ F(z)\ge t\}.
$$
We have
\begin{equation*}
\{U\le t\}
= \{V<z_t\}\ \cup\ \Bigl(\{V=z_t\}\cap\{\tau \le \tfrac{t-F(z_t^-)}{\Delta F(z_t)}\}\Bigr),
\end{equation*}
with the convention that the second event is empty when $\Delta F(z_t)=0$.
Indeed:
\begin{enumerate}[label=(\roman*)]
\item if $V<z_t$, then $F(V)\le F(z_t^-)\le t$, hence $U\le F(V)\le t$;
\item if $V>z_t$, then $F(V^-)\ge F(z_t)\ge t$, hence $U\ge F(V^-)>t$;
\item if $V=z_t$, then $U$ is uniform on $[F(z_t^-),F(z_t)]$ given $V$.
\end{enumerate}

Taking probabilities under $\P_n(\,\cdot\,)$ and using $V\sim F$ and independence of $\tau$,
$$
\begin{aligned}
\P_n(U\le t)
&= \P_n(V<z_t)\;+\;\P_n(V=z_t)\,\P_n\left(\tau \le \frac{t-F(z_t^-)}{\Delta F(z_t)}\ \bigg|\ V=z_t\right)\\[3pt]
&= F(z_t^-)\;+\;\Delta F(z_t)\,\frac{t-F(z_t^-)}{\Delta F(z_t)}\\[3pt]
&= t.
\end{aligned}
$$
If $\Delta F(z_t)=0$, then $\P_n(V=z_t)=0$ and $F(z_t^-)=t$, yielding the same result.
Thus $U\mid\Dcal_n\sim\U(0,1)$.

\paragraph{Exact mass}  
Set $a=\alpha/2$ and $b=1-\alpha/2$. Writing
$$
\Ccal_{n,\tau,1-\alpha}(x)
:=\bigl[\hat F_{n,\tau}^{-1}(a\mid x),\,\hat F_{n,\tau}^{-1}(b\mid x)\bigr),
$$
we obtain by Lemma~\ref{lem:gen-inverse} the event identity
$$
\{\,V\in \Ccal_{n,\tau,1-\alpha}(x)\,\}
\;\Longleftrightarrow\;
\{\,a \le \hat F_{n,\tau}(V\mid x) < b\,\}
\;\Longleftrightarrow\;
\{\,a \le U < b\,\}.
$$
Taking probabilities under $\P_n$  and using 
$U\mid\Dcal_n\sim\U(0,1)$ yields
$$
\P_n\bigl\{\,V\in \Ccal_{n,\tau,1-\alpha}(x)\,\bigr\}
=\P_n\{\,a\le U < b \,\}= b-a = 1-\alpha.
$$

\subsection{Proof of Remark~\ref{rem:pit-under-mu}}
\label{sec:proof-rem-pit-under-mu}

We want to justify the inequalities
$$
\mu\bigl(\{x:\ f(x)<\hat F_n^{-1}(u\mid x)\}\bigr)
\ \le\
G_\mu(u)
\ \le\
\mu\bigl(\{x:\ f(x)\le\hat F_n^{-1}(u\mid x)\}\bigr),
\qquad u\in(0,1),
$$
where
$$
G_\mu(u) := \P_n\bigl(U_{\hat F_n}^{f,\,\mu,\,\tau}\le u\bigr),
\qquad
U_{\hat F_n}^{f,\,\mu,\,\tau}
:= \hat F_n(f(X)^-\mid X)
   + \tau\bigl(\hat F_n(f(X)\mid X)-\hat F_n(f(X)^-\mid X)\bigr),
$$
with $X\sim\mu$ and $\tau\sim\U(0,1)$ independent of $(\hat F_n,\,X,\,\Dcal_n)$.

For each $x$, set $a_x:=\hat F_n(f(x)^-\mid x)$ and $b_x:=\hat F_n(f(x)\mid x)$.
We have
$$
\P_n\bigl(U_{\hat F_n}^{f,\,\mu,\,\tau}\le u \mid X=x\bigr)
=
\begin{cases}
0, & u\le a_x,\\
\dfrac{u-a_x}{b_x-a_x}, & a_x<u<b_x,\\
1, & u\ge b_x.
\end{cases}
$$
Thus,
$$
\one\{\,b_x\le u\,\}\ \le\ \P_n\bigl(U_{\hat F_n}^{f,\,\mu,\,\tau}\le u \mid X=x\bigr)
\ \le\ \one\{\,a_x\le u\,\}.
$$
Integrating over $\mu$ yields
$$
\mu \bigl(\{x:\ b_x\le u\}\bigr)\ \le\ G_\mu(u)\ \le\ \mu \bigl(\{x:\ a_x\le u\}\bigr).
$$
From Lemma~\ref{lem:gen-inverse}:
$$
\{\,G(z)<u\,\}\Longleftrightarrow\{\,z<G^{-1}(u)\},\qquad
\{\,G(z^-)\le u\,\}\Longleftrightarrow\{\,z\le G^{-1}(u)\}\,.
$$
Substituting $G(z)=\hat F_n(z\mid x)$ and $z=f(x)$, we obtain
$$
\mu \bigl (\{x:\ f(x)< \hat F_n^{-1}(u\mid x)\}\bigr)
= \mu\bigl(\{x:\ b_x<u\}\bigr)
\ \le\ G_\mu(u)$$
and
$$ G_\mu(u) \le\
\mu \bigl(\{x:\ a_x\le u\}\bigr)
= \mu \bigl(\{x:\ f(x)\le \hat F_n^{-1}(u\mid x)\}\bigr),
$$
as claimed.

\subsection{Proof of Remark~\ref{rem:mu-coverage-via-PIT}}
\label{sec:proof-mu-coverage-via-PIT}

Let $F_{n,\,x}(z):=\hat F_n(z\mid x)$ and $q_x(u):=\hat F_n^{-1}(u\mid x)$.
Fix $a=\alpha/2$ and $b=1-\alpha/2$, and set $U:=\hat F_n(f(X)\mid X)$ with $X\sim\mu$.

\paragraph{Lower bound for $\delta_\alpha(\hat F_n;\mu)$}
By Lemma~\ref{lem:gen-inverse}, we have $\{U<b\} \iff \{f(X) < q_X(b)\}$ and
$\{a<U\} \Rightarrow \{f(X)\ge q_X(a)\}$ (since $f(X)<q_X(a)$ would imply $U\le a$).
Hence
$$
\{\,a<U<b\,\}\ \subseteq\ \{\,q_X(a)\le f(X)<q_X(b)\,\}
\ \subseteq\ \{\,q_X(a)\le f(X)\le q_X(b)\,\}.
$$
Taking probabilities gives
$$
\P_n\{\,a<U<b\,\}\ \le\ \delta_\alpha(\hat F_n;\mu).
$$

\paragraph{Failure of the naive upper bound}
In general, the inclusion 
$$
\{\,q_X(a)\le f(X)\le q_X(b)\,\}\subseteq\{\,a\le U\le b\,\}
$$
may fail. At the upper endpoint, if $f(X)=q_X(b)$ and $F_{n,\,x}$ has a jump at $q_X(b)$, 
then $U> b$ is possible; similarly, at the lower endpoint $f(X)=q_X(a)$ may give $U>a$.
Thus the upper inequality
$$
\delta_\alpha(\hat F_n;\mu)\le \P_n\{a\le U\le b\}
$$
is not always valid without regularity assumptions.

\paragraph{Safe two-sided bounds}
We have
$$
\mu\bigl(\{x:\ f(x)<q_x(b)\}\bigr)=\P_n\{U<b\},
\qquad
\mu\bigl(\{x:\ f(x)<q_x(a)\}\bigr)=\P_n\{U<a\}.
$$
(Since $\{\,F_{n,\,x}(z)<u\,\}\iff\{\,z<q_x(u)\,\}$, by Lemma~\ref{lem:gen-inverse}.) 
Hence
\begin{align*}
\delta_\alpha(\hat F_n;\mu)
&= \mu\bigl(\{x:\ f(x)\le q_x(b)\}\bigr) - \mu\bigl(\{x:\ f(x)<q_x(a)\}\bigr) \\
&= \P_n\{U < b\} - \P_n\{U < a\} + \mu\bigl(\{x:\ f(x)=q_x(b)\}\bigr) \\[2pt]
&= \P_n\{a\le U<b\} + \mu\bigl(\{x:\ f(x)=q_x(b)\}\bigr)\\[2pt]
&= \P_n\{a\le U\le b\} - \P_n\{U=b\} + \mu\bigl(\{x:\ f(x)=q_x(b)\}\bigr).  
\end{align*}
Notice that $\mu\bigl(\{x:\ f(x)=q_x(b)\}\bigr)=\mu(A_=)+\mu(A_>)$, with
$$
A_{=} = \{x:\ f(x)=q_x(b),\ F_{n,\,x}(q_x(b))=b\}\quad\text{and}\quad A_{>} = \{x:\ f(x)=q_x(b),\ F_{n,\,x}(q_x(b))>b\}\,.
$$

On $A_=$ ($F_{n,\,x}$ has no jumps), we have
$$
U = \hat F_n(f(X)\mid X) = F_{n,\,X}(q_X(b)) = b,
$$
so that
$$
A_{=} \subseteq \{U=b\}\,,
$$
and
$$
\P_n\{U = b\}\ \ge\ \mu(A_=).
$$

Plugging in yields
$$
\delta_\alpha(\hat F_n;\mu)
\le \P_n\{a\le U\le b\}\;-\;\mu(A_=)\;+\;\mu(A_=)+\mu(A_>)
= \P_n\{a\le U\le b\} \;+\; \mu(A_>).
$$
Equivalently,
$$
\delta_\alpha(\hat F_n;\mu)\ \le\ \P_n\{a\le U\le b\}
   + \mu\bigl(\{x:\ f(x)=q_x(b),\ F_{n,\,x}(q_x(b))>b\}).
$$
Therefore, we have the following two-sided bounds:
$$
\P_n\{a\le U<b\}\ \le\ \delta_\alpha(\hat F_n;\mu)\ \le\ \P_n\{a\le U\le b\}
   \;+\; \mu\bigl(\{x:\ f(x)=q_x(b),\ F_{n,\,x}(q_x(b))>b\}\bigr)\,.
$$

\paragraph{Restoring equality}
If, for $\mu$-a.e.\ $x$, $F_{n,\,x}$ is continuous and strictly increasing, then
$$
\{\,q_X(a)\le f(X)\le q_X(b)\,\}\ \Longleftrightarrow\ \{\,a\le U\le b\,\},
$$
hence
$$
\delta_\alpha(\hat F_n;\mu)=\P_n\{\,a\le U\le b\,\}.
$$

In the discontinuous case, by Lemma~\ref{lem:gen-inverse},
$$
\{\,F_{n,\,x,\,\tau}^{-1}(a)\le f(x)<F_{n,\,x,\tau}^{-1}(b)\,\}
\ \Longleftrightarrow\
\{\,a\le F_{n,\,x,\,\tau}(f(x))< b\,\}\,.
$$
With $X\sim\mu$ and
$$
U_{\hat F_n}^{f,\,\mu,\,\tau} := F_{n,\,X,\,\tau}(f(X)),
$$
this gives
$$
\mu\ \left(\{x:\ f(x)\in[\,\hat F_{n,\,\tau}^{-1}(a\mid x),\,\hat
F_{n,\,\tau}^{-1}(b\mid x)\,)\}\right)
=
\P_n \left(\,a\le U_{\hat F_n}^{f,\,\mu,\,\tau}< b \ \middle|\ \tau\right).
$$
Thus, boundary randomization restores exact equivalence in the general case
(see also Proposition~\ref{prop:randomized-pit}).

\subsection{Proof of Proposition~\ref{prop:ks_pit_mu_consistency}}
\label{sec:proof_ks_pit_mu_consistency}

Work throughout conditional on $\Dcal_n$.

Since $(X_j^\star,\tau_j)$ are i.i.d., the $U_j$ are i.i.d. on $[0,1]$ with distribution function
$$
G_\mu(u) = \P \big(U\le u \mid \Dcal_n\big),\quad
U = \hat F_n \big(f(X)^- \mid X \big) + \tau \Big(\hat F_n \big(f(X) \mid X\big) - \hat F_n \big(f(X)^- \mid X \big)\Big),
$$
where $X\sim\mu$ and $\tau\sim\U(0,1)$ are independent. Let
$$
\hat G_m(u) = \frac{1}{m}\sum_{j=1}^{m}\one\{U_j\le u\}.
$$
By the Glivenko--Cantelli theorem,
$$
\sup_{u\in[0,1]} \bigl| \hat G_m(u) - G_\mu(u) \bigr| \xrightarrow[m\to\infty]{\text{a.s.}} 0.
$$
Now write
$$
J_{\mathrm{KS\text{-}PIT},\,m}(\hat F_n) = \sup_{u\in[0,1]}\bigl|\hat G_m(u)-u\bigr|,
\qquad
J_{\mathrm{KS\text{-}PIT},\,\mu}(\hat F_n) = \sup_{u\in[0,1]}\bigl|G_\mu(u)-u\bigr|.
$$
Then, for $A_m(u)=\hat G_m(u)-u$ and $A(u)=G_\mu(u)-u$,
\begin{align*}
\bigl| J_{{\rm KS-PIT},\,m}(\hat F_n) - J_{{\rm KS-PIT},\,\mu}(\hat F_n) \bigr|
&= \bigl| \sup_u |A_m(u)| - \sup_u |A(u)| \bigr|\\
&\le \sup_{u} |A_m(u)-A(u)| \\
&= \sup_u \bigl| \hat G_m(u) - G_\mu(u) \bigr| \xrightarrow{\text{a.s.}} 0.  
\end{align*}
Hence the claim.

\subsection{Proof of Proposition~\ref{prop:ks_pit_vs_iae_general}}
\label{sec:proof_ks_pit_vs_iae_general}

Let $(X_i^\star)_{i=1}^m \stackrel{\text{i.i.d.}}{\sim}\mu$ be an independent
test design and set $Z_i := f(X_i^\star)$.
For each $i$, write the predictive CDF at $X_i^\star$ as
$\hat F_n(\cdot\mid X_i^\star)$ and define the PIT value
$$
U_i =
\begin{cases}
\hat F_n(Z_i \mid X_i^\star), & \text{if $\hat F_n(\cdot\mid X_i^\star)$ is continuous},\\[4pt]
\hat F_n(Z_i^-\mid X_i^\star) + \tau_i\big(\hat F_n(Z_i\mid X_i^\star)-\hat F_n(Z_i^-\mid X_i^\star)\big),
& \text{otherwise},
\end{cases}
$$
with $\tau_i\stackrel{\text{i.i.d.}}{\sim}\U(0,1)$ independent of everything.
Let $\hat G_m(u) := \frac{1}{m}\sum_{i=1}^m \one\{U_i \le u\}$ be the empirical CDF
of the $U_i$s.

\medskip
\noindent\emph{Continuous case.}
If each $\hat F_n(\cdot\mid X_i^\star)$ is continuous and strictly increasing,
the central $(1-\alpha)$ interval is
$$
\Ccal_{i,1-\alpha}
= \big[ \hat F_n^{-1}(\alpha/2 \mid X_i^\star), \; \hat F_n^{-1}(1-\alpha/2 \mid X_i^\star) \big],
$$
and
$$
\{Z_i \in \Ccal_{i,1-\alpha}\}
= \{\alpha/2 \le \hat F_n(Z_i\mid X_i^\star) \le 1-\alpha/2\}
= \{\alpha/2 \le U_i \le 1-\alpha/2\}.
$$
Hence
$$
\hat \delta_{\alpha, m}(\hat F_n)
= \frac{1}{m}\sum_{i=1}^m \one\{Z_i \in \Ccal_{i,1-\alpha}\}
= \hat G_m(1-\alpha/2) - \hat G_m(\alpha/2).
$$
Since $(1-\alpha) = (1-\alpha/2) - (\alpha/2)$,
$$
\big|\hat \delta_{\alpha, m}(\hat F_n) - (1-\alpha)\big|
\le \big|\hat G_m(1-\alpha/2)-(1-\alpha/2)\big|
+ \big|\hat G_m(\alpha/2)-\alpha/2\big|
\le 2 \sup_{u\in[0,1]} \big|\hat G_m(u)-u\big|.
$$
Integrating over $\alpha\in[0,1]$ gives
$$
J_{\mathrm{IAE},m}(\hat F_n)
= \int_0^1 \big|\hat \delta_{\alpha, m}(\hat F_n) - (1-\alpha)\big|\,d\alpha
\le 2\,J_{\mathrm{KS\text{-}PIT}, m}(\hat F_n).
$$

\medskip
\noindent\emph{Population version.}
Let $U$ denote the (randomized) $\mu$--PIT and $G(u):=\P(U\le u)$ its CDF.
Replacing $\hat G_m$ by $G$ in the same argument yields
$$
J_{\mathrm{IAE},\mu}(\hat F_n) \;\le\; 2\,J_{\mathrm{KS\text{-}PIT},\mu}(\hat F_n).
$$

\medskip
\noindent\emph{Discontinuous case.}
If some $\hat F_n(\cdot\mid X_i^\star)$ are discontinuous, using the randomized PIT above gives,
almost surely,
$$
\{Z_i \in \Ccal_{i,1-\alpha}\}
= \{\alpha/2 \le U_i < 1-\alpha/2\}.
$$
Therefore the identity
$\hat \delta_{\alpha, m}(\hat F_n)=\hat G_m(1-\alpha/2)-\hat G_m(\alpha/2)$
still holds (with half-open endpoints), and the inequalities in
Proposition~\ref{prop:ks_pit_vs_iae_general} follow unchanged.

\subsection{Proof of Proposition~\ref{prop:CP-coverage}}
\label{sec:proof-cp-coverage}

Assume $(X_i,Z_i)_{i=1}^{n+1}$ are i.i.d.\ and the conformal score
$R((x,z);\Dcal)$ is permutation-invariant in its dataset argument.
Let $\Dcal_{n+1}=\Dcal_{n+1}^{Z_{n+1}}$ and define $\pi$ as in
\eqref{cp-pi} with a tie–breaker $\tau\sim U(0, 1)$ independent of the data.

For any candidate value $z$, set
$$
A(z):=\#\{i\le n:\,R_i^z<R_{n+1}^z\}, 
\qquad
B(z):=\#\{i\le n:\,R_i^z=R_{n+1}^z\}.
$$
For the realized label $Z_{n+1}$, write $S_i:=R_i^{Z_{n+1}}$. Let
$S_{(1)} \le \cdots \le S_{(n+1)}$ be the order statistics, and let
$v_1< \ldots <v_m$ be the distinct values of $\{S_i\}_{i=1}^{n+1}$
with multiplicities $L_1,\ldots,L_m$.
Define cumulative counts
$$
A_1:=0,\qquad A_r:=\sum_{s<r}L_s\quad(r\ge2).
$$
Equivalently, in the ordered list
$S_{(1)}\le\cdots\le S_{(n+1)}$, the $r$th tie block occupies indices
$A_r+1,\, \ldots,\, A_r+L_r$ and equals the constant value $v_r$.

Define
$$
B_r\ :=\ \{\,S_{n+1}=v_r\,\}
\quad\text{equivalently}\quad
B_r\ :=\ \left\{\,\mathrm{rank}(S_{n+1})\in\{A_r+1,\,\ldots,\,A_r+L_r\}\right\}.
$$

By i.i.d.\ sampling and permutation invariance of $R$, $(S_1,\ldots,S_{n+1})$
is exchangeable. Conditional on the ordered values, the index $n+1$ is equally
likely to occupy any of the $n+1$ ranks. Hence
$$
\P\left(B_r \,\bigm|\, S_{(1)},\ldots,S_{(n+1)}\right)=\frac{L_r}{n+1}.
$$
On $B_r$ we have $A(Z_{n+1})=A_r$ and $B(Z_{n+1})=L_r-1$, so by \eqref{cp-pi}
$$
(n+1)\,\pi(Z_{n+1})\;=\;A_r+\tau\,L_r,
$$
and conditionally, we have $(n+1)\pi(Z_{n+1})\sim U\big([A_r,\,A_r+L_r]\big)$.

Therefore, for any $t\in[0,n+1]$,
\begin{align*}
\P\bigl((n+1)\pi(Z_{n+1}) \le t \bigm| B_r,\, S_{(1)},\ldots,S_{(n+1)} \bigr)
&= \P\Bigl(\tau \le \frac{t - A_r}{L_r} \Bigm| B_r,\, S_{(1)},\ldots,S_{(n+1)} \Bigr)\\
&= \frac{(t - A_r)_+ \wedge L_r}{L_r}.
\end{align*}
By the law of total probability,
\begin{align*}
\P\bigl((n+1)\pi(Z_{n+1}) \le t \,\bigm|\, S_{(1)},\ldots,S_{(n+1)} \bigr)
&= \sum_{r=1}^m \frac{L_r}{n+1}\cdot \frac{(t - A_r)_+ \wedge L_r}{L_r}
= \frac{t}{n+1},
\end{align*}
since the intervals $[A_r,\,A_r+L_r]$ partition $[0,n+1]$ and
$(t-A_r)_+\wedge L_r$ is the length of $[A_r,\,A_r+L_r]\cap[0,t]$.
Hence $(n+1)\pi(Z_{n+1}) \sim U[0,n+1]$ unconditionally, and
$\pi(Z_{n+1}) \sim U(0, 1)$.

In particular, for any $\alpha\in[0,1]$,
$$
\P\bigl( \pi(Z_{n+1}) \le 1-\alpha \bigr) = 1-\alpha.
$$

\subsection{Proof of Proposition~\ref{prop:affine-diff}}
\label{app:AiBi-deriv}

\paragraph{Assumptions}
We work in the noiseless interpolation setting with a fixed, centered Gaussian
process prior $\xi\sim\GP(0,k)$, where the covariance kernel $k$ is strictly
positive definite on the set $\{x_1,\dots,x_n,x_{n+1}\}$. The design points are
pairwise distinct, so the Gram matrices below are symmetric positive definite
(hence invertible). Throughout the derivation we condition on a fixed
dataset $\Dcal_n=\{(x_i,z_i)\}_{i=1}^n$ and a fixed test location $x_{n+1}$.
If a known nonzero mean $m$ is used, replace $z_i$ by $z_i-m(x_i)$ and $z$ by
$z-m(x_{n+1})$; the algebra is unchanged. Extensions to universal and intrinsic
kriging are discussed in Remark~\ref{rem:affine-extensions}.

\smallskip
Let $z_{1:n}=(z_1,\ldots,z_n)^\top$. For a test covariate $x_{n+1}$, set
$$
K_n:=\big(k(x_i,x_j)\big)_{i,j=1}^n,\qquad
k_*:=\big(k(x_1,x_{n+1}),\ldots,k(x_n,x_{n+1})\big)^\top,\qquad
k_{**}:=k(x_{n+1},x_{n+1}).
$$
Form the augmented covariance
$$
K_{n+1}=\begin{bmatrix}K_n & k_*\\ k_*^\top & k_{**}\end{bmatrix},
\qquad
u:=K_n^{-1}k_*,\qquad
v:=k_{**}-k_*^\top K_n^{-1}k_*.
$$
By the block inverse formula,
$$
\bar K_{n+1}:=K_{n+1}^{-1}
=
\begin{bmatrix}
K_n^{-1} + u\,v^{-1} u^\top & - u\,v^{-1} \\
- v^{-1} u^\top & v^{-1}
\end{bmatrix}.
$$
Write $\bar k_{ij}=(\bar K_{n+1})_{ij}$ and $\bar k_i=(\bar K_{n+1})_{ii}$. Then
$$
\bar k_{n+1}=\frac{1}{v},\qquad
\bar k_{i,n+1}=-\frac{u_i}{v},\qquad
\bar k_i=(K_n^{-1})_{ii}+\frac{u_i^2}{v}.
$$

Let $z^{\mathrm{aug}}=(z_1,\ldots,z_n,z)^\top$ and set
$w:=\bar K_{n+1}z^{\mathrm{aug}}$.

\paragraph{Components of $w$}
Block multiplication gives
$$
w_{1:n}=K_n^{-1}z_{1:n}+u\,v^{-1}\big(u^\top z_{1:n}-z\big),
\qquad
w_{n+1}=v^{-1}\big(z-u^\top z_{1:n}\big).
$$
Using $m_n(x_{n+1})=k_*^\top K_n^{-1}z_{1:n}=u^\top z_{1:n}$,
$$
w_j=\big(K_n^{-1}z_{1:n}\big)_j+\frac{u_j}{v}\big(m_n(x_{n+1})-z\big),
\quad j=1,\ldots,n,
\qquad
w_{n+1}=\frac{z-m_n(x_{n+1})}{v}.
$$

\paragraph{Standardized LOO scores}
From GP LOO identities (e.g., \citealp{petit_parameter_2023}), the leave-one-out
posterior at $(x_i,z_i)$ in the augmented dataset satisfies
$$
m_{n+1,-i}(x_i)=z_i-\frac{w_i}{\bar k_i},
\qquad
\sigma_{n+1,-i}(x_i)=\frac{1}{\sqrt{\bar k_i}}.
$$
Therefore the standardized residual score is
$$
R_i^z=\frac{z_i-m_{n+1,-i}(x_i)}{\sigma_{n+1,-i}(x_i)}
=\frac{w_i}{\sqrt{\bar k_i}},\qquad i=1,\ldots,n.
$$
Substituting the expression of $w_i$ gives
$$
R_i^z
=\frac{(K_n^{-1} z_{1:n})_i+\dfrac{u_i}{v}\,(m_n(x_{n+1})-z)}
        {\sqrt{(K_n^{-1})_{ii}+\dfrac{u_i^2}{v}}},\qquad i=1,\ldots,n.
$$
At the test point $x_{n+1}$, the GP posterior variance is
$\sigma_n^2(x_{n+1})=v$ and the mean is $m_n(x_{n+1})$, so
$$
R_{n+1}^z=\frac{z-m_n(x_{n+1})}{\sqrt{v}}
=\frac{w_{n+1}}{\sqrt{\bar k_{n+1}}}.
$$

\paragraph{Affine form for $R_{n+1}^z - R_i^z$}
Introduce the shorthands
$$
a_i:=(K_n^{-1} z_{1:n})_i,\qquad
b_i:=\frac{u_i}{v},\qquad
d_i:=\sqrt{(K_n^{-1})_{ii}+ \frac{u_i^2}{v}},\qquad
s:=\sqrt{v},\qquad m_n:=m_n(x_{n+1}).
$$
Then
$$
R_i^z=\frac{a_i + b_i(m_n-z)}{d_i}=\frac{a_i + b_i m_n - b_i z}{d_i},
\qquad
R_{n+1}^z=\frac{z-m_n}{s}.
$$
Hence
$$
\begin{aligned}
R_{n+1}^z - R_i^z
&= \frac{z-m_n}{s}
   - \frac{a_i + b_i m_n - b_i z}{d_i} \\[4pt]
&= \Big(\frac{1}{s} + \frac{b_i}{d_i}\Big)\,z
   - \Big(\frac{m_n}{s} + \frac{a_i + b_i m_n}{d_i}\Big).
\end{aligned}
$$
This is of the form $\beta_i z - \alpha_i$ with
$$
\beta_i=\frac{1}{s}+\frac{b_i}{d_i}
=\frac{1}{\sqrt{v}}+\frac{u_i}{v\,\sqrt{(K_n^{-1})_{ii}+u_i^2/v}},
\qquad
\alpha_i=\frac{m_n}{s}+\frac{a_i + b_i m_n}{d_i}.
$$
The difference of residual scores can be expressed in centered affine form,
$$
R_{n+1}^z - R_i^z = \beta_i\,(z-c_i),
$$
where
$$
c_i = \alpha_i/\beta_i = \frac{m_n d_i + s a_i + s b_i m_n}{d_i + s b_i}
= m_n + \frac{s a_i}{d_i + s b_i}\,.
$$
Since $s b_i = u_i/\sqrt{v}$,
$$
c_i = m_n + \frac{v\,a_i}{\sqrt{\,v\,(K_n^{-1})_{ii}+u_i^2\,}+\,u_i} =  m_n(x_{n+1})
+ \frac{v\,(K_n^{-1} z_{1:n})_i}
     {\sqrt{\,v\,(K_n^{-1})_{ii}+u_i^2\,}+\,u_i}\,.
$$

\paragraph{Positivity of $\beta_i$}
An equivalent expression for the slope is
$$
\beta_i
= \sqrt{\bar k_{n+1}} \;-\; \frac{\bar k_{i,n+1}}{\sqrt{\bar k_i}},
\qquad
\text{since }\;
\bar k_{n+1}=\frac{1}{v},\;
\bar k_{i,n+1}=-\frac{u_i}{v},\;
\bar k_i=(K_n^{-1})_{ii}+\frac{u_i^2}{v}.
$$
Because $\bar K_{n+1}\succ 0$, Cauchy--Schwarz in the inner product
$\langle x,y\rangle=x^\top \bar K_{n+1} y$ gives
$
\lvert \bar k_{i,n+1}\rvert \;<\; \sqrt{\bar k_i\,\bar k_{n+1}}.
$
Dividing by $\sqrt{\bar k_i}$ yields
$$
-\sqrt{\bar k_{n+1}}
\;<\;
\frac{\bar k_{i,n+1}}{\sqrt{\bar k_i}}
\;<\;
\sqrt{\bar k_{n+1}},
$$
hence
$
\beta_i
= \sqrt{\bar k_{n+1}} - \frac{\bar k_{i,n+1}}{\sqrt{\bar k_i}}
\;>\; 0.
$
(The inequality is strict since $\bar K_{n+1}$ is positive definite.)

\begin{remark}[Extensions to universal and intrinsic kriging]
\label{rem:affine-extensions}
The same affine representation holds (with obvious substitutions) in
two broader cases: (i) \emph{universal kriging} with unknown linear
mean $m(x)=h(x)^\top\beta$ and known regressors $h(x)\in\R^q$, by
replacing $K_n^{-1}$ with
$Q^{-1}:=K_n^{-1}-K_n^{-1}H(H^\top K_n^{-1}H)^{-1}H^\top K_n^{-1}$,
where $H=[h(x_1),\dots,h(x_n)]^\top$; and (ii) \emph{intrinsic
  kriging} with conditionally positive definite kernels and drift
space $\mathcal{P}$, by projecting onto $\mathrm{col}(H)^\perp$
(choose an orthonormal basis $W$ and work with $W^\top K_n W$),
assuming the design is unisolvent for $\mathcal{P}$.
\end{remark}

\subsection{Proof of Proposition~\ref{prop:cps-gp-finite}}
\label{sec:cps-gp-finite-bounds}

With distinct thresholds
$-\infty=c_{(0)}<c_{(1)}<\cdots<c_{(n)}<c_{(n+1)}=\infty$,
the stepwise CPD satisfies
$$
\hat F_{n,\tau}^{\mathrm{CPS\text{-}GP}}(z\mid x_{n+1})
= \frac{i+\tau}{n+1},
\qquad z\in(c_{(i)},c_{(i+1)}),
$$
so that the right limit at $c_{(i)}$ is
$\hat F_{n,\tau}^{\mathrm{CPS\text{-}GP}}(c_{(i)}^{+}\mid x_{n+1})
= (i+\tau)/(n+1)$.

\emph{Lower endpoint.}
The lower bound of $\Ccal_{n,\tau,1-\alpha}^{\mathrm{CPS\text{-}GP}}(x_{n+1})$
is finite if and only if $\alpha/2$ exceeds the leftmost plateau level
$\hat F(c_{(0)}^{+}\mid x_{n+1})=\tau/(n+1)$, i.e.
$$
\frac{\alpha}{2}>\frac{\tau}{n+1}.
$$

\emph{Upper endpoint.}
The upper bound is finite if and only if there exists a finite
$c_{(j)}$ such that
$\hat F(c_{(j)}^{+}\mid x_{n+1})\ge 1-\alpha/2$.
The last finite level is
$$
\hat F(c_{(n)}^{+}\mid x_{n+1})
= \frac{n+\tau}{n+1}
= 1-\frac{1-\tau}{n+1},
$$
so finiteness requires
$$
1-\alpha/2 \le \frac{n+\tau}{n+1}
\quad\Longleftrightarrow\quad
\frac{\alpha}{2} \ge \frac{1-\tau}{n+1}.
$$

Combining both conditions yields
$$
\alpha \ge \frac{2}{n+1}\max\{\tau,\,1-\tau\},
$$
with strict inequality required in the case $\tau>1/2$, since the lower–endpoint
condition involves a strict inequality.

\subsection{Proof of Proposition~\ref{prop:complexity}}

\emph{Precomputation.}  
Compute the Cholesky factorization $K_n=LL^\top$ at a cost of $O(n^3)$.  
From $L$, obtain $\alpha=K_n^{-1}z_{1:n}$ via two triangular solves in $O(n^2)$.  
If needed, $\operatorname{diag}(K_n^{-1})$ can be computed once in $O(n^3)$ 
(for example by forming $L^{-1}$); this cost is included in the same precomputation order.

\emph{Per prediction point $x_{n+1}$.}
\begin{enumerate}[leftmargin=*, itemsep=2pt]
\item Compute $k_*\in\R^n$ and $k_{**}$: $O(n)$ kernel evaluations.
\item Solve $u=K_n^{-1}k_*$ using $L$: $O(n^2)$.
\item Compute $v=k_{**}-k_*^\top u$ and $m=k_*^\top\alpha$: $O(n)$.
\item Evaluate all thresholds $c_i$ using their closed forms, based on $u_i$, $v$, $m$, 
      $(K_n^{-1}z_{1:n})_i$, and $(K_n^{-1})_{ii}$: $O(n)$.
\item If the full stepwise distribution is required, sort $\{c_i\}_{i=1}^n$: $O(n\log n)$.
\end{enumerate}

The resulting per–point complexity is $O(n^2+n\log n)$,
dominated by the triangular solves, with sorting as the only extra operation
beyond standard GP prediction ($O(n^2)$).  
If only the interval endpoints are needed, the relevant order statistics
can be extracted in expected $O(n)$ time (e.g.\ by quickselect),
yielding $O(n^2+n)$ per prediction point.

\section{Forecasting primer: auto-, marginal-, and probabilistic calibration}
\label{app:forecasting-primer}

This section recalls standard calibration notions in a general
prediction-space setting; they serve as background only. The main
article evaluates calibration in a design-based sense with respect to~$\mu$
(Sections~\ref{sec:bg-coverage}--\ref{sec:bg-pit}).

Let $(\Omega,\mathcal A,\P)$ be a probability space. A \emph{predictive CDF} is a
measurable random element $\hat F:\Omega\to\mathcal D$, where $\mathcal D$ is the
set of all cumulative distribution functions on $\R$, equipped with the
$\sigma$-algebra generated by finite-dimensional cylinders. We denote by
$\sigma(\hat F)$ the $\sigma$-algebra generated by~$\hat F$.

Let $Z:\Omega\to\R$ be the outcome of interest with true CDF~$F$. All statements
below are made under the joint distribution of $(\hat F,Z)$ (and, when needed,
an auxiliary $\tau\sim\U(0,1)$ independent of both).

\begin{remark}
This setting differs slightly from the \emph{prediction space} framework of
\citet{gneiting_combining_2013}, where elements of $\Omega$ are realizations of
the triple $(\hat F,Z,\tau)$.
\end{remark}

\paragraph{Auto-calibration}
A predictive distribution is \emph{auto-calibrated}
\citep{Tsyplakov2013} if
$$
  \P(Z \le z \mid \sigma(\hat F)) = \hat F(z)
  \quad \text{a.s. for all } z\in\R.
$$
This requires that, conditional on $\sigma(\hat F)$, the true conditional
distribution of $Z$ coincides with~$\hat F$. This is a strong requirement.
Tests for auto-calibration exist
\citep{strahl2015crosscalibrationprobabilisticforecasts}, but assessing it in
practice is difficult, which motivates weaker notions.

\paragraph{Marginal calibration}
A predictive distribution $\hat F$ is \emph{marginally calibrated} if, for every
$z \in \R$,
$$
  \E[\hat F(z)] = \P(Z \le z).
$$
Auto-calibration implies marginal calibration: since
$\P(Z \le z \mid \sigma(\hat F))=\hat F(z)$ a.s., taking expectations and using
the law of total expectation gives the identity above. Marginal calibration is
weaker than auto-calibration and is more tractable empirically, since both sides
can be estimated by averages over forecast–observation pairs.

\paragraph{Probabilistic calibration}
A second relaxation is \emph{probabilistic calibration}, based on the
\emph{probability integral transform} (PIT). For a continuous predictive CDF
$\hat F$ and an observation $Z$, define
$$
  U_{\hat F}^Z = \hat F(Z).
$$
If $\hat F = F$, then $U_{\hat F}^Z \sim \U(0,1)$; this is the usual PIT
property. A forecast $\hat F$ is called \emph{probabilistically calibrated} if
$U_{\hat F}^Z$ is uniformly distributed on $[0,1]$. In practice, the PIT
property is assessed by comparing the empirical distribution of PIT values to
the uniform law, for example via histograms or empirical CDF plots
\citep{Dawid1984}. For general, possibly discontinuous, predictive
distributions, one uses the randomized PIT
$$
  U_{\hat F}^Z
  = \hat F(Z^-) + \tau\big(\hat F(Z) - \hat F(Z^-)\big),
  \qquad \tau \sim \U(0,1),
$$
which restores uniformity when $\hat F = F$.

\begin{proposition}
\label{prop:pit-unif}
Let $Z$ be a real-valued random variable with CDF $F$. Let $\hat F$ be a random
continuous CDF, independent of $Z$, and define $U_{\hat F}^Z = \hat F(Z)$. Then
the following are equivalent:
\begin{enumerate}[label=(\roman*),leftmargin=*]
\item For all $u\in[0,1]$,
$$
  \P\bigl(U_{\hat F}^Z \le u \mid \hat F\bigr) = u \quad \text{a.s.}
$$
\item $\hat F = F$ almost surely.
\end{enumerate}
\end{proposition}

\begin{proof}
(ii) $\Rightarrow$ (i).  
If $\hat F = F$ a.s., then conditional on $\hat F$ we have $U_{\hat F}^Z = F(Z)$,
which is uniform on $[0,1]$. Hence
$$
  \P\bigl(U_{\hat F}^Z \le u \mid \hat F\bigr) = u \quad \text{a.s.}
$$

(i) $\Rightarrow$ (ii).  
Fix $u\in[0,1]$ and define
$$
  q_u = \inf\{z:\ \hat F(z)\ge u\}.
$$
Continuity of $\hat F$ gives $\hat F(q_u) = u$ and $q_u = \hat F^{-1}(u)$. Using
independence of $\hat F$ and $Z$,
$$
  \P\bigl(U_{\hat F}^Z \le u \mid \hat F\bigr)
  = \P\bigl(\hat F(Z)\le u \mid \hat F\bigr)
  = \P(Z\le q_u \mid \hat F)
  = F(q_u).
$$
By assumption this equals $u$ a.s., so $F(q_u)=u$ for all $u\in[0,1]$. Thus
$$
  F(\hat F^{-1}(u)) = u \quad \text{a.s. for all } u\in[0,1],
$$
which implies $F = \hat F$ almost surely.
\end{proof}

Finally, auto-calibration implies probabilistic calibration. Indeed, if
$$
  \P(Z \le z \mid \sigma(\hat F)) = \hat F(z)
  \quad\text{a.s. for all } z\in\R,
$$
then for any $u\in[0,1]$, with $q_u=\inf\{z:\hat F(z)\ge u\}$,
$$
  \P(U_{\hat F}^Z \le u \mid \sigma(\hat F))
  = \P(Z \le q_u \mid \sigma(\hat F))
  = \hat F(q_u) = u,
$$
so $U_{\hat F}^Z$ is uniform (for discontinuous $\hat F$, the same argument
applies with the randomized PIT). By contrast, marginal calibration does not
imply probabilistic calibration, and probabilistic calibration does not imply
marginal calibration in general (see \citet[§2]{gneiting_combining_2013} and
\citet{Gneiting:2023}).

\section{Generalized normal distribution}
\label{apd:gnp}

The generalized normal distribution $\mathcal{GN}(\beta, \mu, \lambda)$, with
shape parameter $\beta > 0$, location $\mu \in \R$ and scale parameter
$\lambda > 0$, is a continuous distribution that extends the classical normal
distribution. Its probability density function is
$$
  f(z)
  = \frac{\beta}{2\,\Gamma(1/\beta)\,\lambda}
    \exp\!\left[
      -\left(\frac{|z - \mu|}{\lambda}\right)^{\beta}
    \right],
  \qquad z\in\R.
$$

All moments exist. By symmetry, $\E[Z]=\mu$, and the variance is
$$
  \Var(Z)
  = \lambda^{2}\,\frac{\Gamma(3/\beta)}{\Gamma(1/\beta)}.
$$

The shape parameter $\beta$ controls tail behavior. When $\beta = 2$, one
recovers the Gaussian case:
$$
  \mathcal{GN}(2,\mu,\lambda) = \mathcal N\!\big(\mu,\lambda^2/2\big).
$$
For $\beta < 2$, the distribution has heavier tails than the Gaussian, and for
$\beta > 2$ it has lighter tails. Further analytical properties of the
generalized normal family are given in \citet{nadarajah2005analytical}.

\section{SCRPS formulas}
\label{ap:scrps_formulas}

\subsection{SCRPS for the generalized normal distribution}
Let $F_{\beta}$ be the CDF of $\mathcal{GN}(\beta, 0, 1)$, and $\Gamma(\cdot,
\cdot)$ is the upper incomplete gamma function. The SCRPS for $Z\sim\mathcal{GN}
(\beta, 0, 1)$ can be approximated using the Propositions~\ref{prop:scrps_gn_1} 
and~\ref{prop:scrps_gn_2}. Indeed, $\E(|Z' - Z|)$ for $Z'$ a copy of $Z$, can 
be approximated for several values of $\beta$, then an interpolation 
technique is used to learn $\beta \rightarrow \E(|Z' - Z|)$, and $\E(|Z - z'|)$
admits a close form formula.

\begin{proposition}
  Let $z'\in\R$ and $Z\sim GN(\beta, \mu, \lambda)$, then, with $u=(z' - \mu)/(\lambda)$,
  \begin{equation}
    \E(|Z - z'|) = \lambda\left[u\left(2F_{\beta}(u) - 1\right) + \frac{1}{\Gamma\left(1/\beta\right)}\Gamma\left(\frac{2}{\beta}, \left|u\right|^{\beta}\right)\right].
  \end{equation}
  \label{prop:scrps_gn_1}
\end{proposition}

\begin{proof}
  \begin{align*}
    \E(|Z - z'|) & = \int_{\R} |y - z'|\frac{\beta}{2\Gamma(1/\beta)\lambda}\exp\left[ - \left(\frac{|y - \mu|}{\lambda}\right)^{\beta}\right] \dd y                                                      \\
                 & = \lambda\frac{\beta}{2\Gamma(1/\beta)}\left(\int_{-\infty}^{u} (u - z)\exp\left( - |z|^{\beta}\right) \dd z + \int_{u}^{+\infty} (z - u)\exp\left( - |z|^{\beta}\right) \dd z \right) \\
                 & = \lambda\left[u(2F_{\beta}(u) - 1) + \frac{1}{\Gamma\left(1/\beta\right)}\Gamma\left(\frac{2}{\beta}, \left|u\right|^{\beta}\right) \right].
  \end{align*}
  The last line is now proved, first notice that by symmetry of $\int_{0}^{u} z\exp\left( - |z|^{\beta}\right) \dd z$:
  $$
    \left \{
    \begin{array}{rcl}
      \int_{u}^{+\infty} z\exp\left( - |z|^{\beta}\right) \dd z & = \frac{1}{\beta}\Gamma\left(\frac{2}{\beta}\right) - \int_{0}^{|u|} z\exp\left( - |z|^{\beta}\right) \dd z  \\
      \int_{-\infty}^{u} z\exp\left( - |z|^{\beta}\right) \dd z & = \int_{0}^{|u|} z\exp\left( - |z|^{\beta}\right) \dd z - \frac{1}{\beta}\Gamma\left(\frac{2}{\beta}\right),
    \end{array}
    \right.
  $$
  Finally,
  $$
    \int_{u}^{+\infty} z\exp\left( - |z|^{\beta}\right) \dd z - \int_{-\infty}^{u} z\exp\left( - |z|^{\beta}\right) \dd z = 2\frac{1}{\beta} \Gamma(2/\beta, |u|^{\beta})
  $$
  which proves the results.
\end{proof}

\begin{lemma}
  Let the random variables $Z', Z \sim GN(\beta, \mu, \lambda)$ and $U', U \sim GN(\beta, 0, 1)$, then $\E(|Z' - Z|) = \lambda\E(|U' - U|)$.
  \label{prop:scrps_gn_2}
\end{lemma}

\begin{proof}
    Follows from the fact that  $\lambda U + \mu \sim GN(\beta, \mu, \lambda)$.
\end{proof}

\subsection{SCRPS for empirical distributions}

Let an i.i.d sequence $Z_1, \ldots, Z_n$ and denote by
$F_n(z) =  \frac{1}{n} \sum_{i=1}^n \one(Z_i \leq z)$ its empirical CDF. The 
SCRPS for a random variable with CDF $F_n$ can be computed using 
Proposition~\ref{prop:dis_scrps}.

\begin{proposition}
  Let $z\in\R$ and $Z', Z \sim F_n$, then:
  \begin{equation}
    \E_{Z\sim F_n}(|Z - z|)       = \frac{1}{n}\sum_{i=1}^n |Z_i - z|, \quad \text{and} \quad
    \E_{Z', Z \sim F_n}(|Z' - Z|) = \frac{1}{n^2} \sum_{i=1}^n \sum_{j=1}^n |Z_i - Z_j|.
  \end{equation}
  \label{prop:dis_scrps}
\end{proposition}

\begin{proof}
  The proof is straightforward.
\end{proof}

\section{Test functions used in the experiments}
\label{ap:test_functions}

\begin{tabularx}{\textwidth}{l l Y}
  \caption{Test functions used in the numerical experiments.
  Each function is continuous, deterministic, and defined on the domain~$\XX$.
  Expressions and domains follow the standard definitions provided in
  \citet{surjanovic_bingham_optimization}.}
\label{tab:test_functions}
\\
\toprule
\textbf{Name} & \textbf{Domain} & \textbf{Expression of $f(x)$}\\
\midrule
Branin & $[-5,10]\times[0,15]$ &
$\bigl(x_2 - \tfrac{5.1}{4\pi^2}x_1^2 + \tfrac{5}{\pi}x_1 - 6\bigr)^2 + 10\bigl(1 - \tfrac{1}{8\pi}\bigr)\cos(x_1) + 10$ \\
Goldstein--Price & $[-2,2]^2$ &
$(1 + (x_1 + x_2 + 1)^2(19 - 14x_1 + 3x_1^2 - 14x_2 + 6x_1x_2 + 3x_2^2))(30 + (2x_1 - 3x_2)^2(18 - 32x_1 + 12x_1^2 + 48x_2 - 36x_1x_2 + 27x_2^2))$ \\
Rosenbrock & $[-5,10]^d$ &
$\sum_{i=1}^{d-1} \bigl[100(x_{i+1}-x_i^2)^2 + (x_i-1)^2\bigr]$ \\
Ackley & $[-32.168,32.168]^d$ &
$-20\exp\!\left(-0.2\sqrt{\tfrac{1}{d}\sum_i x_i^2}\right) - \exp\!\left(\tfrac{1}{d}\sum_i \cos(\pi x_i)\right) + 20 + e$ \\
Beale & $[-4.5,4.5]^2$ &
$(1.5 - x_1 + x_1x_2)^2 + (2.25 - x_1 + x_1x_2^2)^2 + (2.625 - x_1 + x_1x_2^3)^2$ \\
Dixon--Price & $[-10,10]^d$ &
$(x_1 - 1)^2 + \sum_{i=2}^{d} i(2x_i^2 - x_{i-1})^2$ \\
Hartmann--3 & $[0,1]^3$ &
$-\sum_{i=1}^{4} c_i \exp\!\bigl(-\sum_{j=1}^{3} a_{ij}(x_j - p_{ij})^2\bigr)$ \\
Hartmann--6 & $[0,1]^6$ &
$-\sum_{i=1}^{4} c_i \exp\!\bigl(-\sum_{j=1}^{6} a_{ij}(x_j - p_{ij})^2\bigr)$ \\
\bottomrule
\end{tabularx}

\medskip
\noindent\textbf{Hartmann parameters.}
\scriptsize
$$
a^{(3)}=\begin{bmatrix}
3.0 & 10.0 & 30.0\\
0.1 & 10.0 & 35.0\\
3.0 & 10.0 & 30.0\\
0.1 & 10.0 & 35.0
\end{bmatrix},\quad
c^{(3)}=\begin{bmatrix}1.0 \\ 1.2 \\ 3.0 \\ 3.2\end{bmatrix},
\quad
p^{(3)}=10^{-1}\!\begin{bmatrix}
1 & 1 & 1\\
3 & 3 & 3\\
5 & 5 & 5\\
7 & 7 & 7
\end{bmatrix}.
$$
$$
a^{(6)}=\begin{bmatrix}
10 & 3 & 17 & 3.5 & 1.7 & 8\\
0.05 & 10 & 17 & 0.1 & 8 & 14\\
3 & 3.5 & 1.7 & 10 & 17 & 8\\
17 & 8 & 0.05 & 10 & 0.1 & 14
\end{bmatrix},\quad
c^{(6)}=\begin{bmatrix}1.0 \\ 1.2 \\ 3.0 \\ 3.2\end{bmatrix},\quad
p^{(6)}=10^{-2}\!\begin{bmatrix}
1312 & 1696 & 5569 & 124 & 8283 & 5886\\
2329 & 4135 & 8307 & 3736 & 1004 & 9991\\
2348 & 1451 & 3522 & 2883 & 3047 & 6650\\
4047 & 8828 & 8732 & 5743 & 1091 & 381
\end{bmatrix}.
$$
\normalsize

\bibliographystyle{plainnat}
\bibliography{main}
\end{document}